\documentclass{article}
\usepackage[utf8]{inputenc}

\usepackage{natbib}
\usepackage{graphicx}
\usepackage{hyperref}
\usepackage[margin=1in]{geometry}

\usepackage[T1]{fontenc}
\usepackage{graphicx}
\usepackage{caption}
\usepackage{subcaption}
\usepackage{booktabs}
\usepackage{wrapfig}

\usepackage{xcolor}
\usepackage{amsmath} 
\usepackage{amsthm} 
\usepackage{amsfonts} 
\usepackage{thmtools}
\usepackage{thm-restate}
\usepackage{amssymb}
\usepackage{multirow}
\usepackage{todonotes}
\usepackage{etoolbox}
\usepackage{enumitem}
\usepackage{natbib}
\usepackage[nameinlink,capitalize,noabbrev]{cleveref}
\usepackage{url}
\usepackage{xurl}
\usepackage{hyperref}

\theoremstyle{plain}

\theoremstyle{plain}

\theoremstyle{plain}
\newtheorem{assumption}{Assumption}

\theoremstyle{plain}

\theoremstyle{definition}
\newtheorem{definition}{Definition}

\declaretheorem[style=definition,qed=$\Diamond$]{example}

\newcommand{\groups}{\mathcal{G}}
\newcommand{\group}{g}
\newcommand{\capgroup}{G}

\newcommand{\risk}{\mathcal{R}}
\newcommand{\groupfxn}{r}

\newcommand{\appropto}{\mathrel{\vcenter{
  \offinterlineskip\halign{\hfil$##$\cr
    \propto\cr\noalign{\kern2pt}\sim\cr\noalign{\kern-2pt}}}}}

\DeclareMathOperator*{\argmin}{argmin}

\newcommand*{\medcup}{\mathbin{\scalebox{1.5}{\ensuremath{\cup}}}}%

\newcommand{\newtext}[1]{{#1}}

\newcommand{\yrcite}[1]{\citeyearpar{#1}}
\renewcommand{\cite}[1]{\citep{#1}}

\title{Representation Matters: \\Assessing the Importance of Subgroup Allocations 
in Training Data}

\author{
\and Esther Rolf\thanks{Department of Electrical Engineering and Computer Sciences, University of California, Berkeley}
\and Theodora Worledge\footnotemark[1]
\and Benjamin Recht\footnotemark[1]
\and Michael I. Jordan\footnotemark[1]\,\,\thanks{Department of Statistics, University of California, Berkeley}
}
\date{\today}

\usepackage{natbib}
\usepackage{graphicx}

\newtoggle{icml}
\togglefalse{icml}

\begin{document}

\maketitle

\begin{abstract}
Collecting more diverse and representative training data is often touted as a remedy for the disparate performance of machine learning predictors across subpopulations. However, a precise framework for understanding how dataset properties like diversity affect learning outcomes is largely lacking. By casting data collection as part of the learning process, we demonstrate that diverse representation in training data is key not only to increasing subgroup performances, but also to achieving population-level objectives. Our analysis and experiments describe how dataset compositions influence performance and provide constructive results for using trends in existing data, alongside domain knowledge, to help guide intentional, objective-aware dataset design.

\end{abstract}


\section{Introduction}
\label{sec:intro}

Datasets play a critical role in shaping the perception of performance and progress in machine learning (ML)---the way we collect, process, and analyze data affects the way we benchmark success and form new research agendas \citep{paullada2020data, dotan2019value}.
A growing appreciation of this determinative role of datasets has sparked a concomitant concern that standard datasets used for training and evaluating ML models lack diversity along significant dimensions, for example, geography, gender, and skin type \citep{shankar2017no,buolamwini2018gender}. 
Lack of diversity in evaluation data can obfuscate disparate performance when evaluating based on aggregate accuracy \citep{buolamwini2018gender}. 
Lack of diversity in training data can limit the extent to which learned models can adequately apply to all portions of a population, a concern highlighted in recent work in the
medical domain~\citep{ habib2019impact,hofmanninger2020automatic}.

Our work aims to develop a general unifying perspective on the way that dataset composition affects outcomes of machine learning systems. We focus on \emph{dataset allocations}: the number of datapoints from predefined subsets of the population. 
While we acknowledge that numerical inclusion of groups is an imperfect proxy of representation, we believe that allocations provide a useful initial mathematical abstraction for formulating relationships among diversity, data collection, and statistical risk. We discuss broader implications of our formulation in \cref{sec:limitations_future_work}.  
%

With the implicit assumption that 
the learning task is well specified and performance evaluation from data is meaningful for all groups,
we ask:

\iftoggle{icml}
{\textit{Are group allocations in training data pivotal to performance? To what extent can up-weighting underrepresented groups help, and when might it actually hurt performance?}}
{\begin{enumerate}
\item{\textit{Are group allocations in training data pivotal to performance? To what extent can methods that up-weight underrepresented groups help, and when might upweighting actually hurt performance?}}
\end{enumerate}}

Taking a point of view that data collection is a critical component of the overall machine learning process, we study the effect that dataset composition has on group and population accuracies.
%
This complements work showing that simply gathering more data can mitigate some sources of bias or unfairness in ML outcomes \citep{chen2018my}, a phenomenon which has been observed in practice as well. 
Indeed, in response to the Gender Shades study \citep{buolamwini2018gender}, 
 companies selectively collected additional data to decrease the exposed inaccuracies of their facial recognition models for certain groups, often raising aggregate accuracy in the process \citep{raji2019actionable}. 
Given the potential for targeted data collection efforts to repair unintended outcomes of ML systems, we next ask:

\iftoggle{icml}
{\textit{How can we describe “optimal” data allocations for different learning objectives? Does a lack of diversity in large-scale datasets align with maximizing population accuracy?}}
{\begin{enumerate}
\setcounter{enumi}{1}
\item{\textit{How might we describe “optimal” dataset allocations for different learning objectives? Does the often-witnessed lack of diversity in large-scale datasets align with maximizing population accuracy?}}
\end{enumerate}}

We show that purposeful data collection efforts can proactively support intentional objectives of an ML system, and that diversity and population objectives are often aligned.
Many datasets have recently been designed or amended to exhibit diversity of the underlying population \citep{ryu2017inclusivefacenet, Tschandl2018, yang2020towards}.
\iftoggle{icml}
{These}
{Such endeavors} are significant undertakings, as data gathering and annotation must consider consent, privacy, and power concerns in addition to inclusivity, transparency and reusability  \citep{gebru2018datasheets, jo2020lessons, wilkinson2016fair}. 
Given the importance of more representative and diverse
datasets, and the effort required to create them, our final question asks:

\iftoggle{icml}
{\textit{When and how can we leverage existing datasets to help inform better allocations, towards achieving a diverse set of objectives in a subsequent dataset collection effort?}}
{\begin{enumerate}
\setcounter{enumi}{2}
\item{\textit{When and how can we leverage existing datasets to help inform better allocations, towards achieving diverse objectives in a subsequent dataset collection effort?}}
\end{enumerate}}

{Representation bias, or systematic underrepresentation of subpopulations in data, is one of many forms of bias in ML \citep{suresh2019framework}.
Our work provides a data-focused perspective on the design and evaluation of ML pipelines.}
%
Our main contributions are:
\begin{enumerate}
\item{
We analyze the complementary roles of dataset allocation and algorithmic interventions for achieving per-group and total-population performance (\cref{sec:allocations_and_alternatives}).
Our experiments show that while algorithmically up-weighting underrepresented groups can help, dataset composition is the most consistent determinant of performance (\cref{sec:navigating_tradoffs}).}

\item{We propose a scaling model that describes the impact of dataset allocations on group accuracies (\cref{sec:allocating_samples}).
Under this model, when parameters governing the relative values of within-group data are equal for all groups, the allocation that minimizes \emph{population} risk \emph{overrepresents} minority groups.}
    \item{ We demonstrate that our proposed scaling model captures major trends of the relationship between dataset allocations and performance (\cref{sec:validating_scaling_laws,sec:group_interactions}).
    We evidence that a small initial sample can be used to inform subsequent data collection efforts to, for example, maximize the minimum accuracy over groups without sacrificing population accuracy
    (\cref{sec:informing_future_allocations}).}
\end{enumerate}

\Cref{sec:allocations_and_alternatives,sec:scaling_model} formalize data collection as part of the learning problem and derive results under illustrative settings. Experiments in \cref{sec:experiments} support these results and expose nuances inherent to real-data contexts.
\cref{sec:limitations_future_work} synthesizes results and delineates future work.

\subsection{Additional Related Work}
\label{sec:related_work}

\textbf{Targeted data collection in ML.}
Recent research evidences that targeted data collection can be an effective way to reduce disparate performance of ML models evaluated across sub-populations \citep{raji2019actionable}.
Chen et al.~\yrcite{chen2018my} present a formal argument that the addition of training data can lessen discriminatory outcomes while improving accuracy of learned models, and 
\citet{abernethy2020adaptive} show that adaptively collecting data from the lowest-performing sub-population can increase the minimum accuracy over groups. 
\iftoggle{icml}
{It is important to note, however, there are many complications associated with simply gathering more data as a solution to disparate performance across groups \citep{Jacobs2019measurement,paullada2020data}.}
{At the same time, there are many complexities of gathering data as a solution to disparate performance across groups.
Targeted data collection from specific groups can present undue burdens of surveillance or skirt consent \citep{paullada2020data}. 
When ML systems fail portions of their intended population due to issues of measurement and construct validity, more thorough data collection is unlikely to solve the issue without further intervention \citep{Jacobs2019measurement}.}

With these complexities in mind, we study the importance of numerical representation in training datasets in achieving diverse objectives.
Optimal allocation of subpopulations in statistical survey designs dates back to at least \citet{neyman1934}, 
including stratified sampling methods to ensure coverage across sub-populations \citep{lohr2009sampling}.
For more complex prediction systems,
the field of optimal experimental design \citep{PukelsheimFriedrich2006Odoe} studies what inputs 
are most valuable for reaching a given objective, often focusing on linear prediction functions. 
We consider a constrained sampling structure and directly model the impact of group allocations on subgroup performance for general model classes.

\textbf{Valuing data.}
In economics, allocations indicate a division of goods to various entities \citep{cole2012mechanism}. While we focus on the influence of data allocations on model accuracies across groups, there are many approaches to valuing data. Methods centering on a theory of Shapley valuations \citep{yona2019s, ghorbani2019data} complement studies of the influence of individual data points on model performance to aid subsampling data~\citep{vodrahalli2018all}.

\iftoggle{icml}{\textbf{Handling group-imbalanced data.}}
{\textbf{Methods for handling group-imbalanced data.}}
Importance sampling and importance weighting are standard approaches to addressing class imbalance or small groups sizes \citep{haixiang2017learning,buda2018systematic}, 
though the effects of importance weighting for deep learning may vary with regularization~\citep{byrd2019effect}. 
%
Other methods specifically address differential performance between groups. 
Maximizing minimum performance across groups  
can reduce accuracy disparities
\citep{Sagawa2020Distributionally}
and promote fairer sequential outcomes \citep{hashimoto2018fairness}.
%
For broader classes of group-aware objectives, techniques exist to mitigate unfairness 
or disparate performance 
of black box prediction functions \citep{dwork2018decoupled,kim2019multiaccuracy}.
It might not be clear a priori which subsets need attention;
\citet{sohoni2020no} propose a method to identify and account for hidden strata, while other methods are defined for any subsets \citep{hashimoto2018fairness, kim2019multiaccuracy}. 
\iftoggle{icml}{One can also downsample or augment the input data to match a desired distribution~\citep{chawla2002smote, iosifidis2018dealing}. }
{

In addition to modifying the optimization objective or learning algorithm,
one can also modify the input data itself by re-sampling the training data to match the desired population by downsampling  or by upsampling with data augmentation techniques~\citep{chawla2002smote, iosifidis2018dealing}. }
 

\subsection{Notation}
$\Delta^k$ denotes the $k$-dimensional simplex. $\mathbb{Z}^+$ denotes non-negative integers and $\mathbb{R}^+$ non-negative reals.

\section{Training set allocations and alternatives}
\label{sec:allocations_and_alternatives}
We study settings in which each data instance is associated with a group $g_i$, so that the training set can be expressed as $\mathcal{S} = \{x_i, y_i, g_i\}_{i=1}^n$ where $x_i, y_i$ denote the features and labels of each instance.  We index the discrete \textbf{groups} by integers $\groups = \{1,..,|\groups|\}$, or when we specifically consider just two groups, we write $\groups = \{A, B\}$.
We assume that groups are disjoint and cover the entire population, with $\gamma_\group = P_{(X,Y,\capgroup) \sim \mathcal{D}}[\capgroup = \group]$ denoting the \textbf{population prevalence} of group $\group$, so that $\vec{\gamma} \in \Delta^{|\groups|}$.
%
Groups could represent inclusion in one of many binned demographic categories, or simply a general association with latent characteristics that are relevant to prediction. 
%

For a given population with distribution $\mathcal{D}$ over features, labels, and groups, we are interested in the population level risk, $\risk(\hat{f}({\mathcal{S})}; \mathcal{D}) := \mathbb{E}_{(X,Y,G) \sim \mathcal{D}}[ \ell (\hat{f}(X),Y) ]$, of a predictor $\hat{f}$ trained on dataset $\mathcal{S}$, as well as group specific risks.  
Denoting the \textbf{group distributions} by $\mathcal{D}_\group$, defined as conditional distributions, via
${{P_{(X,Y) \sim \mathcal{D}_\group}[X=x,Y=y]}} = P_{(X,Y, \capgroup )\sim \mathcal{D}}[X=x,Y=y, \capgroup = \group] / \gamma_\group$,
the population risk decomposes as a weighted average over group risks:
\begin{align}
\label{eq:err_per_group}
    \risk(\hat{f}(\mathcal{S}); \mathcal{D}) &= \sum_{\group \in \groups} \gamma_\group \cdot \risk(\hat{f}(\mathcal{S}); \mathcal{D}_\group) ~.
\end{align}
%
In \cref{sec:IW_GDRO} we will assume that the loss $\ell(\hat{y},y)$ is a separable function over data instances.  While this holds for many common loss functions,
some objectives do not decouple in this sense~\citep[e.g., group losses and associated classes of fairness-constrained objectives; see][]{dwork2018decoupled}.
We revisit this point in \cref{sec:experiments,sec:limitations_future_work}. 

\subsection{Training Set Allocations}
In light of the decomposition of the population-level risk as a weighted average over group risks
in Eq.~\eqref{eq:err_per_group}, we now consider the composition of fixed-size training sets, in terms of how many samples come from each group.
%

\begin{definition}[Allocations]
\label{def:allocations}
Given a dataset of $n$ triplets, $\{x_i,y_i, \group_i \}_{i=1}^{n}$, the \textbf{allocation} $\vec{\alpha}  \in \Delta^{|\groups|}$ describes the relative proportions of each group in the dataset: 
\begin{align}
    \label{eq:allocation_def}
    \alpha_{\group} := \tfrac{1}{n}\sum_{i=1}^n \mathbb{I}[\group_i = \group], \quad \group \in \groups.
\end{align}
\end{definition}
It will be illuminating to consider $\vec{\alpha}$ not only as a property of an existing dataset, but as a parameter governing dataset construction, as captured in the following definition. 

\begin{definition}[Sampling from allocation $\vec{\alpha}$]
\label{def:sampling_def}
Given the sample size $n$, group distributions $\{\mathcal{D}_\group\}_{\group \in \groups}$, and allocation $\vec{\alpha} \in \Delta^{|\groups|}$, such that
$ {n_\group :=}{}
\alpha_\group n \in \mathbb{Z}^+, \forall \group \in \groups$, to \textbf{sample from allocation $\vec{\alpha}$} is procedurally equivalent to independent sampling of $|\groups|$ disjoint datasets $\mathcal{S}_\group$ and concatenating:
{\begin{align}
&\mathcal{S}(\vec{\alpha},n) = \underset{\group \in \groups}{\medcup} \mathcal{S}_\group \\ 
&\mathcal{S}_\group = \{ x_i, y_i, \group\}_{i=1}^{n_\group}, \quad (x_i, y_i) \sim_{i.i.d.} \mathcal{D}_\group~. \nonumber
\end{align}}
\iftoggle{icml}{}
{For $\vec{\alpha}$ not satisfying the requirement that $\alpha_\group n$ is integral, we could randomize the fractional allocations, or take $n_\group = \lfloor \alpha_\group n \rfloor$, reducing the total number of samples to $\sum_\group \lfloor \alpha_\group n \rfloor$.}
In the following sections we will generally allow allocations with $n_\group \not\in \mathbb{Z}$, assuming that the effect of up to $|\groups|$ fractionally assigned instances is negligible for large  $n$.
\end{definition}
The procedure in \cref{def:sampling_def} suggests formalizing data collection as a component of the learning process in the following way: in addition to choosing a loss function and method for minimizing the risk, choose the relative proportions at which to sample the groups in the training set:
\begin{align*}
   \vec{\alpha}^* = \argmin_{\vec{\alpha} \in \Delta^{|\groups|}} \min_{\hat{f} \in \mathcal{F}} \mathcal{R}\left(\hat{f}\left({\mathcal{S}(\vec{\alpha},n)}\right); \mathcal{D}\right).
\end{align*}
In Section~\ref{sec:allocating_samples}, we show that when a dataset curator can design dataset allocations in the sense of \cref{def:sampling_def}, they have the opportunity to improve accuracy of the trained model.
\iftoggle{icml}{}{
Of course, one 
does not always have the opportunity to collect new data or 
modify the composition of an existing dataset. }
\cref{sec:IW_GDRO} considers methods for using fixed datasets that have groups with small training set allocation $\alpha_\group$, relative to $\gamma_\group$, or high risk for some groups relative to the population.

\iftoggle{icml}
{\subsection{Accounting for Small Group Allocations}}
{\subsection{Accounting for Small Group Allocations in Fixed Datasets}}
\label{sec:IW_GDRO}
In classical \textbf{empirical risk minimization} (ERM), one learns a function from class $\mathcal{F}$ that minimizes average prediction loss over the training instances $(x_i,y_i,\group_i) \in \mathcal{S}$ (we also abuse notation and write $i \in \mathcal{S}$) with optional regularization $R$:
\begin{align*}
    \hat{f}(\mathcal{S}) = \argmin_{f \in \mathcal{F}} \sum_{i \in \mathcal{S}} \ell(f(x_i), y_i) + R(f, \mathcal{S}).
\end{align*}
There are many methods for addressing small group allocations in data (see \cref{sec:related_work}). Of particular relevance to our work are objective functions that  minimize group or population risks. In particular, one approach is to use \textbf{importance weighting} (IW) to re-weight training samples with respect to a target distribution defined by $\vec{\gamma}$:
\iftoggle{icml}
{\begin{align*}
    \hat{f}^{\scriptscriptstyle \textrm{IW}}
    (\mathcal{S}) = \argmin_{f \in \mathcal{F}} \sum_{\group \in \groups} \frac{\gamma_\group}{\alpha_\group} \bigl( \sum_{i \in 
    \mathcal{S}_\group
    }  \ell(f(x_i), y_i)\bigr) + R(f, \mathcal{S}).
\end{align*}}
{\begin{align*}
    \hat{f}^{\scriptscriptstyle \textrm{IW}}
    (\mathcal{S}) = \argmin_{f \in \mathcal{F}} \sum_{\group \in \groups} \frac{\gamma_\group}{\alpha_\group} \left( \sum_{i \in 
    \mathcal{S}_\group
    }  \ell(f(x_i), y_i)\right) + R(f, \mathcal{S}).
\end{align*}}
%
This empirical risk with instances weighted by ${\gamma_\group}/{\alpha_\group} = {\gamma_\group} n/ {n_\group}$  is an unbiased estimate of the population risk, up to regularization.
While unbiasedness is often desirable, IW can induce high variance of the estimator when $\gamma_\group/\alpha_\group$ is large for some group \citep{cortes2010learning}, which happens when group $g$ is severely underrepresented in the training data relative to their population prevalence.

Alternatively, \textbf{group distributionally robust optimization} (GDRO) \citep{hu2018does,Sagawa2020Distributionally} minimizes the maximum empirical risk over all groups:
\iftoggle{icml}
{\begin{align*}
     \hat{f}^{\scriptscriptstyle \textrm{GDRO}}(\mathcal{S}) = \argmin_{f \in \mathcal{F}} \max_{\group \in \groups}  
     \bigl(
     \tfrac{1}{n_\group}\sum_{i \in 
    \mathcal{S}_\group
     }  \ell(f(x_i), y_i) + R(f, \mathcal{S}_\group) \bigr).
\end{align*}}
{\begin{align*}
    \hat{f}^{\scriptscriptstyle \textrm{GDRO}}(\mathcal{S}) = \argmin_{f \in \mathcal{F}} \max_{\group \in \groups}  
     \left(
     \tfrac{1}{n_\group}\sum_{i \in 
    \mathcal{S}_\group
     }  \ell(f(x_i), y_i) + R(f, \mathcal{S}_\group) \right)~.
\end{align*}}
For losses $\ell$ which are continuous and convex in the parameters of $f$, the optimal GDRO solution corresponds to the minimizer of a group-weighted objective:
$ \tfrac{1}{n} \sum_{i=1}^n w(\group_i) \cdot \ell (f(x_i), y_i)$,
though this is not in general true for nonconvex losses~\citep[see Prop.\ 1 of][and the remark immediately thereafter]{Sagawa2020Distributionally}.

Given the correspondence of GDRO (for convex loss functions) and IW to the optimization of group-weighted ERM objectives, we now investigate the joint roles of sample allocation and group re-weighting for estimating group-weighted risks.
For prediction function $f$, loss function $\ell$, and group weights $w: \groups \rightarrow \mathbb{R}^+$, let $\hat{L}(w, \alpha, n;f, \ell)$ be the random variable defined by:
\begin{align*}
    \hat{L}(w, \alpha, n;f, \ell) &:= \tfrac{1}{n}\sum_{i \in \mathcal{S}(\vec{\alpha}, n)} w(g_i) \cdot \ell(f(x_i),y_i) ~,
\end{align*}
where the randomness in $\hat{L}$ comes from the draws of $x_i, y_i$ from  $\mathcal{D}_{\group_{i}}$ according to procedure $\mathcal{S}(\vec{\alpha},n)$ (\cref{def:sampling_def}), as well as any randomness in $f$. 

The following proposition shows that group weights and allocations play complementary roles in risk function estimation. In particular, if  $w(\group)$ depends on the sampling allocations $\alpha_g$, then there are alternative group weights $w^*$ and allocation $\vec{\alpha}^*$ such that the alternative estimator has the same expected value but lower variance.

{\begin{restatable}[Weights and allocations]{proposition}{weightsallocations}
\label{prop:weights_and_allocations}
For any loss $\ell$, prediction function $f$ and group distributions $\mathcal{D}_\group$, there exist weights with
$w^{*}(\group) 
\propto \left(Var_{(x,y)\sim\mathcal{D}_\group}[\ell(f(x), y))]\right)^{-1/2}$
such that for any triplet $(\vec{\alpha}, w, n)$ with $\sum_\group \alpha_\group w(g) >0 $, if $w\not\appropto w^*$,\footnote{We use the symbol $\not\appropto$ to denote ``not approximately proportional to." The approximately part of this relation stems from finite and integer sample concerns; for example, the proposition holds if we consider $w \not\appropto w^*$ to mean $\exists \group \in \groups: | 1 - \frac{w(\group)}{w^*(\group)}| > \frac{|\groups|}{\alpha_\group n}$. }
\iftoggle{icml}{there exists an alternative $\vec{\alpha}^*$ with}
{there exists an alternative allocation $\vec{\alpha}^*$ such that}
\begin{align*}
    \mathbb{E}[\hat{L}( w^*,\vec{\alpha}^*, n; f, \ell )] & = \mathbb{E}[\hat{L}(w, \vec{\alpha}, n; f, \ell)] \\
    Var[\hat{L}( w^*, \vec{\alpha}^*,n; f, \ell)] & < Var[\hat{L}( w, \vec{\alpha}, n; f, \ell)]~.
\end{align*}
If $w(\group) > w^*(\group)$, $\alpha^*_\group > \alpha_\group$ and if $w(\group) < w^*(\group)$, $\alpha^*_\group < \alpha_\group$.
\end{restatable}

\iftoggle{icml}{\begin{proof} 
\newtext{(Sketch; full proof appears in \cref{sec:weights_and_allocations_proof}).
 For any deterministic weighting function $w : \groups \rightarrow \mathbb{R}^+$, there exists a  vector $\vec{\gamma}' \in \Delta^{|\groups|}$ with  $\gamma'_\group \propto w(\group) \alpha_\group $ such that
\begin{align*}
    \mathbb{E}[\hat{L}( w,\vec{\alpha}, n; f, \ell )]
    = c \cdot \mathbb{E}_{\group}\mathbb{E}_{(x,y) \sim \mathcal{D}_\group} [\ell( f(x), y)]&,
\end{align*}%
where $\group \sim \textrm{Multinomial}(\gamma')$ and 
 $c= \sum_\group \alpha_\group w(\group)$. For any fixed  $f$ and any  ``target distribution" defined by $\gamma'$, the ($\vec{\alpha}^*, w^*$) pair which minimizes the variance of the estimator, constrained so that $ w(\group)\alpha_\group = c \gamma'_\group \, \forall g$,  has weights $w^*$ with form given above. 
Since the original $(\alpha,w)$ pair satisfies this constraint, the pair $(\alpha^*,w^*)$ must satisfy
$Var[\hat{L}( w^*, \vec{\alpha}^*,n; f, \ell)] \leq Var[\hat{L}( w, \vec{\alpha}, n; f, \ell)]$, while the constraint ensures that $\mathbb{E}[\hat{L}( w^*,\vec{\alpha}^*, n; f, \ell )] = \mathbb{E}[\hat{L}(w, \vec{\alpha}, n; f, \ell)]$. 
}
\end{proof}}
{\begin{proof} (Sketch; the full proof appears in \cref{sec:weights_and_allocations_proof}.)
For any deterministic weighting function $w : \groups \rightarrow \mathbb{R}^+$, there exists a  vector $\vec{\gamma}' \in \Delta^{|\groups|}$ with entries $\gamma'_\group \propto w(\group) \alpha_\group $ such that
\begin{align*}
    \mathbb{E}[\hat{L}( w,\vec{\alpha}^*, n; f, \ell )]
    = \tfrac{1}{n} \sum_{(x_i,y_i,\group_i) \in \mathcal{S}} \mathbb{E}[w(\group_i) \ell( f(x_i), y_i)] = c \cdot \mathbb{E}_{\group \sim \textrm{Multinomial}(\gamma')}\mathbb{E}_{(x,y) \sim \mathcal{D}_\group} [\ell( f(x), y)]
\end{align*}
for constant
 $c= \sum_\group \alpha_\group w(\group)$. For any fixed  $f$ and any  ``target distribution" defined by $\gamma'$, the ($\vec{\alpha}^*, w^*$) pair which minimizes the variance of the estimator, constrained so that $ w(\group)\alpha_\group = c \gamma'_\group \, \forall g$,  has weights $w^*$ with form given above. 
Since the original $(\alpha,w)$ satisfy this constraint, the minimizing pair $(\alpha^*,w^*)$ must satisfy
$\mathbb{E}[\hat{L}( w^*,\vec{\alpha}^*, n; f, \ell )] = \mathbb{E}[\hat{L}(w, \vec{\alpha}, n; f, \ell)]$ , and 
$Var[\hat{L}( w^*, \vec{\alpha}^*,n; f, \ell)] \leq Var[\hat{L}( w, \vec{\alpha}, n; f, \ell)]$. %
\end{proof}}
Since the estimation of risk functions is a key component of learning, \cref{prop:weights_and_allocations}
illuminates
an interplay 
between the roles of sampling allocations and  group-weighting schemes like IW and GDRO. 
When allocations and weights are jointly maximized, the optimal allocation accounts for an implicit target distribution $\gamma'$  (defined above), which may vary by objective function. The optimal weights account for per-group variability $Var_{(x,y)\sim\mathcal{D}_\group}[\ell(f(x), y))]$.
In \cref{sec:experiments} we find that it can be advantageous to use IW and GDRO when some groups have small $\alpha_\group/\gamma_\group$; though the boost in accuracy is less than having an optimally allocated training set to begin with, and diminishes when all groups are appropriately represented in the training set allocation. 

\section{Allocating samples to minimize population-level risk}
\label{sec:allocating_samples}
\label{sec:scaling_model}
Having motivated the importance of group allocations, 
we now investigate the direct effects of training set allocations on group and population risks. 
Using a model of per-group performance as a function of allocations, we study the optimal allocations under a variety of settings.

\subsection{A Per-group Power-law Scaling Model}  
We model the impact of allocations on performance with scaling laws that describe per-group risks as a function of the number of data points from their respective group, as well as the total number of training instances. 
\begin{assumption}[Group risk scaling with allocation]
\label{ass:scaling}
The group risks
$\mathcal{R}(\hat{f};\mathcal{D}_\group) := \mathbb{E}_{(x,y) \sim {\mathcal{D}_\group}} [\ell(\hat{f}(x), y) ]$
scale approximately as the sum of inverse power functions on the number of samples from group $g$ and the total number of samples. That is, $\exists~ M_\group > 0$, $\sigma_\group, \tau_\group, \delta_\group  \geq 0$, and  $p,q>0 $ such that for a learning procedure which returns predictor $\hat{f}(\mathcal{S})$, and training set $\mathcal{S}$ with group sizes $n_\group\!\geq\! M_\group$:
\iftoggle{icml}
{\begin{align}
\label{eq:scaling}
  &   \mathcal{R}\bigl(\hat{f}\bigl(\mathcal{S}(\vec{\alpha},n)\bigr);\mathcal{D}_\group\bigr) 
      \approx  \, \groupfxn(\alpha_\group n, n; \sigma_\group, \tau_\group, \delta_\group, p, q) \hspace{.7em}  \forall \,\group \in \groups     \nonumber\\
&\groupfxn(n_\group, n; \sigma_\group, \tau_\group, \delta_\group, p, q)    := 
      \sigma_\group^2 n_\group^{-p} +\tau_\group^2 n^{-q} + \delta_\group ~.
\end{align}}
{\begin{align}
\label{eq:scaling}
     &\mathcal{R}\biggl(\hat{f}\bigl(\mathcal{S}(\vec{\alpha},n)\bigr);\mathcal{D}_\group\biggr) 
  :=
    \mathbb{E}_{(x,y) \sim \mathcal{D}_{\group}}\left[ \ell(\hat{f}(\mathcal{S})(x), y)\right]
      \approx  \, \groupfxn(\alpha_\group n, n; \sigma_\group, \tau_\group, \delta_\group, p, q) \hspace{.7em}  \forall \,\group \in \groups     \nonumber\\
&\groupfxn(n_\group, n; \sigma_\group, \tau_\group, \delta_\group, p, q)    := 
      \sigma_\group^2 n_\group^{-p} +\tau_\group^2 n^{-q} + \delta_\group ~.
\end{align}}
\end{assumption}
\cref{ass:scaling} is  similar to the scaling law in \citet{chen2018my}, but includes a $\tau_\group^2n^{-q}$ term to allow for data from other groups to influence the risk evaluated on group $\group$. 
It additionally requires that the same exponents  $p, q$ apply to each group, an assumption that underpins our theoretical results in~\Cref{sec:allocating_samples}.
%
We examine the extent to which \cref{ass:scaling} holds empirically in~\cref{sec:validating_scaling_laws}, and will modify Eq.~\eqref{eq:scaling} to include group-dependent terms $p_\group, q_\group$ when appropriate. The following examples give intuition into the form of  Eq.~\eqref{eq:scaling}.}

\iftoggle{icml}{\begin{example}}{\begin{example}[Split classifiers per group]}
 \label{ex:2}
 When separate models are trained for each group, using training data only from that group,
 we expect Eq.~\eqref{eq:scaling} to apply with $\tau_\group = 0$ $\forall \group \in \groups$. 
The parameter $p$ could derived through generalization bounds~\citep{boucheron2005theory}, or through modeling assumptions (\cref{ex:linear_model_with_dummies}).%
 \end{example}
\iftoggle{icml}{}{
 It is often advantageous to pool training data to learn a single classifier. In this case, model performance evaluated on group $\group$ will depend on both $n_\group$ and $n$, as the next examples show.}
 
\iftoggle{icml}{\begin{example}}
{\begin{example}[Groups irrelevant to prediction]}
When groups are irrelevant for prediction 
and the model class $\mathcal{F}$ correctly accounts for this, we expect Eq.~\eqref{eq:scaling} to apply with $\sigma_\group = 0$  $\forall \group \in \groups$.
 \end{example}

 \iftoggle{icml}
 { \begin{example}}
 {\begin{example}[Shared linear model with group-dependent intercepts]}
 \label{ex:linear_model_with_dummies}
 Consider a $(d+1)$-dimensional linear model, where two groups, $\{A,B\}$, share a weight vector $\beta$ and  features $x \sim \mathcal{N}(0, \Sigma_x)$, but the intercept varies by group:
 \begin{align*}
     y_i = \beta^\top x_i  + c_A \mathbb{I}[\group_i = A] + c_B \mathbb{I}[\group_i = B] + \mathcal{N}(0,\sigma^2).
 \end{align*}
As we show in \cref{appendix:ols_derivation},
 the ordinary least squares predictor 
 has group risks
$ \mathbb{E}_{(x,y) \sim \mathcal{D}_\group} 
      [
      (x^\top\hat{\beta} + \hat{c}_\group - y)^2 
      ]
      = \sigma^2 \left( 1 + 1/n_\group + O(d/n) \right),$
where the ${1}/{n_\group}$ arises because we need samples from group $\group$ to estimate the intercept $c_\group$, whereas samples from both groups help us estimate $\beta$.
\end{example}
\cref{ex:linear_model_with_dummies} suggests that in some settings, we can relate $\sigma_\group$ and $\tau_\group$ to `group specific' and `group agnostic' model components that affect performance for group $\group$. In general, the relationship between group sizes and group risks can be more nuanced. 
Data from different groups may be correlated, so that samples from groups similar to or different from $g$ have greater effect on $\mathcal{R}(\hat{f};\mathcal{D}_\group)$ (see \cref{sec:group_interactions}). Eq.~\ref{eq:scaling} is meant to capture the dominant effects of training set allocations on group risks and serves as our main structural assumption in the next section, where we study the allocation that minimizes the approximate \emph{population risk}.

\subsection{Optimal (w.r.t. Population Risk) Allocations}

We now study properties of the allocation that minimizes the approximated population risk:
\iftoggle{icml}
{\begin{align}
\label{eq:estimated_pop_risk}
 \hat{\mathcal{R}}(\vec{\alpha}, n):=&\sum_{\group \in \groups} \gamma_\group \groupfxn(\alpha_\group n, n; \sigma_\group, \tau_\group, \delta_\group,p, q) \\
    \approx&  \sum_{\group \in \groups} \gamma_\group \mathcal{R}(\hat{f}(\mathcal{S});\mathcal{D}_\group) = 
     \mathcal{R}(\hat{f}(\mathcal{S});\mathcal{D}). \nonumber
\end{align}}
{\begin{align}
\label{eq:estimated_pop_risk}
 \hat{\mathcal{R}}(\vec{\alpha}, n):=&\sum_{\group \in \groups} \gamma_\group \groupfxn(\alpha_\group n, n; \sigma_\group, \tau_\group, \delta_\group,p, q) 
    \approx  \sum_{\group \in \groups} \gamma_\group \mathcal{R}\left(\hat{f}(\mathcal{S});\mathcal{D}_\group\right) = 
     \mathcal{R}\left(\hat{f}(\mathcal{S});\mathcal{D}\right)~. 
\end{align}}
The following proposition lays the foundation for two corollaries which show that:
\iftoggle{icml}{
(1) when only the population prevalences $\vec{\gamma}$ vary between groups, the allocation that minimizes the approximate \emph{population risk} up-represents groups with small $\gamma_\group$; 
(2) for two groups with different scaling parameters $\sigma_\group$, the optimal allocation of the group with $\gamma_g < \tfrac{1}{2}$ is bounded by functions of $\sigma_A, \sigma_B$, and $\vec{\gamma}$.}
{
\begin{enumerate}
    \item[(1)] when only the population prevalences $\vec{\gamma}$ vary between groups, the allocation that minimizes the approximate \emph{population risk} up-represents groups with small $\gamma_\group$;
    \item[(2)] for two groups with different scaling parameters $\sigma_\group$, the optimal allocation of the group with $\gamma_g < \tfrac{1}{2}$ is bounded by functions of $\sigma_A, \sigma_B$, and $\vec{\gamma}$.
\end{enumerate}
}

\begin{restatable}{proposition}{alphastar}
\label{lem:alpha_star}
Given a population made up of disjoint groups  $\group \in \groups$ 
with population prevalences $\gamma_\group$, 
under the conditions of ~\cref{ass:scaling}, 
the allocation $\vec{\alpha}^* \in \Delta^{|\groups|}$ that minimizes the approximated population risk 
$\hat{\mathcal{R}}$ in eq.~\eqref{eq:estimated_pop_risk}
has elements:
\begin{align}
\label{eq:alpha_star}
\alpha_\group^* = \frac{\left(\gamma_\group \sigma^2_\group\right)^{1/(p+1)}}{\sum_{\group \in \groups} \left(\gamma_\group \sigma^2_\group\right)^{1/(p+1)}} ~.
\end{align}
If $\sigma_\group = 0 \ \forall \group \in \groups$, then any allocation in $\Delta^{|\groups|}$ minimizes $\hat{\mathcal{R}}$.
\end{restatable}
The proof of \cref{lem:alpha_star} appears in~\cref{sec:alpha_star_calculation}.
Note that 
$\vec{\alpha}^*$ does not depend on $n$, $\{\tau_\group\}_{\group \in \groups}$, or $q$; this will in general not hold if powers $p_\group$ differ by group.

We now study the form of $\vec{\alpha}^*$ under illustrative settings. 
%
\Cref{cor:same_sigma} shows that 
when the group scaling parameters $\sigma_g$ in Eq.~\eqref{eq:scaling} are equal across groups,
the allocation  that minimizes the approximate \emph{population risk} allocates samples to minority groups at higher than their \emph{population prevalences}. The proof of \cref{cor:same_sigma} appears in~\cref{sec:cor1_proof}.

\begin{restatable}[Many groups with equal $\sigma_\group$]{corollary}{equalsigma}
\label{cor:same_sigma}
When $\sigma_\group = \sigma > 0, \ \forall \group \in \groups$, the allocation that minimizes  
$\hat{\mathcal{R}}$ in Eq.~\eqref{eq:estimated_pop_risk}
satisfies
$\alpha_\group^* \geq \gamma_\group$ for any group with $\gamma_\group \leq \frac{1}{|\groups|}$.
\end{restatable}
 
This shows that the allocation that minimizes population risk can differ from the actual population prevalences $\vec{\gamma}$.
%
In fact, \cref{cor:same_sigma} asserts that 
near the allocation $\vec{\alpha} = \vec{\gamma}$, the marginal returns to additional data from group $\group$ are largest for groups with small $\alpha_\group$, enough so as to offset the small weight $\gamma_\group$
in Eq.~\eqref{eq:err_per_group}.
This result provides evidence \textit{against} the idea that small training set allocation to minority groups might comply with minimizing population risk as a result of a small
relative contribution to the population risk.

\textbf{Remark.} A  counterexample shows that $\alpha_\group^* \leq \gamma_\group$ does not hold for all $\group$ with $\gamma_\group > {1}/{|\groups|}$. Take $\vec{\gamma} = [.68,.30,.01,.01]$ and $p=1$;  Eq.~\eqref{eq:alpha_star} gives $\alpha_2^* > 0.3 = \gamma_2 > 1/4$. In general, whether group $\group$ with $\gamma_\group \geq {1}/{|\groups|}$ gets up- or down-sampled depends on the distribution of  $\vec{\gamma}$ across all groups. 

Complementing 
\iftoggle{icml}{the results of}{the investigation of
the role of the population proportions $\vec{\gamma}$
in} \cref{cor:same_sigma}, the next corollary
shows that
\iftoggle{icml}{}{the optimal allocation }$\vec{\alpha}^*$ generally depends on the relative values of $\sigma_\group$ between groups.
Inspecting Eq.~\eqref{eq:scaling} shows that $\sigma_\group $ defines a limit of performance: 
if $\sigma_\group^2$ is large, the only way to make the approximate risk for group $g$
small is to make $n_\group$ large.
\iftoggle{icml}{}{For two groups,
we can bound the optimal allocations $\vec{\alpha}^*$ in terms of $\{\sigma_\group\}_{\group \in \groups}$, and the population proportions $\vec{\gamma}$.}
\iftoggle{icml}{}{We let $A$ be the smaller of the two groups without loss of generality.} 
From Eq.~\eqref{eq:alpha_star}, we know that for two groups, $\alpha^*_A$ is increasing in $\tfrac{\sigma_A}{\sigma_B}$; \Cref{cor:two_groups_unequal_sigma} gives upper and lower bounds on $\alpha^*_A$ in terms of $\sigma_A$ and $\sigma_B$. 
\cref{cor:two_groups_unequal_sigma} is proved in \cref{sec:cor2_proof}.
\begin{restatable}[Unequal per-group constants]{corollary}{twogroupsunequalsigma}
\label{cor:two_groups_unequal_sigma}
For two groups $ \{A,B\}= \groups$ with $\gamma_A\!<\!\gamma_B$, and parameters $\sigma_A, \sigma_B\!>\!0$  in Eq.~\eqref{eq:scaling}, the  allocation of the smaller group $\alpha^*_A$  that minimizes 
$\hat{\mathcal{R}}$ in Eq.~\eqref{eq:estimated_pop_risk}
is upper and lower bounded as
\iftoggle{icml}
{\begin{align*}
        &\frac{\gamma_A ( \sigma_A^2)^{1/(p+1)}}{\gamma_A ( \sigma_A^2)^{1/(p+1)} + \gamma_B(\sigma_B^2)^{1/(p+1)} } < \alpha^*_A \\
      &\hspace{10em}  < \frac{(\sigma_A^2)^{1/(p+1)}}{(\sigma_A^2)^{1/(p+1)} + (\sigma_B^2)^{1/(p+1)}}~.
\end{align*}}
{\begin{align*}
        &\frac{\gamma_A ( \sigma_A^2)^{1/(p+1)}}{\gamma_A ( \sigma_A^2)^{1/(p+1)} + \gamma_B(\sigma_B^2)^{1/(p+1)} } < \alpha^*_A 
       < \frac{(\sigma_A^2)^{1/(p+1)}}{(\sigma_A^2)^{1/(p+1)} + (\sigma_B^2)^{1/(p+1)}}~.
\end{align*}}
When $\sigma_A \geq \sigma_B$, 
     $\alpha_A^* > \gamma_A$, and when 
     $\sigma_A \leq \sigma_B$, $\alpha_A^* < 1/2$.
\end{restatable}

Altogether, these results highlight key properties of training set allocations that minimize 
population risk.
Experiments in \cref{sec:experiments} give further insight into the values of weights and allocations for minimizing group and population risks and apply the scaling law model in real data settings.

\section{Empirical Results}
\label{sec:experiments}

Having shown the importance of training set allocations from a theoretical perspective, we now provide a complementary empirical investigation of this phenomenon. 
\iftoggle{icml}{}{Throughout our experiments, we use a diverse collection of datasets to give as full a picture of the empirical phenomena as possible (\cref{sec:datasets_short_description,tab:datasets}).} See \cref{sec:experiment_appendix} for full details on each experimental setup.\footnote{\newtext{Code to replicate the experiments is available at
\scriptsize{\url{https://github.com/estherrolf/representation-matters}.}}}

\iftoggle{icml}{
{\begin{table}[h]
\vspace{-1em}
\setlength{\tabcolsep}{3pt}
\renewcommand{\arraystretch}{1.2}
    \centering
        \caption{Brief description of datasets; details in \cref{sec:datasets_longer_desc}. 
        \label{tab:datasets}} 
        \scriptsize
    \begin{tabular}{@{ }llllll@{ }}
   \toprule
        dataset &  groups $\{A,B\}$& $\gamma_A$ &  $\min n_{g}$&  target label & loss metric 
        \\ \hline
        CIFAR-4 & \{animal, vehicle\} & 0.1 &
        10,000  & air & 0/1 loss\\
        ISIC & \{age $\!\geq\!55$, age $\!<\!55$\} & 0.43 & 4,092 & malignant & 1 - AUROC \\
        Goodreads & \{history, fantasy\}  & 0.38 & 50,000 &  book rating & $\ell_1$ loss\\
         Mooc & \{edu $\!\leq\!2^{\circ}$, edu $\!>\!2^{\circ}\}$ & 0.16 & 3,897   & certified  & 1 - AUROC\\
          Adult & \{female, male\} & 0.5 & 10,771  &   income $\!>\!50$K & 0/1 loss\\
          \bottomrule 
    \end{tabular} 
\end{table}}}{}

\iftoggle{icml}{
We use a wide range of datasets 
to give a full empirical characterization of the phenomena of interest (see \cref{tab:datasets}).
The CIFAR-4 dataset is comprised of bird, car, horse, and plane image instances from  CIFAR-10 \citep{krizhevsky2009learning}. The ISIC dataset contains images of skin lesions labelled as benign or malignant  \citep{codella2019skin}.
The Goodreads dataset consists of written book reviews and numerical ratings \citep{DBLP:conf/recsys/WanM18}. The Mooc dataset contains student demographic and participation data \citep{HMdata}. The Adult dataset  consists of demographic data from the 1994 Census \citep{Dua:2019}. For Adult we exclude groups from features (\Cref{sec:groups_as_features}).}{}

\iftoggle{icml}{}{
The first set of experiments (\cref{sec:navigating_tradoffs}) investigates group and population accuracies as a function of the training set allocation $\vec{\alpha}$ by sub-sampling different training set allocations for a fixed training set size.
We also study the amount to which 
importance weighting and group distributionally robust optimization
can increase group accuracies, complementing the results of~\cref{prop:weights_and_allocations} with an empirical perspective. 
The second set of experiments (\cref{sec:validating_scaling_laws}) uses a similar subsetting procedure to examine the fit of the scaling model proposed in \cref{sec:allocating_samples}. 
The third set of experiments (\cref{sec:informing_future_allocations}) investigates using the estimated scaling law fits to inform future sampling practices. We simulate using a small pilot training set to inform targeted dataset augmentation through collecting additional samples.
In \cref{sec:group_interactions}, we probe a setting where we might expect the scaling law to be too simplistic, exposing the need for more nuanced modelling in such settings.}
In contrast to \cref{sec:allocations_and_alternatives},
here losses are defined 
over sets of data; note that AUROC is not separable over  groups, and thus Eq.~\eqref{eq:err_per_group} does not apply for this metric.

\iftoggle{icml}{}{
\subsection{Datasets}
\label{sec:datasets_short_description}
Here we give a brief description of each dataset we use. For more details, see (\cref{tab:datasets,sec:datasets_longer_desc}).

\textbf{Modified CIFAR-4}. To create a dataset where we can ensure class balance across groups, we 
modify the CIFAR-10 dataset \citep{krizhevsky2009learning} by subsetting to the  bird, car, horse, and plane classes. We predict whether the image subject moves primarily by air (plane/bird) or land (car/horse) and group by whether the image contains an animal (bird/horse) or vehicle (car/plane); see \cref{fig:cifar_explanation1}. We set $\gamma=0.9$.

\textbf{ISIC.} The International Skin Imaging Collaboration dataset is a set of labeled images of skin lesions designed to aid development and standardization of imaging procedures for automatic detection of melanomas \citep{codella2019skin}. For our main analysis, we follow similar preprocessing steps to \citep{sohoni2020no}, removing any images with patches.
We predict whether a lesion is benign or malignant, and group instances by the approximate age of the patient of whom each photo was taken.

\textbf{Goodreads.} Given publicly available book reviews compiled from the Goodreads database \citep{DBLP:conf/recsys/WanM18}, we predict the  rating (1-5) corresponding to each review. We group instances by genre.

\textbf{Mooc.} The  HarvardX-MITx Person-Course Dataset contains student demographic and activity data from 2013 offerings of online courses \citep{HMdata}. We predict whether a student earned a certification in the course and we group instances by highest completed level of education.

\textbf{Adult.}  The Adult dataset, downloaded from the  UCI Machine Learning Repository \citep{Dua:2019} and originally taken from the 1994 Census database, has a prediction task of whether an individual's annual income is over $\$50,000$. We group instances by sex, codified as binary male/female, and exclude features that directly determine groups status. 
} 

\iftoggle{icml}{}
{\begin{table*}[]
\setlength{\tabcolsep}{6pt}
\renewcommand{\arraystretch}{1.2}
    \centering
        \caption{Brief description of datasets; details are given in \cref{sec:datasets_longer_desc}. 
        \label{tab:datasets}} 
        \footnotesize
    \begin{tabular}{@{ }lllllllll@{ }}
   \toprule
        dataset &  groups $\{A,B\}$& $\gamma_A$ &  $\min_g n_{g}$& $n_{\textrm{test}}$
        & target label & loss metric & main model used
        \\ \hline
        CIFAR-4 & \{animal, vehicle\} & 0.1 &
        10,000 & 4,000 &  air/land & 0/1 loss & resnet-18 \\
        ISIC & \{age $\geq55$, age $< 55$\} & 0.43 & 4,092 & 2,390 &  benign/malignant &  1 - AUROC & resnet-18 \\
        Goodreads & \{history, fantasy\}  & 0.38 & 50,000 & 25,000 &  book rating (1-5) & $\ell_1$ loss & logistic regression \\
         Mooc & \{edu $\leq 2^{\circ}$, edu $>2^{\circ}$\} & 0.16 & 3,897  & 6,032 &  certified &  1 - AUROC &random forest \\ 
          Adult & \{female, male\} & 0.5 & 10,771  & 16,281 &  income $> \$50$K & 0/1 loss & random forest \\ 
          \bottomrule
    \end{tabular}
\end{table*}}

\subsection{Allocation-aware Objectives vs. Ideal Allocations}
\label{sec:navigating_tradoffs}

\iftoggle{icml}{\begin{table*}[!t]
\iftoggle{icml}
{\renewcommand{\arraystretch}{1.2}}
{\renewcommand{\arraystretch}{1.2}}

    \centering
    \footnotesize
    \caption{\newtext{Estimated scaling parameters for Eq.~\eqref{eq:general_scaling_law}. Parentheses denote standard deviations estimated by the nonlinear least squares fit. Parameters are constrained so that  $\hat{\tau}_\group, \hat{\sigma}_\group, \hat{\delta}_\group \geq 0$ and $\hat{p}_\group, \hat{q}_\group \in [0,2]$.}}
    \label{tab:scaling_fits}
    \begin{tabular}{c|c|c|c|c|c|c|c}
    \toprule
        dataset & 
         $M_g$ &group $\group$ & $\hat{\sigma}_\group$ & $\hat{p}_\group$ & $\hat{\tau}_\group$ & $\hat{q}_\group$ & $\hat{\delta}_\group$  \\ \hline
         \multirow{2}{*}{CIFAR-4} & \multirow{2}{*}{500} & animal & 1.9 (0.12) & 0.47 (9.8e-04) & 4.5e-09 (1.8e+06) & 2.0 (0.0e+00) & 1.1e-03 (8.9e-06) \\ & & vehicle & 1.6 (0.19) & 0.54 (2.0e-03) & 3.2e-12 (1.1e+06) & 2.0 (0.0e+00) & 1.4e-03 (2.8e-06) \\
   \hline
  \multirow{2}{*}{ISIC} & \multirow{2}{*}{200} &  age $\geq 55$ & 0.61 (1.7e-03) & 0.20 (1.1e-03) & 1.7e-09 (1.9e+04) & 1.9 (0.0e+00) & 1.4e-15 (6.1e-04) \\ & &  age $< 55$ & 0.26 (9.3e-04) & 0.13 (0.012) & 0.61 (0.044) & 0.3 (7.5e-03) & 7.5e-11 (7.2e-03)\\
  \hline
    \multirow{2}{*}{Goodreads} & \multirow{2}{*}{2500} & history & 0.16 (1.2e-03) & 0.074 (2.5e-03) & 2.5 (0.058) & 0.37 (2.0e-04) & 0.41 (3.0e-03) \\ & & fantasy & 0.62 (0.69) & 0.020 (1.2e-03) & 3.1 (0.093) & 0.39 (1.9e-04) & 7.2e-21 (0.72)
   \\ \hline
 \multirow{2}{*}{Mooc} & \multirow{2}{*}{50} & edu $\leq 2^{\circ}$ & 0.08 (2.6e-05) & 0.14 (6.0e-03) & 0.73 (0.059) & 0.63 (4.8e-03) & 1.3e-15 (2.6e-04) \\ & & edu $> 2^{\circ}$ & 0.038 (6.2e-04) & 0.068 (6.3e-03) & 0.54 (6.5e-03) & 0.61 (9.8e-04) & 2.8e-12 (8.0e-04)\\
  \hline
 \multirow{2}{*}{Adult }  & \multirow{2}{*}{50} & female & 0.078 (0.051) & 0.018 (3.6e-03) & 0.43 (8.3e-03) & 0.59 (1.6e-03) & 8.0e-16 (0.052) \\ & & male & 0.066 (2.6e-05) & 0.21 (1.2e-03) & 0.47 (6.5e-03) & 0.50 (1.1e-03) & 0.16 (5.4e-06) \\
   \bottomrule
    \end{tabular}
\end{table*}}{}

\iftoggle{icml}{}{\iftoggle{icml}{\newcommand{\uplotscale}{.68}}{\newcommand{\uplotscale}{.31}}
\begin{figure*}[!tbh]
    \begin{subfigure}[b]{\uplotscale\columnwidth}
         \centering
         \includegraphics[width=\textwidth]{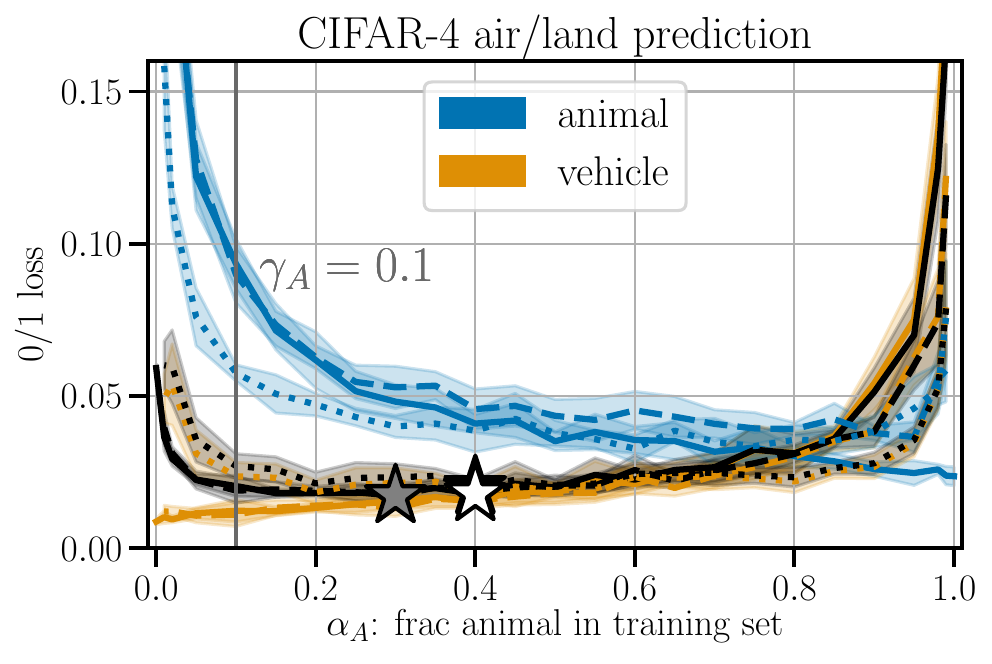}
    \end{subfigure}
    ~
    \begin{subfigure}[b]{\uplotscale\columnwidth}
        \includegraphics[width=\textwidth]{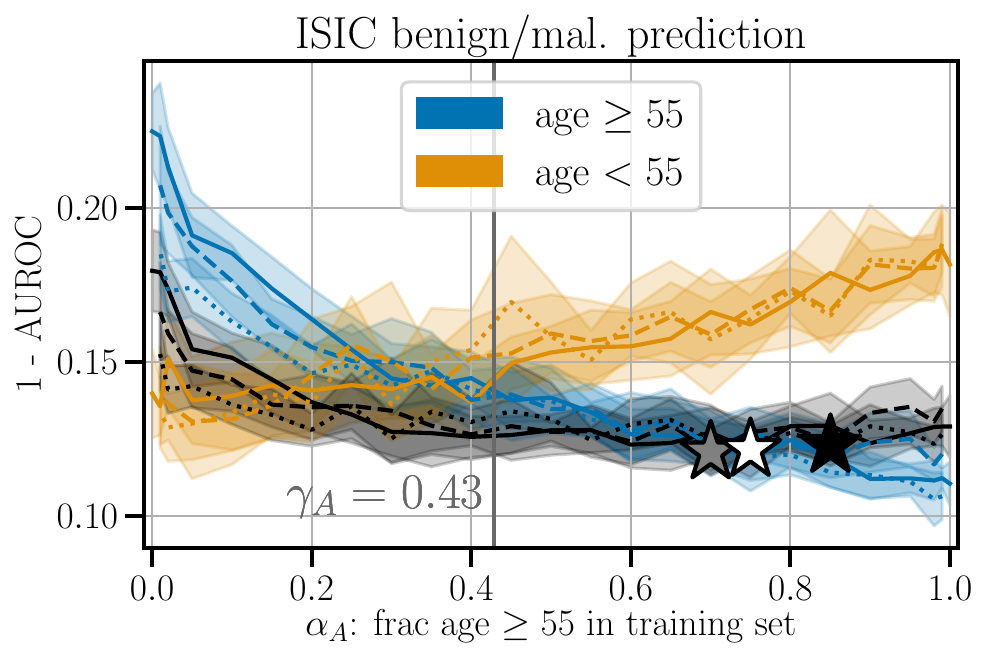}
    \end{subfigure}
      ~
    \begin{subfigure}[b]{\uplotscale\columnwidth}
        \includegraphics[width=\textwidth]{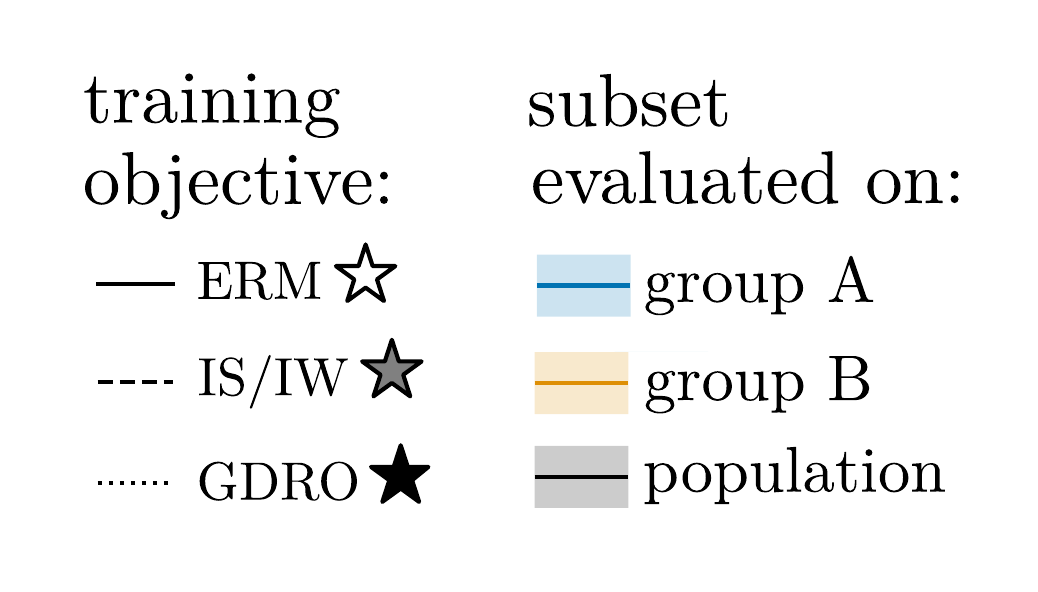}
        \vspace{.4em}
    \end{subfigure}
   \\[-1ex]
   \begin{subfigure}[b]{\uplotscale\columnwidth}
        \includegraphics[width=\textwidth]{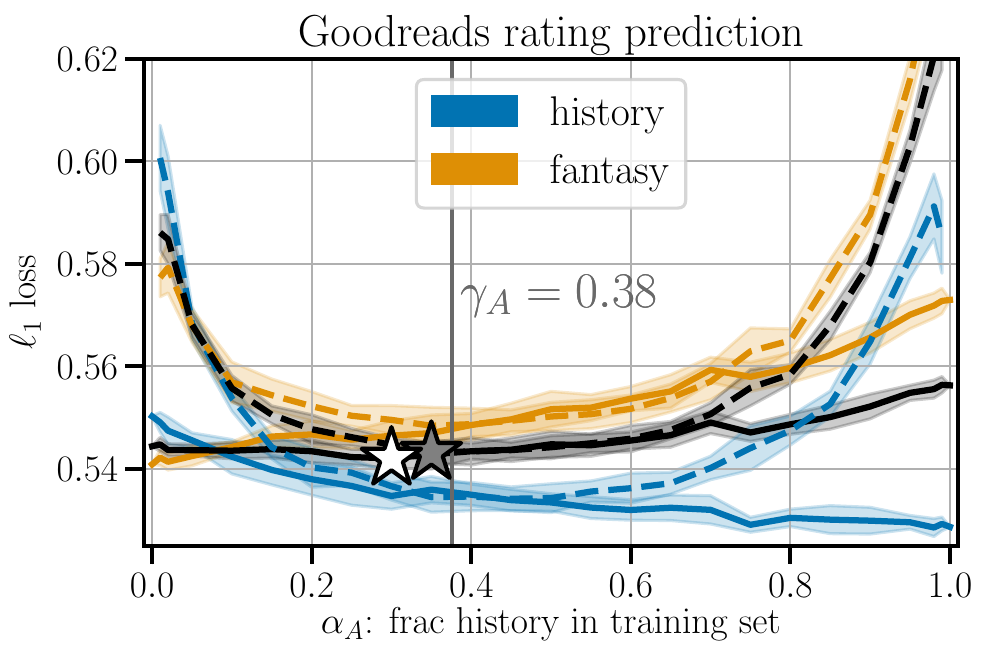}
    \end{subfigure}
    \hfill
    \begin{subfigure}[b]{\uplotscale\columnwidth}
        \includegraphics[width=\textwidth]{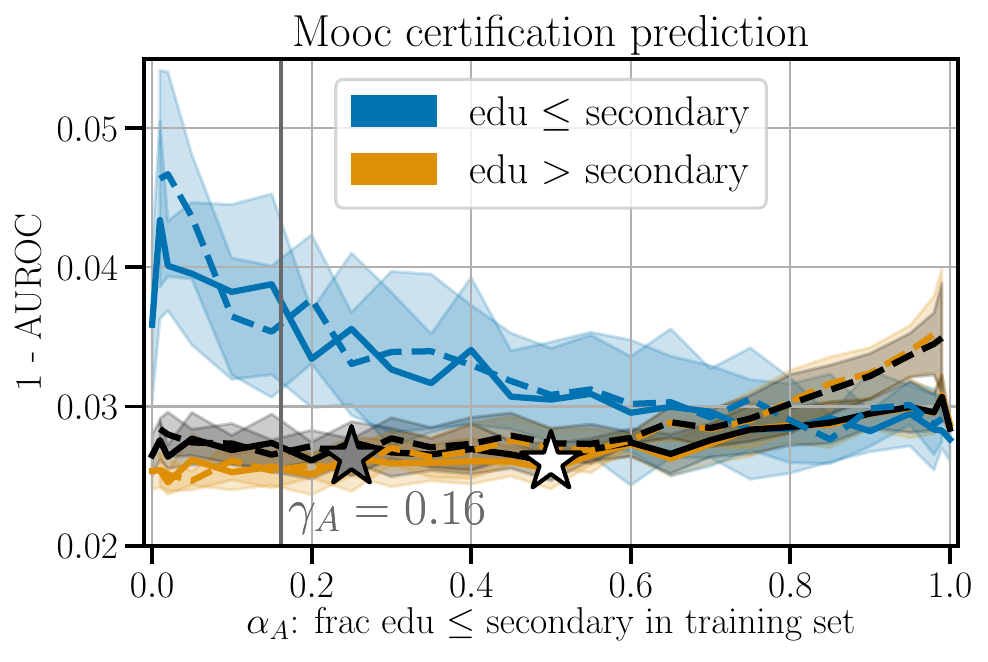}
     \end{subfigure} 
     \hfill
     \begin{subfigure}[b]{\uplotscale\columnwidth}
        \includegraphics[width=\textwidth]{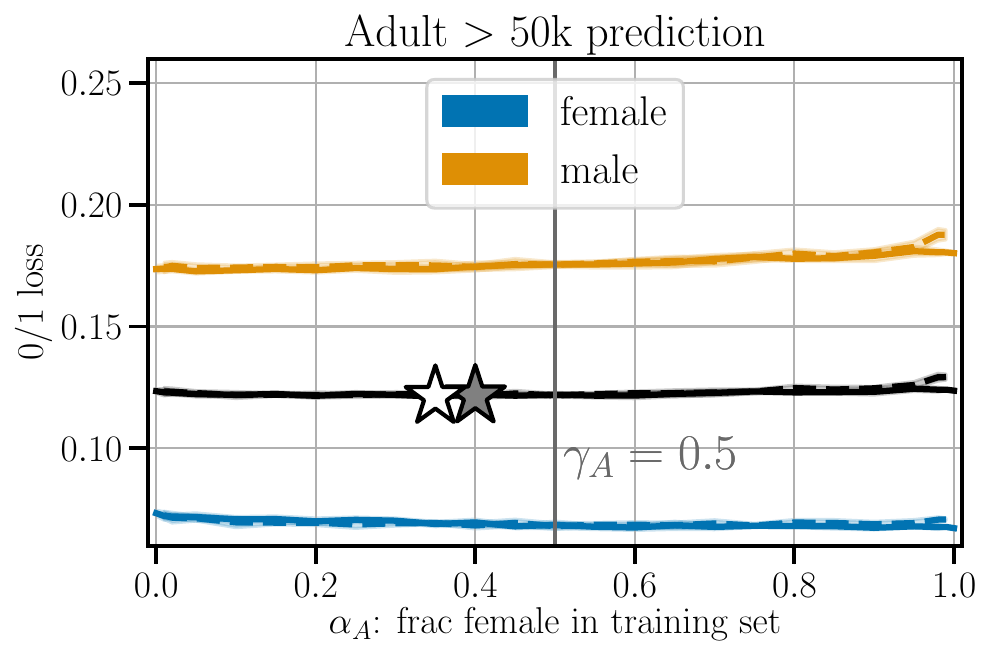}
     \end{subfigure}
        \caption{
        \label{fig:uplot_multifig}
        \iftoggle{icml}
        {Performance across $\vec{\alpha}$.
        Shaded regions: one stddev. above/ below mean (10 trials).
        Stars: population minima for each objective. The loss metrics reported (vertical axes) are the same within panels, while the training objectives differ across  solid and dashed lines.}
        {Performance across $\vec{\alpha}$.
        Shaded regions denote one stddev. above and below the mean over 10 trials.
        Stars indicate population minima for each objective (ERM: white, IS/IW:grey, GDRO:black). }
        }
\end{figure*}}

We first investigate (a) the change in group and population performance at different training set allocations, and (b) the extent to which optimizing the three objective functions defined in \cref{sec:IW_GDRO} decreases average and group errors.

For each dataset, we vary the training set allocations $\vec{\alpha}$ \iftoggle{icml}{}{between $(0,1)$ and $(1,0)$}  while fixing the training set size as $n = \min_{\group}  n_{\group}$ (see \cref{tab:datasets}) and evaluate the per-group and population losses on subsets of the heldout test sets.\iftoggle{icml}{\footnote{We pick models and parameters via a cross-validation procedure over a coarse grid of $\vec{\alpha}$; details are given in \cref{sec:experiments_hp}.}}{}
For the image classification tasks, we compare group-agnostic empirical risk minimization (ERM) to importance weighting (implemented via importance sampling (IS) batches following the findings of~\citet{buda2018systematic}) and group distributionally robust optimization (GDRO) with group-dependent regularization as \iftoggle{icml}{}{described }in~\citet{Sagawa2020Distributionally}.
For the non-image datasets, we implement importance weighting (IW) by weighting instances in the loss function during training, and do not compare to GDRO\iftoggle{icml}{.\footnote{The gradient-based algorithm of~\citet{Sagawa2020Distributionally} is not easily adaptable to the predictors we use for these datasets.}}{, as the gradient-based algorithm of~\citet{Sagawa2020Distributionally} is not easily adaptable to the predictors we use for these datasets.}
\iftoggle{icml}{}{We pick hyperparameters for each method based on cross-validation results over a coarse grid of $\vec{\alpha}$; for IS, IW, and GDRO, we allow the hyperparameters to vary with $\vec{\alpha}$; for ERM we choose a single hyperparameter setting for all $\vec{\alpha}$ values.}

\cref{fig:uplot_multifig} highlights the
importance of at least a minimal representation of each group in order to achieve low population loss (black curves) for all objectives. For CIFAR-4, the population loss increases sharply for $\alpha_A < 0.1$ and $\alpha_A > 0.8$, and for ISIC, 
when $\alpha_A < 0.2$.
While not as crucial for achieving low population losses for the remaining datasets, the \emph{optimal} allocations $\vec{\alpha}^*$ (stars)  do require a minimal representation of each group.
The $\vec{\alpha}^*$ are largely consistent across the training objectives (different star colors).
%
The population losses (black curves)  are largely consistent across mid-range values of $\alpha_A$ for all training objectives.
\iftoggle{icml}{This stands}
{The relatively shallow slopes of the black curves for $\alpha_A$ near $\alpha^*_A$ (stars) stand}
in contrast to the per-group losses (blue and orange curves), which can vary considerably as $\vec{\alpha}$ changes.
From the perspective of model evaluation, this reinforces a well-documented need for more comprehensive reporting of performance. 
From the view of  dataset design, 
this exposes an opportunity to choose allocations which optimize diverse evaluation objectives while maintaining low population loss.
Experiments in \cref{sec:informing_future_allocations} investigate this further.

Across the CIFAR-4 and ISIC tasks, GDRO  (dotted curves) is more effective than IS (dashed curves) at reducing per-group losses.
This is expected, as minimizing the largest loss of any group is the explicit objective of GDRO. 
\cref{fig:uplot_multifig} shows that GDRO can also improve the \emph{population loss} (see $\alpha_A > 0.7$ for CIFAR-4 and $\alpha_A < 0.2$ for ISIC)\iftoggle{icml}{.}
{ as a result of minimizing the worst group loss.}
\iftoggle{icml}{IW}{Importance weighting} (dashed curves) has little effect on performance for Mooc and Adult (random forest models), and actually increases the loss for Goodreads (multiclass logistic regression model).
%
%


For all the datasets we study, the advantages of using IS or GDRO are greatest when one group has very small training set allocation ($\alpha_A$ near $0$ or $1$).  When allocations are optimized (stars in \cref{fig:uplot_multifig}), the boost that these methods give over ERM diminishes. In light of \cref{prop:weights_and_allocations}, these results suggest that in practice, part of the value of such  methods is in compensating for sub-optimal allocations. 
We find, however, that
explicitly optimizing the maximum per-group loss with GDRO
can reduce population loss more effectively    
than directly accounting for allocations with IS.

\cref{sec:additional_model_results} shows that similar phenomena hold for different loss functions and models on the same dataset, though the exact $\vec{\alpha}^*$ can differ.
In \cref{sec:groups_as_features}, we show that losses of groups with small $\alpha_\group$ can degrade more severely when group attributes are included in the feature matrix, likely a result of the small number of samples from which to learn group-specific model components (see \cref{ex:linear_model_with_dummies}).

\subsection{Assessing the Scaling Law Fits}
\label{sec:validating_scaling_laws}

\iftoggle{icml}{}{}

For each dataset, we combine the results in \cref{fig:uplot_multifig} with 
extra subsetting runs where we vary both $n_g$ and $n$. 
\iftoggle{icml}{We}{From the combined results, we} use nonlinear least squares to estimate the parameters of modified scaling laws, where exponents can differ by group
\begin{align}
\label{eq:general_scaling_law}
   \textrm{loss}_\group \approx \sigma_\group^2 n_\group^{-p_\group} + \tau_\group^2 n^{-q_\group} + \delta_\group ~.
\end{align}
The estimated parameters of Eq.~\eqref{eq:general_scaling_law} given in \cref{tab:scaling_fits} capture different high-level phenomena across the five datasets.
For CIFAR-4, $\hat{\tau}_g \approx 0$ for both groups, indicating that most of the group performance is explained by $n_g$\iftoggle{icml}{.}{, the number of training instances from that group, whereas the total number of data points $n$, has less influence. }
For Goodreads, both $n_g$ and $n$ have influence in the fitted model, though $\hat{\tau}_\group$ and $\hat{q}_\group$ are larger than $\hat{\sigma}_g$ and $\hat{p}_\group$, respectively.
For ISIC, $\hat{\tau}_A \approx 0$ but  $\hat{\tau}_B \not\approx 0$, suggesting other-group data has little effect on the first group, but is beneficial to the latter.
For 
the non-image datasets (Goodreads, Mooc, and Adult), 
$0\!<\!\hat{\sigma}_\group \!<\!\hat{\tau}_\group$ and $\hat{p}_\group \!<\!\hat{q}_\group$ for all groups.

\iftoggle{icml}{}{
These results shed light on the applicability of the assumptions made in \cref{sec:scaling_model}.}
\cref{fig:empirical_scaling} in \cref{sec:scaling_law_fit_details}  shows that the fitted curves capture the overall trends of per-group losses as a function of $n$ and $n_g$. However, the assumptions of \cref{lem:alpha_star} and \cref{cor:same_sigma,cor:two_groups_unequal_sigma} (e.g., equal $p_\group$ for all $\group \in \groups$) are not always reflected in the empirical fits.
Results in \cref{sec:scaling_model} use Eq.~\eqref{eq:scaling} to describe optimal allocations under different hypothetical settings; we find that
\iftoggle{icml}{}{allowing the scaling parameters vary by group as in} Eq.~\eqref{eq:general_scaling_law} is more realistic in empirical settings.

\newtext{The estimated models describe the overall trends (\cref{fig:empirical_scaling}), but the parameter estimates are variable (\cref{tab:scaling_fits}), indicating that a range of parameters 
can fit the data, a well-known phenomenon in fitting power laws to data~\citep{clauset2009power}. While we caution against absolute or prescriptive interpretations based on the estimates given in \cref{tab:scaling_fits}, if such interpretations are desired~\citep{chen2018my}, we suggest evaluating variation due to subsetting patterns and comparing to alternative models such as log-normal and exponential fits~\citep[cf.][]{clauset2009power}.}

{\begin{figure*}[!hbt]
     \centering
     \includegraphics[width=.9\textwidth]{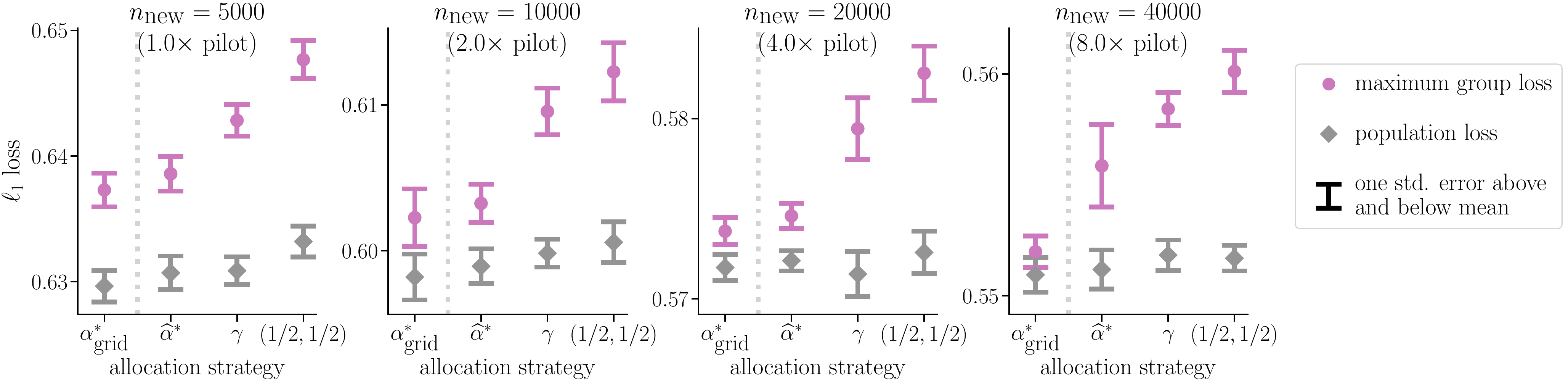}
         \footnotesize 
         \caption{ 
         Pilot sample experiment. Panels show the result of the three  allocations $\vec{\alpha} \in [\hat{\alpha}^*_\textrm{{minmax}}, \vec{\gamma}, (1/2,1/2)]$ for different sizes of the new training sets \newtext{compared with an $\alpha^*_{\textrm{grid}}$ baseline that minimizes the maximum group loss over a grid of resolution 0.01, averaged over the random trials.
         Purple circles indicate average maximum error over groups and grey diamonds indicate average population error.}
         Ranges denote standard errors taken over the 10 trials. \label{fig:pilot_results_goodread}   \label{fig:pilot_results_goodread_all}}
\end{figure*}}
\subsection{Targeted Data Collection with Fitted Scaling Laws}
\label{sec:informing_future_allocations}
We now study the use of scaling models fitted on a small pilot dataset to inform a subsequent data collection effort.
Given the results of \cref{sec:navigating_tradoffs}, we aim to collect a  training set that minimizes the maximum loss on any group.
This procedure goes beyond the descriptive use of the estimated scaling models in \cref{sec:validating_scaling_laws}; 
important considerations for operationalizing these findings are discussed below.


We perform this experiment with the Goodreads dataset, the largest of the five we study.
The pilot sample contains 2,500 instances from each group, drawn at random 
from the full training set.
We estimate the parameters of Eq.~\eqref{eq:general_scaling_law} using a procedure similar to that described in \cref{sec:validating_scaling_laws}.
For a new training set of size $n_{\textrm{new}}$,
we suggest an allocation to minimize the maximum forecasted loss of any group:
\begin{align*}
    \hat{\alpha}^*_\textrm{{minmax}} &= 
    \argmin_{\vec{\alpha} \in \Delta^2} \max_{\group \in \groups}
    \left(\hat{\sigma}_\group^2(\alpha_\group n_{\textrm{new}} )^{-\hat{p}_\group} + \hat{\tau}_\group^2 n_{\textrm{new}}^{-\hat{q}_\group} + \hat{\delta}_\group \right)
    . 
\end{align*}
For \iftoggle{icml}{$n_{\textrm{new}} \in [2\times, 4\times, 8\times]$}{$n_{\textrm{new}} \in [1\times, 2\times, 4\times, 8\times]$}, the pilot sample size, we simulate collecting a new training set by drawing $n_{\textrm{new}}$ fresh samples from the training set with allocation $\vec{\alpha} = \hat{\alpha}^*_\textrm{{minmax}}(n_{\textrm{new}})$.
We train a model on this sample (ERM objective) and evaluate on the test set.
For comparison, we also sample at $\vec{\alpha} = \vec{\gamma}$ (population proportions) and $\vec{\alpha} = (0.5,0.5)$ (equal allocation to both groups). 
We repeat the experiment, starting with the random instantiation of the pilot dataset, for ten trials. \newtext{As a point of comparison, we also compute the results for all $\alpha$ in a grid of resolution $0.01$, and denote the allocation value in this grid that minimizes the average maximum group loss over the ten trials as $\alpha^*_{\textrm{grid}}$.}

%

\iftoggle{icml}{
Among the three allocation strategies we compare, $\hat{\alpha}^*_{\textrm{minmax}}$  minimizes the average maximum loss over groups, across $n_{\textrm{new}}$  (\cref{fig:pilot_results_goodread}).
In contrast, $\hat{\alpha}^*_{\textrm{minmax}}$ does not increase the population loss (grey bars) over that of the other allocation strategies. 
This reinforces the finding of \cref{sec:navigating_tradoffs} 
and provides evidence that we can leverage information from a small initial sample to help raise the minimum accuracy over groups, without sacrificing population accuracy.}
{
Among the three allocation strategies we compare, $\vec{\alpha} = \hat{\alpha}^*_{\textrm{minmax}}$  minimizes the average maximum loss over groups, across $n_{\textrm{new}}$  (\cref{fig:pilot_results_goodread}).
Since per-group losses generally decrease with the increased allocations to that group, we expect the best minmax loss over groups to be achieved when the purple and grey bars meet in \cref{fig:pilot_results_goodread}. The allocation strategy 
$\vec{\alpha} = \hat{\alpha}^*_{\textrm{minmax}}$ does not quite achieve this; however, it does not increase the population loss over that of the other allocation strategies. 
This reinforces the finding of \cref{sec:navigating_tradoffs} that different per-group losses can be reached for similar population losses and provides evidence that we can navigate these possible outcomes by leveraging information from a small initial sample.}

While the results in \cref{fig:pilot_results_goodread} are promising, error bars highlight the variation across trials.
\newtext{
The variability in performance across trials for allocation baseline $\alpha^*_\textrm{grid}$ (which  is kept constant across the ten trials) is largely consistent with that of the other allocation sampling strategies examined (standard errors in \cref{fig:pilot_results_goodread}).}
However, the estimation of $\hat{\alpha}^*$ in each trial does introduce additional variation:
across the ten draws of the pilot data, the range of $\hat{\alpha}^*$ values for subsequent dataset size
\iftoggle{icml}{
$n_{\textrm{new}} = 10000$  is [1e-04,0.05], 
for $n_{\textrm{new}} = 20000$ it is [5e-05,0.14], 
and for $n_{\textrm{new}} = 40000$ it is [2e-05,0.82].  
}
{$n_{\textrm{new}} = 5000$ is [2e-04,0.04],  
for $n_{\textrm{new}} = 10000$ it is [1e-04,0.05], 
for $n_{\textrm{new}} = 20000$ it is [5e-05,0.14], 
and for $n_{\textrm{new}} = 40000$ it is [2e-05,0.82].  
}
{
Therefore, the estimated $\hat{\alpha}^*$ should be leveraged with caution, especially if the subsequent sample will be much larger than the pilot sample. Further caution should be taken if there may be distribution shifts between the pilot and subsequent samples. We suggest to interpret  estimated $\hat{\alpha}^*$ values as one signal among many that can inform a dataset design \newtext{in conjunction with current and emerging practices for ethical data collection (see \cref{sec:limitations_future_work}).}

\subsection{Interactions Between Groups}
\label{sec:group_interactions}
\newtext{
We now shift the focus of our analysis to explore potential between- and within- group interactions that are more nuanced than the scaling law in Eq.~\eqref{eq:scaling} provides for. The results  highlight the need for and encourage future work extending our analysis to more complex notions of groups (e.g., intersectional, continuous, or proxy groups). }

As discussed in \cref{sec:allocating_samples}, 
%
data from groups similar to or different from group $\group$ may have greater effect on $\mathcal{R}(\hat{f}(\mathcal{S});\mathcal{D}_g)$ compared to data drawn at random from the entire distribution.
We examine this possibility on the ISIC dataset, which is aggregated from different studies (\cref{sec:datasets_longer_desc}). 
We measure baseline performance of the model trained on data from all of the studies. 
We then remove one study at a time from the training set, retrain the model, and evaluate the change in performance for all studies in the test set.

\cref{fig:isic_loo_studies} shows the \iftoggle{icml}{}{percent }changes in performance due to leaving out studies from the training set. \iftoggle{icml}{}{The MSK and UDA studies are comprised of 5 and 2 sub-studies, respectively; \cref{fig:isic_loo_substudies} shows the results of leaving out each sub-study.}
Rows correspond to the study withheld from the training set and columns correspond to the study used for evaluation. Rows and columns are ordered by $\%$ malignancy. For \cref{fig:isic_loo_studies} this is the same as ordering by dataset size, SONIC being the largest study.

Consistent with our modelling assumptions and results so far, 
accuracies evaluated on group $g$ decrease as a result of removing group $g$ from the training set (diagonal entries of \cref{fig:isic_loo_groups}). However, additional patterns 
show more nuanced relationships between groups.
%

Positive values in the  \iftoggle{icml}{upper right region of \cref{fig:isic_loo_studies}}{upper right regions of \cref{fig:isic_loo_studies,fig:isic_loo_substudies}} show that 
excluding studies with low malignancy rates can raise performance evaluated on studies with high malignancy rates.
\newtext{
This could be partially due to differences in label distributions when removing certain studies from the training data.
Importantly,} this provides a counterpoint to an assumption implicit in  \cref{ass:scaling}, that group risks decrease in the 
total training set size $n$, 
regardless of the groups these $n$ instances belong to. 
To study more nuanced interactions between pairs $\group'\!\neq\! \group$, future work could
modify Eq.~\eqref{eq:scaling} by reparameterizing $r(\cdot)$ to directly account for $n_{\group'}$.



\iftoggle{icml}
{\begin{figure}[!tb]
     \centering
        \centering
         \includegraphics[width=.92\columnwidth]{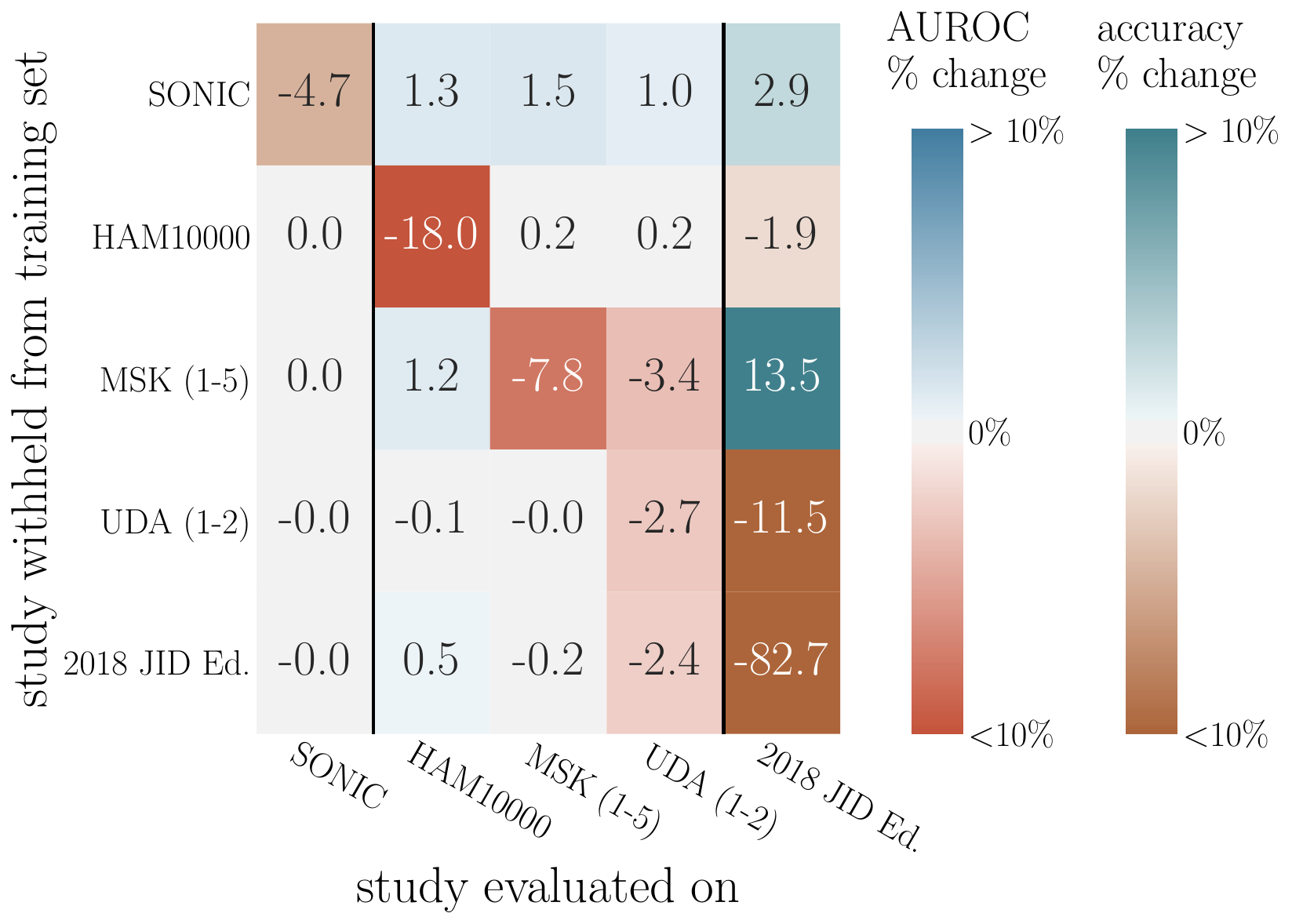}
         \caption{
          \label{fig:isic_loo_studies}
         \label{fig:isic_loo_groups}
         Percent change in performance (AUROC / accuracy) due to withholding a study from the training set.
         SONIC contains only  benign instances and 2018 JID Ed. contains only malignant instances; for these we report $\%$ change in binary accuracy. 
         }
\end{figure}}
{\begin{figure}[!thb]
     \centering
     \begin{subfigure}[b]{0.48\textwidth}
        \centering
         \includegraphics[width=\textwidth]{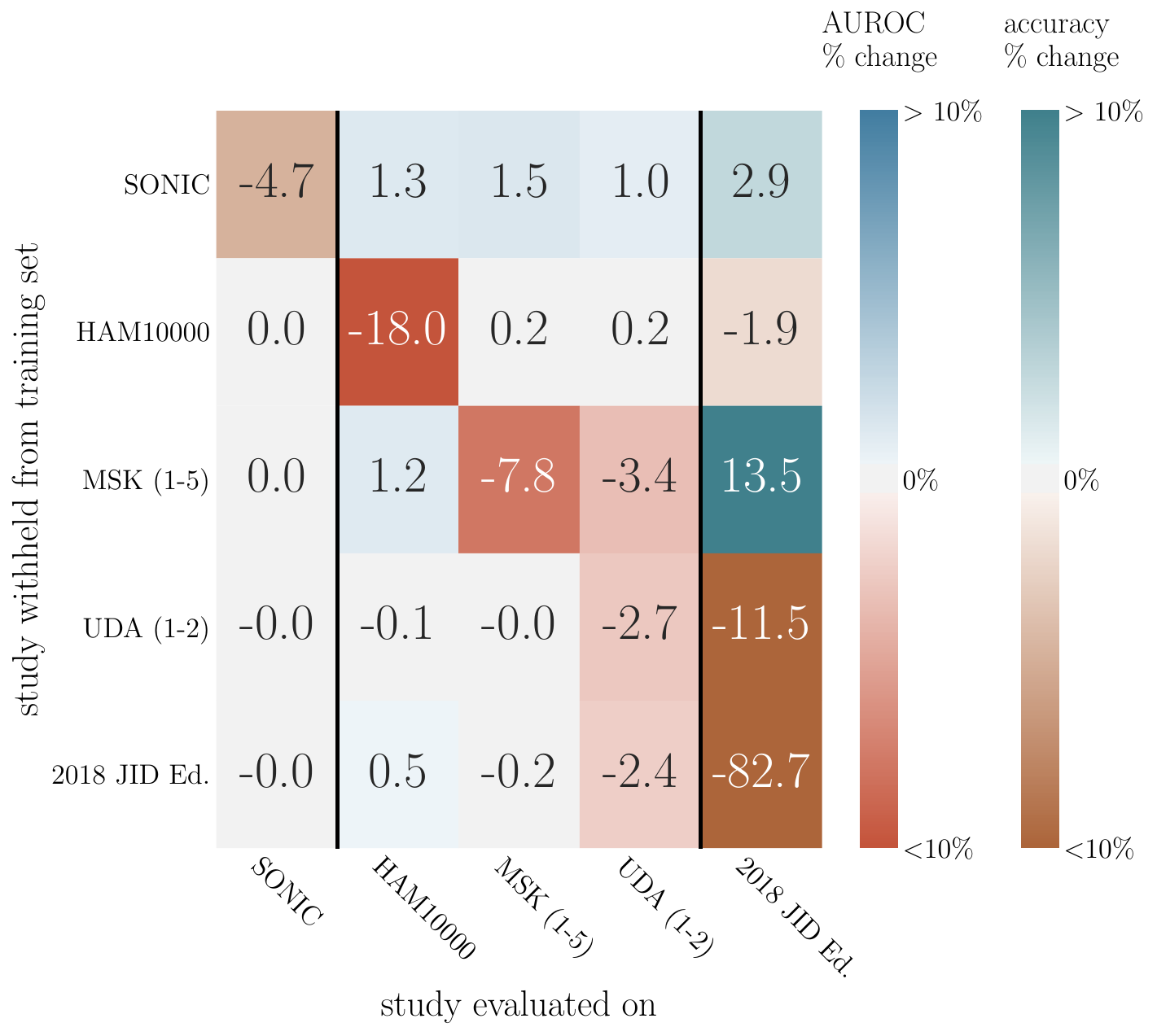}
         \caption{Grouped by study. \label{fig:isic_loo_studies}}
         \end{subfigure}
         \hfill
     \begin{subfigure}[b]{0.48\textwidth}
        \centering
         \includegraphics[width=\textwidth]{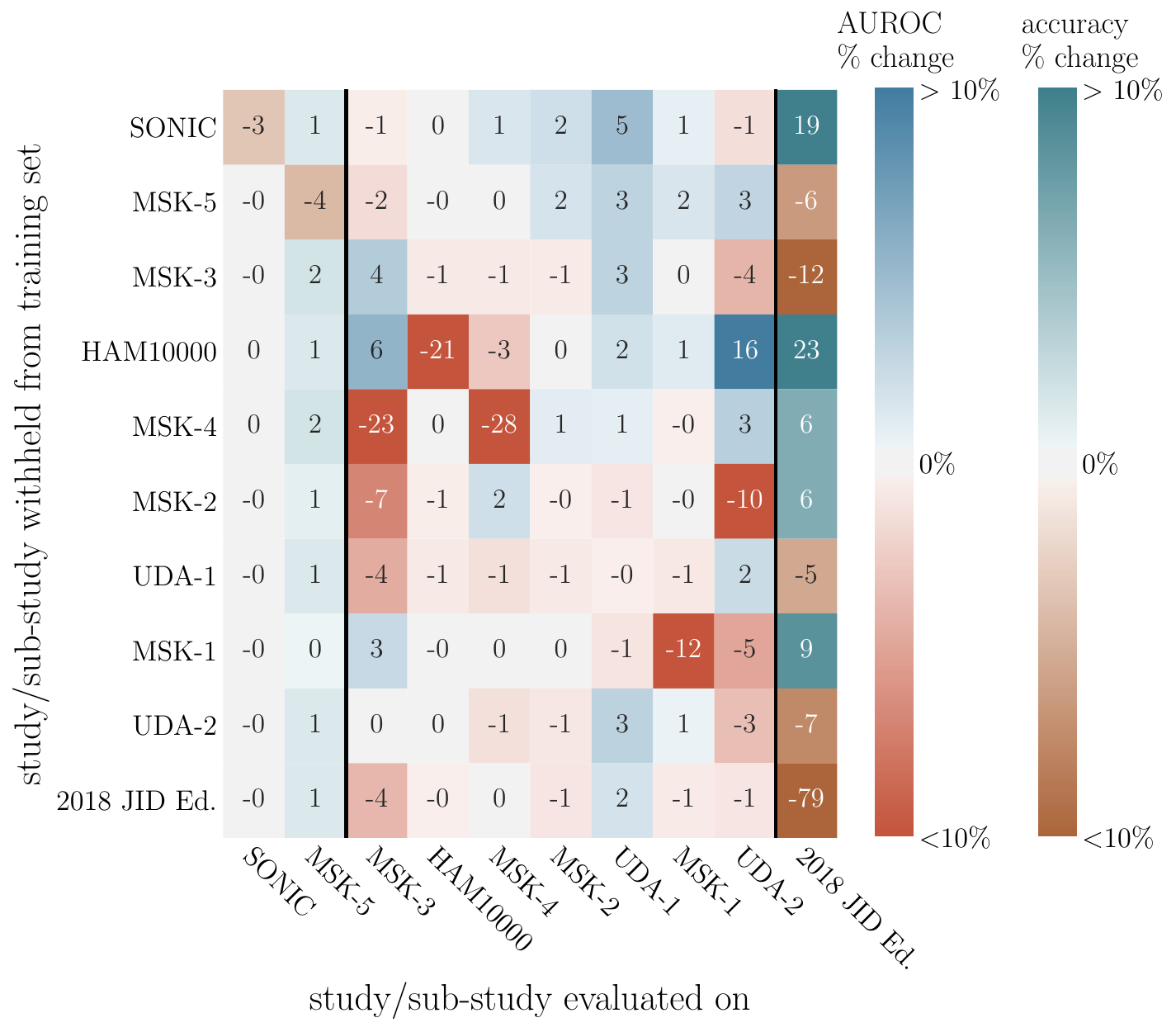}
         \caption{Grouped by sub-study.  \label{fig:isic_loo_substudies}}
         \end{subfigure}
         \caption{
         \label{fig:isic_loo_groups}
         Percent change in performance (AUROC / accuracy) due to withholding a study from the training set.
         Studies are ordered by $\%$ malignancy of the data in the evaluation set. 
         SONIC and MSK-5 contain all benign instances and the 2018 JID Editorial Images dataset has all malignant instances; for these we report $\%$ change in binary accuracy. For the remaining groups, we report $\%$ change in AUROC. \newtext{Note that the random training/test splits differ between (a) and (b), accounting for the differences in values for corresponding cells between the two figures.}}
\end{figure}}
Grouping by substudies within the UDA and MSK studies reveals that even within well defined groups, interactions between subgroups can arise. 
%
Negative off-diagonal entries in \cref{fig:isic_loo_substudies} suggest strong interactions between different groups, 
underscoring the importance of evaluating results across hierarchies and intersections of groups when feasible. 

Of the 16,965 images in the full training set, 7,401 are from the SONIC study.
When evaluating on all non-SONIC instances (like the evaluation set from the rest of the paper), withholding the SONIC study from the training set \iftoggle{icml}{}{(akin to the training set of the rest of the paper) }leads to higher AUROC (.905) than training on all studies (0.890).
This demonstrates that more data is not always better, especially if the distributional differences between the additional data and the target populations are not well accounted for.

\section{Generalizations, limitations, and future work}
\label{sec:limitations_future_work}

We study the ways in which group and population performance depend on the numerical allocations of discrete groups in training sets. 
While focusing on discrete groups allows us to derive meaningful results, understanding similar phenomena for intersectional groups and continuous notions of inclusion is an important next step.  Addressing the more nuanced relationships between the allocations of different data sources (\cref{sec:group_interactions}) is a first step in this direction.

We find that underrepresentation of groups in training data can limit group and population accuracies. 
\iftoggle{icml}
{
However, naive targeted data collection attempts can present undue burdens of surveillance or skirt consent \citep{paullada2020data}. 
When ML systems fail 
subpopulations due to measurement or construct validity issues, more comprehensive interventions  are needed \citep{Jacobs2019measurement}.}
{
However, assuming we can easily and ethically collect more data about any group is often naive, as there may be unintended consequences of upweighting certain groups in an objective function. Naive targeted data collection attempts can present undue burdens of surveillance or skirt consent \citep{paullada2020data}. 
When ML systems fail to represent
subpopulations due to measurement or construct validity issues, more comprehensive interventions  are needed \citep{Jacobs2019measurement}.}

\newtext{Our results expose key properties of sub-group representation in training data from a statistical sampling perspective, complementary to current and emerging practices for ethical, contextualized data collection and curation \cite{gebru2018datasheets,jo2020lessons,denton2020bringing,abebe2021narratives}}.
Studying the role of numerical allocation targets within ethical and context-aware data collection practices will be an important step toward operationalizing our findings.

Representation is a broad\iftoggle{icml}{,}{ and} often ambiguous concept \citep{chasalow2021representativeness}, and numerical allocation is an imperfect proxy of representation or inclusion.
\iftoggle{icml}{}{For example, if the data annotation process systematically misrepresents certain groups, optimizing allocations to maximize  accuracy with respect to those labels would not reflect our true goals across all groups. 
If prediction models are tailored to majority groups and are less relevant for smaller groups, an optimized allocation might allocate more data points to smaller groups as a remedy, when in reality a better model or feature representation is preferable.
True representation thus requires that each data instance measures the intended variables and their complexities in addition to numerical allocations. }
That said, if the optimal allocation for a certain group is well beyond that group's population proportion, this may be cause to reflect on why that is the case.
\iftoggle{icml}
{Future work could consider allocations as a lens for auditing the limits of prediction models from a data-focused perspective and
extend analysis to more objectives and loss functions (e.g. robustness or fairness objectives). }
{Future work could contextualize allocations as a way of auditing the limits of machine learning systems from a data-focused perspective.

By incorporating dataset collection as a design step in the learning procedure, 
we were able to assess and leverage the value of different data allocations toward reaching high group and population accuracies.
Extending this framework to other objectives and loss functions (e.g., robustness for out-of-distribution prediction and fairness objectives) will be an important area of future work. }


\section{Conclusions}
\label{sec:conclusion}
We demonstrate that representation in data is fundamental to training machine learning models that work for the entire population of interest.
By casting dataset design as part of the learning procedure, we can formalize the characteristics that training data must satisfy in order to reach the objectives and specifications of the overall machine learning system.
Empirical results bolster our theoretical findings and explore the nuances of real-data phenomena that call for domain dependent analyses in order to operationalize our general results in specific contexts.
Overall, we provide insight and constructive results toward understanding, prioritizing, and leveraging conscientious data collection for successful applications of machine learning.

\section*{Acknowledgments}

We thank Inioluwa Deborah Raji and Ludwig Schmidt for feedback at various stages of this work, and Andrea Bajcsy, Sara Fridovich-Keil, and Max Simchowitz for comments and suggestions during the editing of this manuscript.
We thank Nimit Sohoni and Jared Dunnmon for helpful discussions regarding the ISIC dataset. 

This material is based upon work supported by the NSF Graduate Research Fellowship under Grant No. DGE 1752814. ER acknowledges the support of a Google PhD Fellowship.
This research is generously supported in part by ONR awards N00014-20-1-2497 and N00014-18-1-2833, NSF CPS award 1931853, and the DARPA Assured Autonomy program (FA8750-18-C-0101).

\bibliographystyle{icml2021}
\bibliography{hetdata_reformatted.bib}

\begin{thebibliography}{48}
\providecommand{\natexlab}[1]{#1}
\providecommand{\url}[1]{\texttt{#1}}
\expandafter\ifx\csname urlstyle\endcsname\relax
  \providecommand{\doi}[1]{doi: #1}\else
  \providecommand{\doi}{doi: \begingroup \urlstyle{rm}\Url}\fi

\bibitem[Abebe et~al.(2021)Abebe, Aruleba, Birhane, Kingsley, Obaido, Remy, and
  Sadagopan]{abebe2021narratives}
Abebe, R., Aruleba, K., Birhane, A., Kingsley, S., Obaido, G., Remy, S.~L., and
  Sadagopan, S.
\newblock Narratives and counternarratives on data sharing in africa.
\newblock In \emph{Proceedings of the 2021 ACM Conference on Fairness,
  Accountability, and Transparency}, pp.\  329--341, 2021.

\bibitem[Abernethy et~al.(2020)Abernethy, Awasthi, Kleindessner, Morgenstern,
  and Zhang]{abernethy2020adaptive}
Abernethy, J., Awasthi, P., Kleindessner, M., Morgenstern, J., and Zhang, J.
\newblock Adaptive sampling to reduce disparate performance.
\newblock \emph{arXiv preprint arXiv:2006.06879}, 2020.

\bibitem[Boucheron et~al.(2005)Boucheron, Bousquet, and
  Lugosi]{boucheron2005theory}
Boucheron, S., Bousquet, O., and Lugosi, G.
\newblock Theory of classification: A survey of some recent advances.
\newblock \emph{ESAIM: Probability and Statistics}, 9:\penalty0 323--375, 2005.

\bibitem[Buda et~al.(2018)Buda, Maki, and Mazurowski]{buda2018systematic}
Buda, M., Maki, A., and Mazurowski, M.~A.
\newblock A systematic study of the class imbalance problem in convolutional
  neural networks.
\newblock \emph{Neural Networks}, 106:\penalty0 249--259, 2018.

\bibitem[Buolamwini \& Gebru(2018)Buolamwini and Gebru]{buolamwini2018gender}
Buolamwini, J. and Gebru, T.
\newblock Gender shades: Intersectional accuracy disparities in commercial
  gender classification.
\newblock In \emph{Conference on Fairness, Accountability and Transparency},
  pp.\  77--91, 2018.

\bibitem[Byrd \& Lipton(2019)Byrd and Lipton]{byrd2019effect}
Byrd, J. and Lipton, Z.~C.
\newblock What is the effect of importance weighting in deep learning?
\newblock In Chaudhuri, K. and Salakhutdinov, R. (eds.), \emph{Proceedings of
  the 36th International Conference on Machine Learning, {ICML} 2019, 9-15 June
  2019, Long Beach, California, {USA}}, volume~97 of \emph{Proceedings of
  Machine Learning Research}, pp.\  872--881. {PMLR}, 2019.
\newblock URL \url{http://proceedings.mlr.press/v97/byrd19a.html}.

\bibitem[Chasalow \& Levy(2021)Chasalow and
  Levy]{chasalow2021representativeness}
Chasalow, K. and Levy, K.
\newblock Representativeness in statistics, politics, and machine learning.
\newblock \emph{arXiv preprint arXiv:2101.03827}, 2021.

\bibitem[Chawla et~al.(2002)Chawla, Bowyer, Hall, and
  Kegelmeyer]{chawla2002smote}
Chawla, N.~V., Bowyer, K.~W., Hall, L.~O., and Kegelmeyer, W.~P.
\newblock {SMOTE}: {S}ynthetic minority over-sampling technique.
\newblock \emph{Journal of artificial intelligence research}, 16:\penalty0
  321--357, 2002.

\bibitem[Chen et~al.(2018)Chen, Johansson, and Sontag]{chen2018my}
Chen, I.~Y., Johansson, F.~D., and Sontag, D.~A.
\newblock Why is my classifier discriminatory?
\newblock In Bengio, S., Wallach, H.~M., Larochelle, H., Grauman, K.,
  Cesa{-}Bianchi, N., and Garnett, R. (eds.), \emph{Advances in Neural
  Information Processing Systems 31: Annual Conference on Neural Information
  Processing Systems 2018, NeurIPS 2018, December 3-8, 2018, Montr{\'{e}}al,
  Canada}, pp.\  3543--3554, 2018.
\newblock URL
  \url{https://proceedings.neurips.cc/paper/2018/hash/1f1baa5b8edac74eb4eaa329f14a0361-Abstract.html}.

\bibitem[Clauset et~al.(2009)Clauset, Shalizi, and Newman]{clauset2009power}
Clauset, A., Shalizi, C.~R., and Newman, M.~E.
\newblock Power-law distributions in empirical data.
\newblock \emph{SIAM {R}eview}, 51\penalty0 (4):\penalty0 661--703, 2009.

\bibitem[Codella et~al.(2019)Codella, Rotemberg, Tschandl, Celebi, Dusza,
  Gutman, Helba, Kalloo, Liopyris, Marchetti, et~al.]{codella2019skin}
Codella, N., Rotemberg, V., Tschandl, P., Celebi, M.~E., Dusza, S., Gutman, D.,
  Helba, B., Kalloo, A., Liopyris, K., Marchetti, M., et~al.
\newblock Skin lesion analysis toward melanoma detection 2018: A challenge
  hosted by the international skin imaging collaboration ({ISIC}).
\newblock \emph{arXiv preprint arXiv:1902.03368}, 2019.

\bibitem[Codella et~al.(2017)Codella, Gutman, Celebi, Helba, Marchetti, Dusza,
  Kalloo, Liopyris, Mishra, Kittler, and Halpern]{codella2018skin}
Codella, N. C.~F., Gutman, D., Celebi, M.~E., Helba, B., Marchetti, M.~A.,
  Dusza, S.~W., Kalloo, A., Liopyris, K., Mishra, N.~K., Kittler, H., and
  Halpern, A.
\newblock Skin lesion analysis toward melanoma detection: {A} challenge at the
  2017 international symposium on biomedical imaging ({ISBI}), hosted by the
  international skin imaging collaboration {(ISIC)}.
\newblock \emph{arXiv preprint}, 2017.

\bibitem[Cole et~al.(2013)Cole, Gkatzelis, and Goel]{cole2012mechanism}
Cole, R., Gkatzelis, V., and Goel, G.
\newblock Mechanism design for fair division: Allocating divisible items
  without payments.
\newblock In \emph{Proceedings of the Fourteenth ACM Conference on Electronic
  Commerce}, EC '13, pp.\  251–268, New York, NY, USA, 2013. Association for
  Computing Machinery.

\bibitem[Cortes et~al.(2010)Cortes, Mansour, and Mohri]{cortes2010learning}
Cortes, C., Mansour, Y., and Mohri, M.
\newblock Learning bounds for importance weighting.
\newblock In Lafferty, J.~D., Williams, C. K.~I., Shawe{-}Taylor, J., Zemel,
  R.~S., and Culotta, A. (eds.), \emph{Advances in Neural Information
  Processing Systems 23: 24th Annual Conference on Neural Information
  Processing Systems 2010. 
  Vancouver, British Columbia, Canada}, pp.\  442--450. Curran Associates,
  Inc., 2010.
\newblock URL
  \url{https://proceedings.neurips.cc/paper/2010/hash/59c33016884a62116be975a9bb8257e3-Abstract.html}.

\bibitem[Denton et~al.(2020)Denton, Hanna, Amironesei, Smart, Nicole, and
  Scheuerman]{denton2020bringing}
Denton, E., Hanna, A., Amironesei, R., Smart, A., Nicole, H., and Scheuerman,
  M.~K.
\newblock Bringing the people back in: Contesting benchmark machine learning
  datasets.
\newblock \emph{arXiv preprint arXiv:2007.07399}, 2020.

\bibitem[Dotan \& Milli(2020)Dotan and Milli]{dotan2019value}
Dotan, R. and Milli, S.
\newblock Value-laden disciplinary shifts in machine learning.
\newblock In \emph{Proceedings of the 2020 Conference on Fairness,
  Accountability, and Transparency}, pp.\  294--294, 2020.

\bibitem[Dua \& Graff(2017)Dua and Graff]{Dua:2019}
Dua, D. and Graff, C.
\newblock {UCI} machine learning repository, 2017.
\newblock URL \url{http://archive.ics.uci.edu/ml}.

\bibitem[Dwork et~al.(2018)Dwork, Immorlica, Kalai, and
  Leiserson]{dwork2018decoupled}
Dwork, C., Immorlica, N., Kalai, A.~T., and Leiserson, M.
\newblock Decoupled classifiers for group-fair and efficient machine learning.
\newblock In \emph{Conference on Fairness, Accountability and Transparency},
  pp.\  119--133, 2018.

\bibitem[Gebru(2020)]{jo2020lessons}
Gebru, T.
\newblock Lessons from archives: Strategies for collecting sociocultural data
  in machine learning.
\newblock In Gupta, R., Liu, Y., Tang, J., and Prakash, B.~A. (eds.),
  \emph{{KDD} '20: The 26th {ACM} {SIGKDD} Conference on Knowledge Discovery
  and Data Mining, Virtual Event, CA, USA, August 23-27, 2020}, pp.\  3609.
  {ACM}, 2020.
\newblock URL \url{https://dl.acm.org/doi/10.1145/3394486.3409559}.

\bibitem[Gebru et~al.(2018)Gebru, Morgenstern, Vecchione, Vaughan, Wallach,
  Daum{\'e}~III, and Crawford]{gebru2018datasheets}
Gebru, T., Morgenstern, J., Vecchione, B., Vaughan, J.~W., Wallach, H.,
  Daum{\'e}~III, H., and Crawford, K.
\newblock Datasheets for datasets.
\newblock \emph{arXiv preprint arXiv:1803.09010}, 2018.

\bibitem[Ghorbani \& Zou(2019)Ghorbani and Zou]{ghorbani2019data}
Ghorbani, A. and Zou, J.~Y.
\newblock Data shapley: Equitable valuation of data for machine learning.
\newblock In Chaudhuri, K. and Salakhutdinov, R. (eds.), \emph{Proceedings of
  the 36th International Conference on Machine Learning, {ICML} 2019, 9-15 June
  2019, Long Beach, California, {USA}}, volume~97 of \emph{Proceedings of
  Machine Learning Research}, pp.\  2242--2251. {PMLR}, 2019.
\newblock URL \url{http://proceedings.mlr.press/v97/ghorbani19c.html}.

\bibitem[Habib et~al.(2019)Habib, Karmakar, and Yearwood]{habib2019impact}
Habib, A., Karmakar, C., and Yearwood, J.
\newblock Impact of {ECG} dataset diversity on generalization of {CNN} model
  for detecting {QRS} complex.
\newblock \emph{IEEE {A}ccess}, 7:\penalty0 93275--93285, 2019.

\bibitem[Haixiang et~al.(2017)Haixiang, Yijing, Shang, Mingyun, Yuanyue, and
  Bing]{haixiang2017learning}
Haixiang, G., Yijing, L., Shang, J., Mingyun, G., Yuanyue, H., and Bing, G.
\newblock Learning from class-imbalanced data: Review of methods and
  applications.
\newblock \emph{Expert Systems with Applications}, 73:\penalty0 220--239, 2017.

\bibitem[HarvardX(2014)]{HMdata}
HarvardX.
\newblock {HarvardX Person-Course Academic Year 2013 De-Identified dataset,
  version 3.0}, 2014.
\newblock URL \url{https://doi.org/10.7910/DVN/26147}.

\bibitem[Hashimoto et~al.(2018)Hashimoto, Srivastava, Namkoong, and
  Liang]{hashimoto2018fairness}
Hashimoto, T.~B., Srivastava, M., Namkoong, H., and Liang, P.
\newblock Fairness without demographics in repeated loss minimization.
\newblock In Dy, J.~G. and Krause, A. (eds.), \emph{Proceedings of the 35th
  International Conference on Machine Learning, {ICML} 2018,
  Stockholmsm{\"{a}}ssan, Stockholm, Sweden, July 10-15, 2018}, volume~80 of
  \emph{Proceedings of Machine Learning Research}, pp.\  1934--1943. {PMLR},
  2018.
\newblock URL \url{http://proceedings.mlr.press/v80/hashimoto18a.html}.

\bibitem[Hofmanninger et~al.(2020)Hofmanninger, Prayer, Pan, R{\"o}hrich,
  Prosch, and Langs]{hofmanninger2020automatic}
Hofmanninger, J., Prayer, F., Pan, J., R{\"o}hrich, S., Prosch, H., and Langs,
  G.
\newblock Automatic lung segmentation in routine imaging is primarily a data
  diversity problem, not a methodology problem.
\newblock \emph{European Radiology Experimental}, 4\penalty0 (1):\penalty0
  1--13, 2020.

\bibitem[Hu et~al.(2018)Hu, Niu, Sato, and Sugiyama]{hu2018does}
Hu, W., Niu, G., Sato, I., and Sugiyama, M.
\newblock Does distributionally robust supervised learning give robust
  classifiers?
\newblock In Dy, J.~G. and Krause, A. (eds.), \emph{Proceedings of the 35th
  International Conference on Machine Learning, {ICML} 2018,
  Stockholmsm{\"{a}}ssan, Stockholm, Sweden, July 10-15, 2018}, volume~80 of
  \emph{Proceedings of Machine Learning Research}, pp.\  2034--2042. {PMLR},
  2018.
\newblock URL \url{http://proceedings.mlr.press/v80/hu18a.html}.

\bibitem[Iosifidis \& Ntoutsi(2018)Iosifidis and Ntoutsi]{iosifidis2018dealing}
Iosifidis, V. and Ntoutsi, E.
\newblock Dealing with bias via data augmentation in supervised learning
  scenarios.
\newblock In \emph{Proceedings of the Workshop on Bias in Information,
  Algorithms}, pp.\  24--29, 2018.

\bibitem[Jacobs \& Wallach(2019)Jacobs and Wallach]{Jacobs2019measurement}
Jacobs, A.~Z. and Wallach, H.
\newblock Measurement and fairness.
\newblock \emph{arXiv preprint arXiv:1912.05511}, 2019.

\bibitem[Kim et~al.(2019)Kim, Ghorbani, and Zou]{kim2019multiaccuracy}
Kim, M.~P., Ghorbani, A., and Zou, J.
\newblock Multiaccuracy: Black-box post-processing for fairness in
  classification.
\newblock In \emph{Proceedings of the 2019 AAAI/ACM Conference on AI, Ethics,
  and Society}, pp.\  247--254, 2019.

\bibitem[Krizhevsky(2009)]{krizhevsky2009learning}
Krizhevsky, A.
\newblock Learning multiple layers of features from tiny images.
\newblock Technical report, University of Toronto, 2009.

\bibitem[Lohr(2009)]{lohr2009sampling}
Lohr, S.~L.
\newblock \emph{Sampling: design and analysis}.
\newblock Nelson Education, 2009.

\bibitem[Neyman(1934)]{neyman1934}
Neyman, J.
\newblock On the two different aspects of the representative method: The method
  of stratified sampling and the method of purposive selection.
\newblock \emph{Journal of the Royal Statistical Society}, 97\penalty0
  (4):\penalty0 558--625, 1934.

\bibitem[Paullada et~al.(2020)Paullada, Raji, Bender, Denton, and
  Hanna]{paullada2020data}
Paullada, A., Raji, I.~D., Bender, E.~M., Denton, E., and Hanna, A.
\newblock Data and its (dis) contents: A survey of dataset development and use
  in machine learning research.
\newblock \emph{arXiv preprint arXiv:2012.05345}, 2020.

\bibitem[Pukelsheim(2006)]{PukelsheimFriedrich2006Odoe}
Pukelsheim, F.
\newblock \emph{Optimal design of experiments}.
\newblock Classics in applied mathematics ; 50. Society for Industrial and
  Applied Mathematics, classic ed. edition, 2006.

\bibitem[Raji \& Buolamwini(2019)Raji and Buolamwini]{raji2019actionable}
Raji, I.~D. and Buolamwini, J.
\newblock Actionable auditing: Investigating the impact of publicly naming
  biased performance results of commercial ai products.
\newblock In \emph{Proceedings of the 2019 AAAI/ACM Conference on AI, Ethics,
  and Society}, pp.\  429--435, 2019.

\bibitem[Ryu et~al.(2017)Ryu, Adam, and Mitchell]{ryu2017inclusivefacenet}
Ryu, H.~J., Adam, H., and Mitchell, M.
\newblock Inclusivefacenet: Improving face attribute detection with race and
  gender diversity.
\newblock \emph{arXiv preprint arXiv:1712.00193}, 2017.

\bibitem[Sagawa et~al.(2020)Sagawa, Koh, Hashimoto, and
  Liang]{Sagawa2020Distributionally}
Sagawa, S., Koh, P.~W., Hashimoto, T.~B., and Liang, P.
\newblock Distributionally robust neural networks.
\newblock In \emph{8th International Conference on Learning Representations,
  {ICLR} 2020, Addis Ababa, Ethiopia, April 26-30, 2020}. OpenReview.net, 2020.
\newblock URL \url{https://openreview.net/forum?id=ryxGuJrFvS}.

\bibitem[Shankar et~al.(2017)Shankar, Halpern, Breck, Atwood, Wilson, and
  Sculley]{shankar2017no}
Shankar, S., Halpern, Y., Breck, E., Atwood, J., Wilson, J., and Sculley, D.
\newblock No classification without representation: Assessing geodiversity
  issues in open data sets for the developing world.
\newblock \emph{arXiv preprint arXiv:1711.08536}, 2017.

\bibitem[Sohoni et~al.(2020)Sohoni, Dunnmon, Angus, Gu, and
  R{\'e}]{sohoni2020no}
Sohoni, N., Dunnmon, J., Angus, G., Gu, A., and R{\'e}, C.
\newblock No subclass left behind: Fine-grained robustness in coarse-grained
  classification problems.
\newblock \emph{Advances in Neural Information Processing Systems}, 33, 2020.

\bibitem[Suresh \& Guttag(2019)Suresh and Guttag]{suresh2019framework}
Suresh, H. and Guttag, J.~V.
\newblock A framework for understanding unintended consequences of machine
  learning.
\newblock \emph{arXiv preprint arXiv:1901.10002}, 2019.

\bibitem[Tschandl et~al.(2018)Tschandl, Rosendahl, and Kittler]{Tschandl2018}
Tschandl, P., Rosendahl, C., and Kittler, H.
\newblock The {HAM}10000 dataset, a large collection of multi-source
  dermatoscopic images of common pigmented skin lesions.
\newblock \emph{Scientific Data}, 5\penalty0 (1), 2018.
\newblock URL \url{https://doi.org/10.1038/sdata.2018.161}.

\bibitem[Vodrahalli et~al.(2018)Vodrahalli, Li, and Malik]{vodrahalli2018all}
Vodrahalli, K., Li, K., and Malik, J.
\newblock Are all training examples created equal? {A}n empirical study.
\newblock \emph{arXiv preprint arXiv:1811.12569}, 2018.

\bibitem[Wan \& McAuley(2018)Wan and McAuley]{DBLP:conf/recsys/WanM18}
Wan, M. and McAuley, J.~J.
\newblock Item recommendation on monotonic behavior chains.
\newblock In Pera, S., Ekstrand, M.~D., Amatriain, X., and O'Donovan, J.
  (eds.), \emph{Proceedings of the 12th {ACM} Conference on Recommender
  Systems, RecSys 2018, Vancouver, BC, Canada, October 2-7, 2018}, pp.\
  86--94. {ACM}, 2018.
\newblock \doi{10.1145/3240323.3240369}.
\newblock URL \url{https://doi.org/10.1145/3240323.3240369}.

\bibitem[Wilkinson et~al.(2016)Wilkinson, Dumontier, Aalbersberg, Appleton,
  Axton, Baak, Blomberg, Boiten, da~Silva~Santos, Bourne,
  et~al.]{wilkinson2016fair}
Wilkinson, M.~D., Dumontier, M., Aalbersberg, I.~J., Appleton, G., Axton, M.,
  Baak, A., Blomberg, N., Boiten, J.-W., da~Silva~Santos, L.~B., Bourne, P.~E.,
  et~al.
\newblock The fair guiding principles for scientific data management and
  stewardship.
\newblock \emph{Scientific data}, 3\penalty0 (1):\penalty0 1--9, 2016.

\bibitem[Wright(2020)]{wright2020general}
Wright, T.
\newblock A general exact optimal sample allocation algorithm: With bounded
  cost and bounded sample sizes.
\newblock \emph{Statistics \& Probability Letters}, pp.\  108829, 2020.

\bibitem[Yang et~al.(2020)Yang, Qinami, Fei-Fei, Deng, and
  Russakovsky]{yang2020towards}
Yang, K., Qinami, K., Fei-Fei, L., Deng, J., and Russakovsky, O.
\newblock Towards fairer datasets: Filtering and balancing the distribution of
  the people subtree in the imagenet hierarchy.
\newblock In \emph{Proceedings of the 2020 Conference on Fairness,
  Accountability, and Transparency}, pp.\  547--558, 2020.

\bibitem[Yona et~al.(2019)Yona, Ghorbani, and Zou]{yona2019s}
Yona, G., Ghorbani, A., and Zou, J.
\newblock Who's responsible? {J}ointly quantifying the contribution of the
  learning algorithm and training data.
\newblock \emph{arXiv preprint arXiv:1910.04214}, 2019.

\end{thebibliography}

\pagebreak
\appendix
\renewcommand\thefigure{\thesection.\arabic{figure}}  
\renewcommand\thetable{\thesection.\arabic{table}} 
\section{Proofs and Derivations}
\iftoggle{icml}
{\subsection{Proof of \cref{prop:weights_and_allocations}}
\label{sec:weights_and_allocations_proof}}
{\subsection{Full proof of \cref{prop:weights_and_allocations}}
\label{sec:weights_and_allocations_proof}}

Recall the random variable $\hat{L}$
defined with respect to function $f$, loss function $\ell$, and group weights $w: \mathcal{\groups} \rightarrow \mathbb{R}^+$:
\begin{align*}
    \hat{L}(w, \alpha, n;f, \ell) &:= \frac{1}{n}\sum_{i \in \mathcal{S}(\vec{\alpha}, n)} w(\group_i) \cdot \ell(f(x_i),y_i) 
\end{align*}
where the randomness comes from the draws of $x_i, y_i$ from  $\mathcal{D}_{\group_{i}}$ according to procedure $\mathcal{S}(\vec{\alpha},n)$ (\cref{def:sampling_def}), as well as any randomness in $f$. 

\weightsallocations* 
\begin{proof}
For any n, any $(\alpha, w)$ pair induces a vector $\vec{\gamma}'$ with entries
$\gamma_\group'(w,\vec{\alpha}) := w(\group) \alpha_\group /
\left(\sum_{\group \in \groups} w(\group) \alpha_\group \right)$, where
\begin{align*}
    \mathbb{E}[\hat{L}(w, \vec{\alpha}, n;f, \ell)]
   & = \frac{1}{n} \sum_{(x_i,y_i,g_i) \in \mathcal{S}(\vec{\alpha},n)} w(\group_i) \mathbb{E}[\ell( f(x_i), y_i)] \\
    &= 
    \sum_{\group \in \groups}
    \frac{n_\group}{n}  w(\group) \mathbb{E}_{(x,y) \sim \mathcal{D}_\group}[\ell( f(x), y)] \\
    &= c \cdot \mathbb{E}_{\group \sim \textrm{Multinomial}(\vec{\gamma}')}\left[\mathbb{E}_{(x,y) \sim \mathcal{D}_\group} \left[\ell( f(x), y)\right]\right]
\end{align*}
for constant $c= \sum_\group \alpha_\group w(\group)$.
The vector $\vec{\gamma}'$ in this sense describes an implicit `target distribution' induced by the applying weights $w$ after sampling with allocation $\vec{\alpha}$.
Note that unless $w_\group = 0$ for all $g$ with $\alpha_g > 0$, $\vec{\gamma}'$ has at least one nonzero entry.
The constant $c$ re-scales the weighted objective function with original weights $w$ so as to match the expected loss with respect to the group proportions $\vec{\gamma}'$. 
Stated another way, for any alternative allocation $\alpha'$, we could pick weights $w'(\group) = c \gamma'_\group / \alpha'_\group$ (letting $w'(\group)=0$ if $\alpha'_\group=0$), and satisfy
\begin{align*}
    \mathbb{E}[\hat{L}( w',\vec{\alpha}', n; f, \ell )] & = \mathbb{E}[\hat{L}(w, \vec{\alpha}, n; f, \ell)] ~.
\end{align*}

Given this correspondence, we now find the pair $(\vec{\alpha}^*, w^*)$ which minimizes $Var[\hat{L}( cw', \vec{\alpha}',n; f, \ell)]$,
subject to $w'(\group) \alpha'_\group = c \gamma'_\group  $. Since the original pair $(\vec{\alpha}, w)$ satisfies this constraint (by construction), we must have 
\begin{align*}
    \min_{\vec{\alpha'}, w' : w'(\group) \alpha'_\group = c \gamma'_\group } Var[\hat{L}( w', \vec{\alpha}',n; f, \ell)]
    \leq Var[\hat{L}( w, \vec{\alpha},n; f, \ell)] ~.
\end{align*}
We first compute $Var[\hat{L}( w', \vec{\alpha}',n; f, \ell)]$. By \cref{def:sampling_def}, samples $(x_i,y_i)$ are assumed to be independent draws from distributions $\mathcal{D}_{g_i}$, so that the variance of the estimator can be written as (for convenience we assume here that $n \alpha'_\group \in \mathbb{Z}$, see the discussion below):
\begin{align*}
Var\left[\hat{L}( w', \vec{\alpha}',n; f, \ell) \right] 
     &=  \frac{1}{n^2}\sum_{\group \in \groups} w'(\group)^2
     \sum_{i=1}^{n \alpha'_\group}\mathbb{E}_{(x_i,y_i) \sim \mathcal{D}_\group}\left[ \left(\ell(f(x_i),y_i)- \mathbb{E}_{(x,y)\sim\mathcal{D}_\group}[\ell(f(x),y)] \right)^2   \right] \\
     &=  \frac{1}{n}\sum_{\group \in \groups} \alpha'_\group w'(\group)^2   Var\left[\ell_\group^{(i)}   \right] ~,
\end{align*}
where $Var\left[\ell_\group^{(i)} \right]$ denotes shorthand for $\mathbb{E}_{(x_i,y_i) \sim \mathcal{D}_\group}\left[ \left(\ell(f(x_i),y_i)- \mathbb{E}_{(x,y)\sim\mathcal{D}_\group}[\ell(f(x),y)] \right)^2 \right]$~.
Now, to respect the constraint $w'(\group)\alpha'_\group = c \gamma_\group'$ means that for any $\group$ with $\gamma'_\group >0$,  $w'(\group)$ is a deterministic function of $\alpha'_\group$, since $c$ and $\vec{\gamma}'$ are determined by the initial pair $(\vec{\alpha},w)$. Then it is sufficient to compute
\begin{align*}
   \argmin_{\alpha' \in \Delta^{|\groups|}} \frac{1}{n}\sum_{\group \in \groups: \gamma'_\group > 0} \alpha'_\group 
    \left(\frac{c \gamma_\group'}{\alpha'_\group} \right)^2
    Var\left[\ell_\group^{(i)}   \right] &=
    \argmin_{\alpha' \in \Delta^{|\groups|}} \frac{c^2}{n}\sum_{\group \in \groups : \gamma'_\group > 0} \frac{ (\gamma_\group')^2}{\alpha'_\group} 
    Var\left[\ell_\group^{(i)}   \right] ~.
\end{align*}
The minimizer $\vec{\alpha}^*$ has entries $\alpha^*_\group = \gamma_\group' \sqrt{Var[\ell_\group^{(i)}]} / \left(\sum_{\group \in \groups} \gamma_\group ' \sqrt{Var[\ell_\group^{(i)}]}\right)
$. 
Because $\vec{\alpha}^*$ is unique and determines $w^*$, $Var[\hat{L}( w^*, \vec{\alpha}^*,n; f, \ell) ]  < Var[\hat{L}( w', \vec{\alpha}',n; f, \ell) ] $ with strict inequality unless $(w', \vec{\alpha}') = (w^*, \vec{\alpha}^*)$.
The optimal weights  are 
$w^*(\group) = c\gamma'_\group / \alpha^*_\group = c {\left(\sum_{\group \in \groups}\gamma'_\group\sqrt{Var[\ell_\group^{(i)}]}\right)}/{\sqrt{Var[\ell_\group^{(i)}]}}$. Note that for pair of groups $(\group, \group')$, the relative weights satisfy $w^*(\group)/w^*(\group') = \sqrt{{Var[\ell_{\group'}^{(i)}]}/{Var[\ell_\group^{(i)}]}}$, and thus do not depend on $\vec{\gamma}'$. 

If we consider finite sample concerns, the minimizer $\alpha^*$ must satisfy integer values $n\alpha^*_\group \in \mathbb{Z}^+ \, \forall \group \in \groups$. In this case, efficient algorithms exist for finding the integral solution to allocating  $n_g$ \citep{wright2020general}. However, the non-integer restricted solution $\alpha^*$ has a closed form solution, and we will use the fact that for any group g, $\alpha_g^*$ as defined above and its variant with the additional constraint that $n\alpha^*_\group \in \mathbb{Z}^+$ can differ by at most $\tfrac{|\groups|}{n}$. This means that any $\vec{\alpha}$ with $|\alpha_g - \alpha^*_\group| > \tfrac{|\groups|}{n}$ cannot be a minimizer of the objective function, even constrained to $n\alpha^*_\group \in \mathbb{Z}$. Since $w(g) \alpha_g = c \gamma_g'$, an equivalent statement in terms of $w$ is $| 1 - \frac{w(\group)}{w^*(\group)}| > \frac{w(\group)|\groups| }{n c \gamma'(\group)} = \frac{|\groups|}{n \alpha_{\group}} = \frac{|\groups|}{n_\group}$.

Finally, we show that if $w^*(\group) < w(\group)$, $\alpha_\group^* > \alpha_\group$.
This follows from our definition of $\vec{\gamma}'$ such that $w(\group) = c \gamma'_\group / \alpha_\group$, and our constraint, such that $w^*(\group) = c \gamma'_\group / \alpha^*_\group$. From these, we must have that $w^*(\group) \alpha_\group^* = w(\group) \alpha_\group$, from which the claim and its reverse follow.
\end{proof}

\subsection{Proof of \cref{lem:alpha_star}}
\label{sec:alpha_star_calculation}

\alphastar*

\begin{proof}
Recall the decomposition of the estimated 
population risk:
\begin{align*}
\label{eq:group_decomposition}
    \mathbb{E}_{(x,y) \sim \mathcal{D}} \left[ \ell(\hat{f}_\mathcal{S}(x),y)\right] 
    = 
    \sum_{\group  \in \groups} \gamma_\group \cdot \mathbb{E}_{(x,y) \sim \mathcal{D}_\group } \left[ \ell(\hat{f}(x) , y)\right] 
\approx \sum_{\group \in \groups} \gamma_\group  \left(\sigma_\group ^2 (n \alpha_\group )^{-p} +\tau_\group ^2 n^{-q} + \delta_\group \right) ~.
\end{align*}
Now we find
\begin{align*}
\vec{\alpha}^* =
\argmin_{\vec{\alpha} \in \Delta^{|\groups|}} \sum_{\group  \in \groups} \gamma_\group  \left(\sigma_\group ^2 (n \alpha_\group )^{-p} +\tau_\group ^2 n^{-q} + \delta_\group \right) 
=
\argmin_{\vec{\alpha} \in \Delta^{|\groups|}} (n^{-p})\sum_{\group  \in \groups} \gamma_\group  \left(\sigma_\group ^2 ( \alpha_\group )^{-p} \right)
&=
\argmin_{\vec{\alpha} \in \Delta^{|\groups|}}\sum_{\group  \in \groups} \gamma_\group  \sigma_\group ^2 \alpha_\group ^{-p} ~.
\end{align*}
If $\sigma_\group = 0 \ \forall \group \in \groups$, then any allocation $\vec{\alpha}^* \in \Delta^{|\groups|}$ minimizes the approximated population loss. Otherwise, $\alpha^*_\group = 0$ will be $0$ for any group with $\sigma_\group = 0$; what follows describes the solution  $\alpha^*_\group$ for $g$ with $\sigma_\group > 0$. 
If any $\alpha_\group = 0$, then the objective is unbounded above, so we can restrict our constraints to $\vec{\alpha} \in (0,1)^{|\groups|}$. As the sum of convex functions, the objective function is convex in $\vec{\alpha}$. It is continuously differentiable when $\alpha_\group > 0, \, \forall \group \in \groups$. 
The KKT conditions 
are satisfied when 
\begin{align*}
   p \gamma_\group  \sigma^2_\group  \alpha_\group ^{-(p+1)} &= \lambda \quad \forall \group  \\
   \sum_{\group } \alpha_\group   &= 1 ~.
\end{align*}
Solving this system of equations yields that the KKT conditions are satisfied when 
$
    \alpha^*_\group = {\left(\gamma_\group  \sigma^2_\group \right)^{1/(p+1)}}/{\sum_\group  \left(\gamma_\group  \sigma^2_\group \right)^{1/(p+1)}} ~.
$
Since this is the only solution to the KKT conditions, it is the unique minimizer. 
\end{proof}

\subsection{Proof of \cref{cor:same_sigma}}
\label{sec:cor1_proof}

\equalsigma*

\begin{proof}
Let $m = |\groups|$ denote the number of groups. When $\sigma_\group = \sigma \ \forall \group \in \groups$, 
\begin{align*}
\alpha^*_\group = \frac{\gamma_\group^{1/(p+1)}}{\sum_\group\gamma_\group^{1/(p+1)}} &= \gamma_\group \cdot \frac{1}{\gamma_\group + \gamma_\group^{p/(p+1)}\sum_{\group' \neq \group} \gamma_{\group'}^{1/(p+1)}} ~.
\end{align*}
Since $p > 0$ by \cref{ass:scaling}, we have that $\sum_{i=1}^n{\gamma_i}^{1/(p+1)}$ with $\gamma_i$ subject to (a) $\sum_{i=1}^n\gamma_i = s$ and (b) $\gamma_i > 0$ is maximized when all $\gamma_i$ are equal. Then, since $\sum_{\group' \neq \group} \gamma_\group' = 1 - \gamma_\group$, 
\begin{align*}
\gamma_\group \cdot \frac{1}{\gamma_\group + \gamma_\group^{p/(p+1)}\sum_{\group' \neq \group} \gamma_{\group'}^{1/(p+1)}} & \geq \gamma_\group \cdot \frac{1}{\gamma_\group + \gamma_\group^{p/(p+1)}\sum_{\group' \neq \group} \left(\frac{1-\gamma_\group}{m-1}\right)^{1/(p+1)}} \\
&= \gamma_\group \cdot \frac{1}{\gamma_\group +\left((m-1) \gamma_\group\right)^{p/(p+1)} (1-\gamma_\group)^{1/(p+1)}} ~.
\end{align*}
When $\gamma_\group \leq 1/m$, $\gamma_\group / (1-\gamma_\group) \leq \frac{1}{m-1}$, so that
\begin{align*}
\alpha^*_\group &\geq
\gamma_\group \cdot \frac{1}{\gamma_\group + (1-\gamma_\group)\left(m-1\right)^{p/(p+1)} \left(\gamma_\group/(1-\gamma_\group)\right)^{p/(p+1)}} \geq \gamma_\group \cdot \frac{1}{\gamma_\group + (1-\gamma_\group)} =  \gamma_\group ~. \qedhere
\end{align*} 
\end{proof}

\subsection{Proof of \cref{cor:two_groups_unequal_sigma}}
\label{sec:cor2_proof}

\twogroupsunequalsigma*

\begin{proof}
For the setting of two groups, we can express the optimal allocations as:
\begin{align*}
    \alpha^*_{A} &= \frac{\left(\gamma_A \sigma^2_A\right)^{1/(p+1)}}{\left(\gamma_A \sigma^2_A\right)^{1/(p+1)} + \left((1-\gamma_A) \sigma^2_B\right)^{1/(p+1)}}, \quad
    \alpha^*_B = 1 - \alpha^*_A
\end{align*}
Rearranging,
\begin{align*}
    \alpha^*_{A} 
     & = \gamma_A \frac{1}{\gamma_A + (\sigma_B^2 / \sigma_A^2)^{1/(p+1)}(1-\gamma_A)^{1/(p+1)} (\gamma_A)^{p/(p+1)}}~.
 \end{align*}
For $p > 0$, it holds that $0 < \frac{1}{p+1} < 1$. Therefore, for any $p >0$ and $\gamma < 0.5$,
\begin{align*}
  \gamma <  (\gamma)^{\frac{1}{p+1}} (1-\gamma)^{ \frac{p}{p+1}} <  (1-\gamma) ~.
\end{align*}
From this, we derive the upper bound
\begin{align*}
 \alpha^*_A < \gamma_A \frac{1}{\gamma_A + (\sigma_B^2 / \sigma_A^2)^{1/(p+1)}(\gamma_A) }
 = \frac{(\sigma_A^2)^{1/(p+1)}}{(\sigma_A^2)^{1/(p+1)} + (\sigma_B^2)^{1/(p+1)}},
\end{align*}
and the lower bound
\begin{align*}
    \alpha^*_A > \gamma_A \frac{1}{\gamma_A + (\sigma_B^2 / \sigma_A^2)^{1/(p+1)}(1-\gamma_A) }
 = \gamma_A \frac{( \sigma_A^2)^{1/(p+1)}}{\gamma_A ( \sigma_A^2)^{1/(p+1)} + (\sigma_B^2)^{1/(p+1)}(1-\gamma_A) } ~.
\end{align*}
When $\sigma_A \geq \sigma_B$,
\begin{align*}
    \alpha^*_A >\gamma_A \frac{( \sigma_A^2)^{1/(p+1)}}{\gamma_A ( \sigma_A^2)^{1/(p+1)} + (\sigma_B^2)^{1/(p+1)}(1-\gamma_A) }
    > \gamma_A \frac{( \sigma_A^2)^{1/(p+1)}}{(\gamma_A + 1-\gamma_A)( \sigma_A^2)^{1/(p+1)}} = \gamma_A ~, 
\end{align*}
and when $\sigma_A \leq \sigma_B$,
\begin{align*}
 \alpha^*_A < 
 \frac{(\sigma_A^2)^{1/(p+1)}}{(\sigma_A^2)^{1/(p+1)} + (\sigma_B^2)^{1/(p+1)}} 
 < \frac{(\sigma_A^2)^{1/(p+1)}}{(\sigma_A^2)^{1/(p+1)} + (\sigma_A^2)^{1/(p+1)}} = 1/2
 ~. 
\end{align*} 
\end{proof}

\subsection{Additional Derivations}
\label{appendix:ols_derivation}
In \cref{ex:linear_model_with_dummies}, we consider the model
$
    y_i = x_i^\top \beta + \alpha_1 \mathbb{I}[g_i = A] + \alpha_2 \mathbb{I}[g_i = B] + \mathcal{N}(0,\sigma^2)
$
where $x_i \sim \mathcal{N}(0,\Sigma_x)$ and denote $\theta = [\alpha_1, \alpha_2, \beta^\top]$ (note: here $\alpha_i$ denote the model coefficients, not allocations). We want to compute 
\begin{align*}
    \mathbb{E}_{(x,y,g) \sim \mathcal{D}} \left[ (\hat{f}(x) - y)^2 | g = A\right]  & = \sigma^2 + \mathbb{E} \left[\|x^\top (\hat{\beta} -\beta) + (\hat{\alpha}_1 - \alpha_1)\|^2\right]
    \\
    &= \sigma^2 + \mathbb{E} \left[x^\top (\hat{\beta} -\beta)(\hat{\beta} -\beta)^\top x \right] + 2 \mathbb{E} \left[(\hat{\alpha}_1 - \alpha_1) (\hat{\beta} - \beta)^\top x \right] + \mathbb{E} \left[(\hat{\alpha}_1 - \alpha_1)^2\right].
\end{align*}
Since the draw $(x,y,g) \sim \mathcal{D}$ is independent of the data from which the ordinary least squares solution $\hat{\theta}$ is predicted, we can write out each of these terms in terms of the dependence on $n$, the total number of data points, as well as $n_A$ and $n_B$ (where $n_A + n_B = n$), the total number of datapoints for each group, from which $\hat{\theta}$ is estimated. 
To do this, we'll solve for the entries of the covariance matrix:
\begin{align*}
    \mathbb{E} [(\hat{\theta}-\theta)(\hat{\theta}-\theta)^\top] 
    = 
   \begin{bmatrix}
    \mathbb{E} \left[ (\hat{\alpha_1}-\alpha_1)^2\right] & \mathbb{E} \left[ (\hat{\alpha_1}-\alpha_1)(\hat{\alpha_2}-\alpha_2)\right]& \mathbb{E} \left[ (\hat{\beta}-\beta)(\hat{\alpha_1}-\alpha_1)\right]^\top\\
    \mathbb{E} \left[ (\hat{\alpha_1}-\alpha_1)(\hat{\alpha_2}-\alpha_2)\right]& \mathbb{E} \left[ (\hat{\alpha_2}-\alpha_2)^2\right]&  \mathbb{E} \left[ (\hat{\beta}-\beta)(\hat{\alpha_2}-\alpha_2)\right]^\top\\
   \mathbb{E} \left[ (\hat{\beta}-\beta)(\hat{\alpha_1}-\alpha_1)\right]  &  \mathbb{E} \left[ (\hat{\beta}-\beta)(\hat{\alpha_2}-\alpha_2)\right] & \mathbb{E} \left[(\hat{\beta}-\beta)(\hat{\beta}-\beta)^\top\right]
  \end{bmatrix}
    = \sigma^2 (Z^\top Z)^{-1}
\end{align*}
where $Z$ is the $n \times (d+2)$ design matrix with rows $\{(\mathbb{I}[g_i = A], \mathbb{I}[g_i = B], x_i^\top)\}_{i=1}^n$~. 
Next, we find the block entries of the matrix $(Z^\top Z)^{-1}$. We first interrogate the term within the inverse:
\begin{align*}
    Z^\top Z = 
    \begin{bmatrix}
    n_A & 0 & n_A \bar{x}_A^\top\\
    0 & n_B & n_B \bar{x}_B^\top\\
    n_A \bar{x}_A & n_B \bar{x}_B & X^\top X
    \end{bmatrix}
\end{align*}
where $\bar{x}_A = \frac{1}{n_A} \sum_{i=1}^n x_i \cdot \mathbb{I}[g_i = A]$, and similarly for $\bar{x}_B$.
We'll now use the Schur complement to compute the desired blocks of $\Sigma$. The Schur complement is
\begin{align*}
S = X^\top X - 
    \begin{bmatrix}
    n_A \bar{x}_A & n_B \bar{x}_B
    \end{bmatrix}
    \begin{bmatrix}
    n_A^{-1} & 0 \\
    0 & n_B^{-1} 
    \end{bmatrix}
    \begin{bmatrix}
     n_A \bar{x}_A^\top\\
    n_B \bar{x}_B^\top\\
    \end{bmatrix} 
   =
    X^\top X - n \bar{x}\bar{x}^\top~,
\end{align*}
which we simplify to $S = X^\top X$ by assuming that we zero-mean the sample feature matrix $X$ before calculating the least squares solution. 
Using the Schur complement, the covariance matrix in block form is
\begin{align*}
    (Z^\top Z)^{-1} 
    = 
    \begin{bmatrix}
        \frac{1}{n_A} + \bar{x}_A^\top S^{-1}\bar{x}_A & \bar{x}_A^\top S^{-1}\bar{x}_B & -\bar{x}_A^\top S^{-1} \\ 
        \bar{x}_B^\top S^{-1}\bar{x}_A &\frac{1}{n_B} + \bar{x}_B^\top S^{-1}\bar{x}_B  & -\bar{x}_B^\top S^{-1} \\
        -S^{-1} \bar{x}_A &  -S^{-1} \bar{x}_B & S^{-1} 
        \end{bmatrix} ~.
\end{align*}
Plugging in the appropriate blocks to our original equation, we get:
\begin{align*}
    \mathbb{E}_{(x,y,g) \sim \mathcal{D}} \left[ (\hat{f}(x) - y)^2 | g = A\right] 
    &= \sigma^2 + \mathbb{E}_{(x,y) \sim \mathcal{D}_A} \left[x^\top (\hat{\beta} -\beta)(\hat{\beta} -\beta)^\top x  + 2 (\hat{\alpha}_1 - \alpha_1) (\hat{\beta} - \beta)^\top x +  (\hat{\alpha}_1 - \alpha_1)^2\right] \\
    & = \sigma^2( 1 +\frac{1}{n_A} + \mathbb{E}_{g=A} \left[x^\top S^{-1} x  +  \bar{x}_A^\top S^{-1} \bar{x}_A\right]) , 
\end{align*}
where the middle term cancels since $\mathbb{E}[x] = 0$.
Note that $S$ is the scaled sample covariance matrix.
The vectors $x_i$ are drawn i.i.d. from $\mathcal{N}(0,\Sigma_x)$ so that $S^{-1}$ follows an inverse Wishart distribution with parameters $n,d, \Sigma_x$. For the fresh sample $x$, 
\begin{align*}
    \mathbb{E}[x^\top S^{-1} x] &=  \textrm{Trace}(\mathbb{E}[S^{-1}]\mathbb{E}[xx^\top])= \frac{d}{n-d-1}Tr(\Sigma_{x}^{-1}\Sigma_{x})    = \frac{d}{n-d-1} ~.
\end{align*} 
For the $\bar{x}_A^\top S^{-1} \bar{x}_A$ term  we invoke the matrix inversion lemma. For  a single row $x_i$ of $X$, let $X_{-i}$ denote the $(n-1) \times d$ matrix comprised of all rows of $X$ except $X_i$. Then 
\begin{align*}
    x_i^\top (X^\top X)^{-1} x_i &= 
    x_i^\top (X_{-i}^\top X_{-i} + x_i x_i^\top)^{-1} x_i \\
    &= 
    x_i^\top \left((X_{-i}^\top X_{-i})^{-1} - (X_{-i}^\top X_{-i})^{-1} x_i (1 + x_i^\top(X_{-i}^\top X_{-i})^{-1}x_i)^{-1}x_i^\top (X_{-i}^\top X_{-i})^{-1}  \right) x_i ~.
\end{align*}
Letting $a_i = x_i^\top (X_{-i}^\top X_{-i})^{-1} x_i \geq 0$, we rewrite the above as
\begin{align*}
    x_i^\top (X^\top X)^{-1} x_i &= a_i - \frac{a_i^2}{1 +a_i}
    = \frac{a_i}{1+a_i} \leq a_i ~.
\end{align*}
Since the $x_i$ are independent and zero mean,
$
\mathbb{E}\left[ x_i^\top (X_{-i}^\top X_{-i})^{-1} x_j\right] =\mathbb{E}[x_i^\top]\mathbb{E}[ (X_{-i}^\top X_{-i})^{-1} x_j]  =  0 \, \forall i \neq j.
$
From a similar argument to that given above, we derive that $\mathbb{E}[a_i] = d / (n - d - 2)$, so that
\begin{align*}
    \mathbb{E}\left[\bar{x}_A^\top S^{-1} \bar{x}_A\right] = 
    \mathbb{E}\left[\left(\frac{1}{n_A}\sum_{i}^{n_A} x_i\right)^\top S^{-1}\left(\frac{1}{n_A}\sum_{i}^{n_A}x_i\right) \right] 
    = \frac{1}{n_A^2}\mathbb{E}\left[\sum_{i}^{n_A} x_i^\top S^{-1} x_i \right]  
    \leq \frac{1}{n_A} \cdot \frac{d}{n-d-2} ~.
\end{align*}
Putting this all together, we conclude that for $n\gg d$,
\begin{align*}
      \mathbb{E}_{(x,y,g) \sim \mathcal{D}} \left[ (\hat{f}(x) - y)^2 | g = A\right] 
     = 
     \sigma^2 \left(1 + \frac{1}{n_A} + \frac{d}{n-d-1} +
     \mathbb{E}[\bar{x}_A^\top S^{-1} \bar{x}_A]\right)
     = 
     \sigma^2\left( 1 + \frac{1}{n_A}
     +
     O\left(\frac{d}{n}\right)\right)
     ~.
\end{align*} 

\section{Experiment Details}

\label{app:experiment_appendix}
\Cref{sec:datasets_longer_desc} details the datasets used including preprocessing steps and any data excluded from the experiments; the remainder of \cref{app:experiment_appendix} provides a detailed explanation of each experiment described in the main text.  \cref{sec:appendix_experiments} describes additional experiments to complement the findings of the main experiments. \newtext{Code to process data and replicate the experiments detailed here is available at \url{https://github.com/estherrolf/representation-matters}.}

\iftoggle{icml}{

All image prediction models were run on a machine with 8 nVidia 2080 GPUs, 384GB RAM, and and Intel(R) Xeon(R) Gold 6126 CPU @ 2.60GHz, 12C/24HT, with the exception of the experiments detailed in \cref{sec:group_interactions_details}, which was run on a machine with 1 nVidia Titan-V GPU, 256GB RAM, Intel(R) Xeon(R) Gold 6148 CPU @ 2.40GHz, 20C/40HT. The non-image models were run on the latter machine.}
{The code repository accompanying this work can be found at \url{https://github.com/estherrolf/representation-matters}}.

\setcounter{figure}{0}   
\label{sec:experiment_appendix}
\subsection{Dataset Descriptions}
\label{sec:datasets_longer_desc}
We use and modify benchmark machine learning and fairness datasets, as well as more recent datasets on image diagnosis and student performance to study the effect of training set allocations on group and population accuracies in a systematic fashion. Each of the datasets we use is described below, with download links given at the end of this section.

\textbf{Modified CIFAR-4.} 
We modify an existing machine learning benchmark dataset, CIFAR-10 \cite{krizhevsky2009learning}, to instantiate binary classification tasks with binary groups, where groups $g$ are statistically independent of the labels $y$.
We take four of the ten CIFAR-10 classes: \{plane, car, bird, horse\}, and sort them into binary categories of \{land/air, animal/vehicle\} as in Fig.~\ref{fig:cifar_explanation1}.
This instantiation balances the classes labels among the two groups $g$, so that no matter the allocation of groups to the training set, the label distribution will remain balanced.
This will eliminate class imbalance as the cause of changes in accuracy due to training set composition or up-weighting methods.

There are $5,000$ training and $1,000$ test instances of each class in the CIFAR-10 dataset, resulting in $10,000$ training instances of each group in our modified CIFAR-4 dataset ($20,000$ instances total), and $2,000$ instances of each group in the test set ($4,000$ instances total). By construction, the average label in the test and train sets is $0.5$. Since this dataset is designed to assess the main questions of our study under a controlled setting and there is not a natural setting of population rates of animal and vehicle photos, we set the population prevalence parameter of group A (images of animals) as $\gamma_A = 0.1$. 

\begin{figure}[ht]
    \centering
         \includegraphics[width=.4\textwidth]{{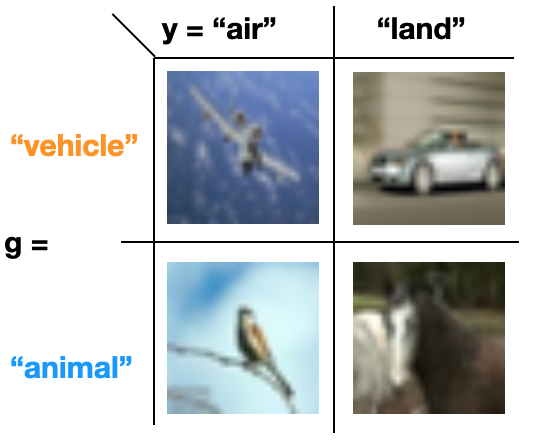}}
         \caption{Modified CIFAR-4 dataset setup.}
         \label{fig:cifar_explanation1}
\end{figure}

\textbf{Goodreads ratings prediction.}
The Goodreads dataset~\citep{DBLP:conf/recsys/WanM18} contains a large corpus of book reviews and ratings, collected from public book reviews online. From the text of the reviews, we predict the numerical rating (integers 1-5) a reviewer assigns to a book.
From the larger corpus, we consider reviews for two genres: history/biography (henceforth ``history") and fantasy/paranormal (henceforth ``fantasy"). 
We calculate the population prevalences from the total number of reviews in the corpus for each genre. As of the writing of this document, there are $2,066,193$ reviews for history and biography books, and $3,424,641$ for fantasy and paranormal books, so that $\vec{\gamma} = [0.376, 0.624]$.

After dropping entries with no valid review or rating, we have 1,985,464 reviews from the history genre, and 3,309,417 reviews from the fantasy genre, with no books assigned to both genres.
To reduce dataset size and increase performance of the basic prediction task, we further reduce each dataset to only the 100 most frequently reviewed books of each genre (following a procedure similar to~\citep{chen2018my}). 
To instantiate the dataset we use in our analysis, we draw 62,500 review/rating pairs uniformly at random from each of these pools. The mean review for fantasy instances is 4.146, for history it is 4.103. We allocate $20\%$ of the data to a test set, and the remaining $80\%$ to the training set, with an equal number of instances of each genre in each set. 

We use  tf-idf vectors of the 2,000 most frequently occurring reviews from the entire dataset of 125,000 instances as features (a similar featurization to \citep{chen2018my}, with fewer total features). We note that this is a different version of the goodreads dataset from that used in~\citep{chen2018my}; the updated dataset we use has more reviews, and we use different group variables.

\textbf{Adult.} The Adult dataset, originally extracted from the 1994 Census database, is downloaded from the UCI Machine Learning Repository \citep{Dua:2019}. The dataset contains demographic features and the classification task is to predict whether an individual makes more than 50K per year. We drop the final-weight feature, a measure of the proportion of the population represented by that profile. There are $d=13$ remaining features; we one-hot encode non-binary discontinuous features including work class, education, marital status, occupation, relationship, race, and native country, resulting in $103$ feature columns. 

For the main analysis, we exclude features for sex, husband, and wife as they indicate group status but do not affect predictive ability (see \cref{sec:groups_as_features}), for a total of $10$ features and $100$ columns in the feature matrix.  We keep the original train and test splits, with $32,561$ ($21,790$ from group male and $10,771$ from group female) and $16,281$ ($10,860$ from group male and $5,421$ from group female) instances respectively. 
We group instances based on gender (encoded as binary male/female).
While the training and tests sets have roughly a $2/3$ male group, $1/3$ female group balance, we set $\vec{\gamma} = [0.5,0.5]$ to more adequately reflect the true proportions of men and women in the world.
In the test set, the average label for the male group is $0.300$, for the female group it is $0.109$ (for the train set, the numbers are similar; $0.306$ for the male group and $0.109$ for the female group).

\textbf{Mooc dataset.} The dataset we refer to as the MOOC (Massive Open Online Course) dataset is the HarvardX-MITx Person-Course Dataset \citep{HMdata}. This dataset contains anonymized data from the 2013 academic year offerings of MITx and HarvardX edX classes. Each row in the dataset corresponds to a student-course pairing; students taking more than one course may correspond to multiple rows. We keep the following demographic, participation, and administrative features: gender, LoE\_DI (highest completed level of education), final\_cc\_cname\_DI (country), YoB,
ndays\_act (number of active days on the edX platform), nplay\_video (number of video plays on edX), nforum\_posts (number of discussion forum posts on edX), nevents (number of interactions on edX), course\_id, certified (whether the student achieved a high enough grade for certification). After dropping instances without valid entries for these features, we have 25,213 instances. We one-hot encode non-binary discontinuous features including course\_id, LoE\_DI, and final\_cc\_cname\_DI, resulting in an expanded 47 feature columns from the original 9 features. In \cref{sec:groups_as_features} we exclude demographic features. We partition the data into a train set of size 24,130 and a test set of size 6,032, with the same proportion of groups in each set.

We group instances by the highest level of education self-reported by the person taking the class; we bin this into those who have completed at most secondary education, and those who have completed more than secondary education. For reference, in the USA, completing secondary education  corresponds to completing high school.
The train and test sets contain an equal fraction of each group, about $16\%$ instances where the person taking the course had no recorded education higher than secondary. Of all training and test instances, those from group A (student had at most secondary education) have a $5.2\%$ certification rate, and those from group B (student had more than secondary education) have an $8.2\%$ certification rate.

\textbf{ISIC skin cancer detection dataset.} 
We download the dataset from the ISIC archive \citep{codella2018skin,Tschandl2018,codella2019skin} using the repository \url{https://github.com/GalAvineri/ISIC-Archive-Downloader} to assist in downloading.
To match the dataset used in Sohoni et al.~\yrcite{sohoni2020no}, we use only images added or modified prior to 2019, which gives us 23,906 instances. From this set we include only images that are explicitly labeled as benign or malignant.
We additionally remove any data points from the `SONIC' study, as these are all benign cases, and are easily identified via colorful dots on the images~\citep{sohoni2020no}.
The remaining 11,952 instances are to our knowledge identical to the ``Non-patch" dataset in \citep{sohoni2020no}, up to the random splits into training/validation and test sets.

As groups, we subset based on the approximate age of the patient that is the subject of the photo. Approximate age is binned to the nearest 5 years in the original dataset; we design groups as A $ = \{i : \texttt{age\_approx[i]} \geq 55\}$ and B $ = \{i : \texttt{age\_approx[i]} < 55\}$. 
Of the group A instances in the training and test sets, $31.4\%$ are malignant; of the group B instances, $8.37\%$ are malignant.
We set $\vec{\gamma}$ to match the distribution in the 11,952 data points in our preprocessed set, so that $\vec{\gamma} = [0.43,0.57]$. We split 80\% of the data into a train set and 20\% into a test set, with the same proportion of groups in each set.

The ISIC archive is a compilation of medical images aggregated from many individual studies. Experiments detailed in \cref{sec:group_interactions,sec:group_interactions_details} use the study from which the image originates as groups to investigate the interactions between groups when $|\groups| > 2$. These results also motivate our choice to exclude the SONIC study from the dataset we use for the main analysis.

\textbf{Download links.}
Code for processing these datasets as described above can be found 
\iftoggle{icml}{in the code repository accompanying this paper.}{at \url{https://github.com/estherrolf/representation-matters}.}
The datasets we use (before our subsetting/preprocessing) can be downloaded at the following links:
\begin{itemize}
     \item{CIFAR-10: \url{https://www.cs.toronto.edu/~kriz/cifar.html}}
    \item{Goodreads: \url{https://sites.google.com/eng.ucsd.edu/ucsdbookgraph/home}}
    \item{Adult: \url{https://archive.ics.uci.edu/ml/datasets/adult}}
    \item{Mooc: \url{https://dataverse.harvard.edu/dataset.xhtml?persistentId=doi:10.7910/DVN/26147}}
    \item{ISIC: \url{https://www.isic-archive.com}}
\end{itemize}

\textbf{Loss metrics.}
For binary prediction problems, we report the the 0/1 loss when there is not significant class imbalance (modified CIFAR-4, Adult). For the ISIC and Mooc tasks, we report 1 - AUROC, where AUROC is the area under the reciever operating curve. AUROC is a standard metric for medical image prediction~\cite{sohoni2020no}, and for Mooc we choose this metric due to the label imbalance (low certification rates). 
Since AUROC is constrained to be between 0 and 1, the loss metric 1 - AUROC will also be between 0 and 1.
For the Goodreads dataset, we optimize the $\ell_1$ loss (mean absolute error, MAE); in \cref{sec:additional_model_results} we compare this to minimizing the $\ell_2$ loss (mean squared error, MSE) of the predictions.

\subsection{Models Details}

For the neural net classifiers, we compare the empirical risk minimization (ERM) objective with an importance weighted objective, implemented via importance sampling (IS) following results in~\citep{buda2018systematic}, and a group distributionally robust (GDRO) objective~\cite{Sagawa2020Distributionally}. 
We adapt our group distributionally robust training procedure from \url{https://github.com/kohpangwei/group_DRO}, the code repository accompanying~\citep{Sagawa2020Distributionally}.  
For the non-image datasets, we choose the model class from a set of common prediction functions (\cref{sec:experiments_hp}). We use implementations of these algorithms from \url{https://scikit-learn.org}, using the built in sample weight parameters to implement importance weighting (IW). Since many of the prediction functions we consider are not gradient-based, we cannot apply the algorithm from~\citep{Sagawa2020Distributionally}, thus we do not compare to GDRO for the non-image datasets.  

\begin{table}[]
 \renewcommand{\arraystretch}{1.2}
    \centering
    \footnotesize
    \begin{tabular} {c|c|c|c|c}
    \toprule
        dataset name &  metric  & model used & objective & parameters \\  \hline
        \multirow{5}{*}{CIFAR-4 }
         & \multirow{5}{*}{\shortstack[1]{0/1 loss} }  & \multirow{5}{*}{\shortstack[1]{ pretrained \\ resnet-18 \\ (fine-tuned) }} & ERM & \texttt{num\_epochs} = 20, \texttt{lr} = 1e-3 , \texttt{wd} = 1e-4, \texttt{momentum} = 0.9 \\
         \cline{4-5} 
        && & \multirow{2}{*}{IS} &  \texttt{num\_epochs} = 20,  \texttt{lr} =  1e-3, \texttt{momentum} = 0.9  \\ & & & &  \texttt{wd} = 1e-2 if $\alpha_A \geq$ 0.98 or $\alpha_A \leq$ 0.02; \texttt{wd} = 1e-3 otherwise  \\
        \cline{4-5} 
        &&& \multirow{2}{*}{GDRO} & \texttt{num\_epochs} = 20,  \texttt{lr} =  1e-3, \texttt{wd} = 1e-3, \texttt{momentum} =  0.9, \\ & & & & \texttt{gdro\_ss} = 1e-2,  \texttt{group\_adj} = 4.0 \\
        \hline
        \multirow{4}{*}{ISIC} & \multirow{4}{*}{1 - AUROC}  & \multirow{4}{*}{\shortstack[1]{ pretrained \\ resnet-18 \\
        (fine-tuned) }} & ERM & \texttt{num\_epochs} = 20,  \texttt{lr} =  1e-3, \texttt{wd} = 1e-4, \texttt{momentum} =  0.9 \\
        \cline{4-5} 
        && & IS & \texttt{num\_epochs} = 20,  \texttt{lr} =  1e-3, \texttt{wd} = 1e-3, \texttt{momentum} =  0.9\\
        \cline{4-5} 
        &&& \multirow{2}{*}{GDRO} & \texttt{num\_epochs} = 20,  \texttt{lr} =  1e-3, \texttt{wd} = 1e-4, \texttt{momentum} =  0.9, \\ & & & &  \texttt{gdro\_ss} = 1e-1, \texttt{group\_adj} = 1.0 \\
        \hline
        \multirow{2}{*}{Goodreads*} & \multirow{2}{*}{\shortstack[1]{$\ell_1$ loss \\(MAE)}  }& \multirow{2}{*}{ \shortstack[1]{ multiclass \\ logistic regression}} & ERM & \texttt{C} = 1.0, \texttt{penalty} = $\ell_2$\\
        \cline{4-5} 
        &&& IW & \texttt{C} = 1.0, \texttt{penalty} = $\ell_2$\\
        \hline
          \multirow{2}{*}{Goodreads} & \multirow{2}{*}{\shortstack[1]{$\ell_2$ loss \\ (MSE)}  }& \multirow{2}{*}{ \shortstack[1]{ multiclass \\ logistic regression}} & ERM & $\lambda = 0.1$ \\
        \cline{4-5} 
        &&& IW & $\lambda = 1.0$ if $\alpha_A \geq 0.95$ or $\alpha_A \leq 0.05$; $\lambda = 0.1$ otherwise\\
        \hline
         \multirow{2}{*}{ \shortstack[1]{Mooc* (with \\ dem. features) }} & \multirow{2}{*}{1 - AUROC  }& \multirow{2}{*}{ \shortstack[1]{random forest \\ classifier} } & ERM & \texttt{num\_trees} = 400, \texttt{max\_depth} = 16  \\
         \cline{4-5}
         &&&  IW &  \texttt{num\_trees} = 400, \texttt{max\_depth} = 16 \\
        \hline
        \multirow{2}{*}{ \shortstack[1]{Mooc (no \\ dem. features) }} & \multirow{2}{*}{1 - AUROC  }& \multirow{2}{*}{ \shortstack[1]{random forest \\ classifier} } & ERM & \texttt{num\_trees} = 400, \texttt{max\_depth} = 8  \\
         \cline{4-5}
         &&&  IW &  \texttt{num\_trees} = 400, \texttt{max\_depth} = 8 \\
        \hline
         \multirow{2}{*}{\shortstack[1]{Adult (with  \\ group features)} } & \multirow{2}{*}{\shortstack[1]{0/1 \\loss } } & \multirow{2}{*}{\shortstack[1]{random forest \\ classifier} } & ERM & \texttt{num\_trees} = 200, \texttt{max\_depth} = 16 \\
         \cline{4-5}
         &&&  IW &  \texttt{num\_trees} = 200, \texttt{max\_depth} = 16\\
         \hline
         \multirow{2}{*}{\shortstack[1]{Adult* (without  \\ group features)} } & \multirow{2}{*}{\shortstack[1]{0/1 \\ loss} } & \multirow{2}{*}{\shortstack[1]{random forest \\ classifier} } & ERM & \texttt{num\_trees} = 400, \texttt{max\_depth} = 16 \\
         \cline{4-5}
         &&&  IW &  \texttt{num\_trees} = 400, \texttt{max\_depth} = 16\\
         \bottomrule
    \end{tabular}
    \caption{Models and hyperparameters used to generate \cref{fig:uplot_multifig}. Asterisks denote the Goodreads, Mooc and Adult setting that appear in \cref{fig:uplot_multifig}, the other settings are shown in \cref{fig:with_and_without_demographics_adult_mooc}.}
    \label{tab:hps_chosen}
\end{table}

\subsection{Hyperparameter Selection}
\label{sec:experiments_hp}

We evaluate hyperparameters for each prediction model using 5-fold cross validation on the training sets (see \cref{sec:datasets_longer_desc}), stratifying folds to maintain group proportions. We evaluate the cross-validation across a sparse grid of $\vec{\alpha} \in [0.02, 0.05, 0.2, 0.5, 0.8, 0.95, 0.98]$, allowing us to determine if hyperparameters should be set differently for different values of $\vec{\alpha}$. 
\cref{tab:hps_chosen} describes the model and parameters which are chosen as a result of this process.

\textbf{Image Datasets.}
For the results shown in \cref{fig:uplot_multifig} with image datasets (CIFAR-4 and ISIC), we train a pretrained resnet-18 by running SGD with momentum for 20 epochs for each dataset. We did not find significant improvements from training for more epochs for either dataset. Sohoni et al.  \yrcite{sohoni2020no} use a pretrained resnset-50 for the ISIC prediction task; we use a resnet-18 since we are mostly working with smaller training set sizes due to subsetting. We did not find major differences in performance for ERM for ISIC between the resnset-18 and resnet-50 for the dataset sizes we considered. We use a resnset-18 for all models for the CIFAR-4 and ISIC tasks.

In the 5-fold cross validation procedure, we search over learning rate in $[0.01, 0.001, 0.0001]$ and weight decay in $[0.01, 0.001, 0.0001]$, keeping the momentum at $ 0.9$ for both the importance sampling and ERM procedures.
Given these results, for GDRO, we search over group-adjustment in $[1.0, 4.0, 8.0]$, gdro step size in $[0.1, 0.01, 0.001]$, and  weight decay in $[0.001, 0.0001]$, fixing momentum at $0.9$ and learning rate at $[0.001]$. 

The optimal hyperparameter configurations for the coarse grid of $\vec{\alpha}$ are in \cref{tab:hps_chosen}.  
Across $\vec{\alpha}$ values from the coarse grid,
either the optimal parameters for each objective were largely consistent, or performance did not vary greatly between hyperparameter configurations for almost all dataset/objective configurations.
As a result, for both the modified CIFAR and ISIC datasets, we keep the hyperparameters fixed across $\vec{\alpha}$, with the exception of IS for CIFAR-4, where we increase weight decay for extreme $\vec{\alpha}$ (see \cref{tab:hps_chosen}).

\textbf{Non-Image Datasets.}
For the Goodreads dataset, we evaluate the following models and parameters: ridge regression model with $\lambda \in [0.01,0.1,1.0,10.0,100.0]$, random forest regressor with splits determined by MSE, and with number of trees and maximum depth of trees $\in  [100,200,400] \times [32,64,128]  $, and a multiclass logistic regression classifier with $\ell_2$ inverse regularization strength parameter $C \in [0.01,0.1,1.0, 10.0]$. 


The multiclass logistic regression model minimized mean absolute error (MAE) over the models we considered. 
For ERM, the optimal regularization parameter was $C = 1.0$ for all  $\vec{\alpha}$ in our sparse grid.
For IW, the optimal regularization value  was $C = 1.0$ for all $\alpha_A$ other than $0.98$, where the optimal for history MAE was $C=10.0$. Since this only effected one group, and was not symmetric across $\vec{\alpha}$, for the evaluation results, we set $C= 1.0$ for all $\vec{\alpha}$.

For the Mooc dataset, we evaluate a binary logistic regression model with $\ell_2$ penalty and inverse regularization parameter $C \in [0.001,0.01,0.1,1.0,10.0]$, a random forest classifier with number of tress and maximum depth of trees $\in [100,200,400] \times [8,16,32]$, and ridge regression model with
\iftoggle{icml}{$\lambda \in [0.0001, 0.001, 0.01, 0.1, 1.0, 10.0]$}{$\lambda \in [0.0001, 0.001, 0.01$, $0.1, 1.0, 10.0]$} and threshold for binary classification decision $0.5$.
The best model across both group and population accuracies was a random forest model; the best maximum depth was $16$ for both ERM and IW, and the results were robust to the number of estimators, so we chose $200$ for both. The best hyperparameters were largely consistent for all $\vec{\alpha}$ in the sparse grid.

For the Adult dataset, we evaluate the same models and parameter configurations as the MOOC dataset. The best model across both ERM and IW was the random forest model, with optimal values given in \cref{tab:hps_chosen}.

\subsection{Navigating Tradeoffs}
Using the selected hyperparameters from the procedure described in \cref{sec:experiments_hp}, we evaluate performance on the heldout test sets (see \cref{sec:datasets_longer_desc}). For the final evaluation, we evaluate
\iftoggle{icml}{$\alpha \in [0.0, 0.01, 0.02, 0.05, 0.1, 0.15, 0.2, 0.25, .3, 0.35, .4, 0.45, 0.5, 0.55, 0.6, 0.65, 0.7, 0.75, 0.8, 0.85, 0.9, 0.95, 0.98, 0.99, 1.0]$}{$\alpha \in [0.0, 0.01, 0.02, 0.05, 0.1$, $0.15, 0.2, 0.25, .3, 0.35, .4, 0.45, 0.5, 0.55, 0.6, 0.65, 0.7, 0.75, 0.8, 0.85, 0.9, 0.95, 0.98, 0.99, 1.0]$}, skipping the extremes for GDRO and IS/IW.

The training set size is determined by the smaller of the training set sizes for each group, denoted as $\min_\group n_\group$ in \cref{tab:datasets}. This results in  $n=10,000$ for the modified CIFAR-4, $n=4,092$ for ISIC, $n=50,000$ for goodreads, $n=3,897$ for Mooc,  and $10,771$ for adult. Note that the number of test set instances \iftoggle{icml}{}{($n_\textrm{test}$ in \cref{tab:datasets})} does not necessarily have the proportions of instances indicated by $\gamma_A$ in  \cref{tab:datasets}. For the CIFAR-4 and Goodreads dataset, the test set is constructed to have $50\%$ instances from group A, and $50\%$ instances from group B; in reporting performance, we take a weighted average over the errors from each group, as in Eq.~\eqref{eq:err_per_group}. For adult, we re-weight the test set instances to reflect $\vec{\gamma} = (0.5,0.5)$. For the remaining two datasets, the test set compositions align with $\vec{\gamma}$.  
We assess variability in the performance of each method under each setting by examining results over 10 random seeds, which apply to both the random sampling in the subsetting procedure and the randomness in the training procedure.

\subsection{Assessing scaling law fit}
\label{sec:scaling_law_fit_details}

\begin{figure}[!ht]
     \begin{subfigure}[b]{0.32\textwidth}
         \centering
 \includegraphics[height=4cm]{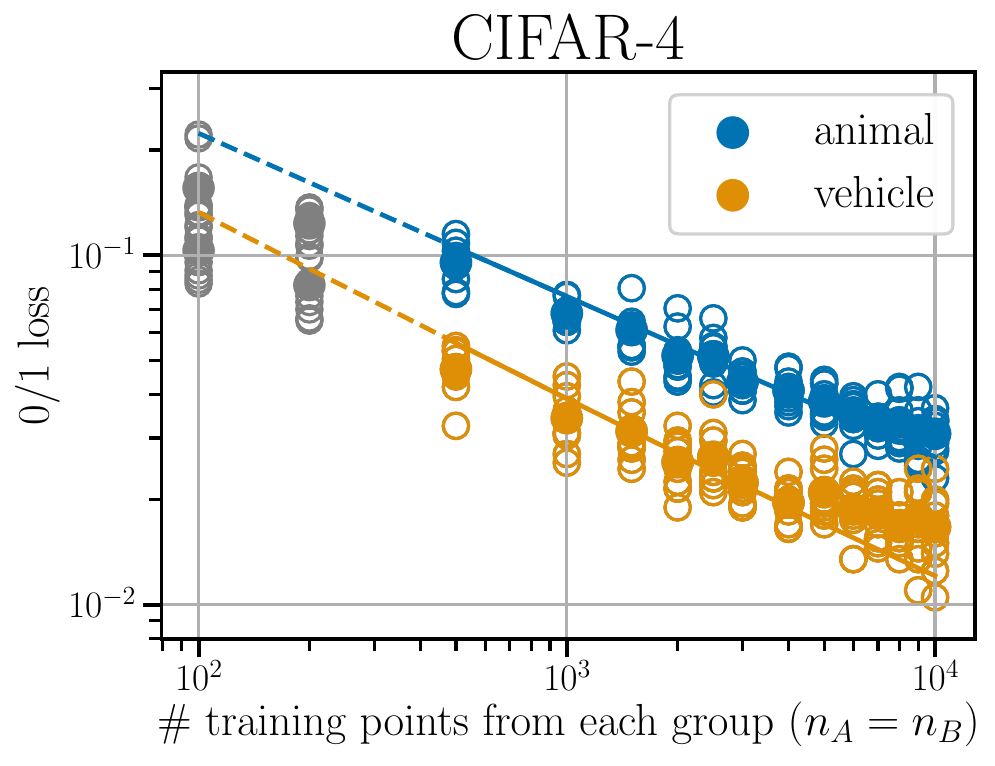}
     \end{subfigure}
      \begin{subfigure}[b]{0.32\textwidth}
         \centering
 \includegraphics[height=4cm]{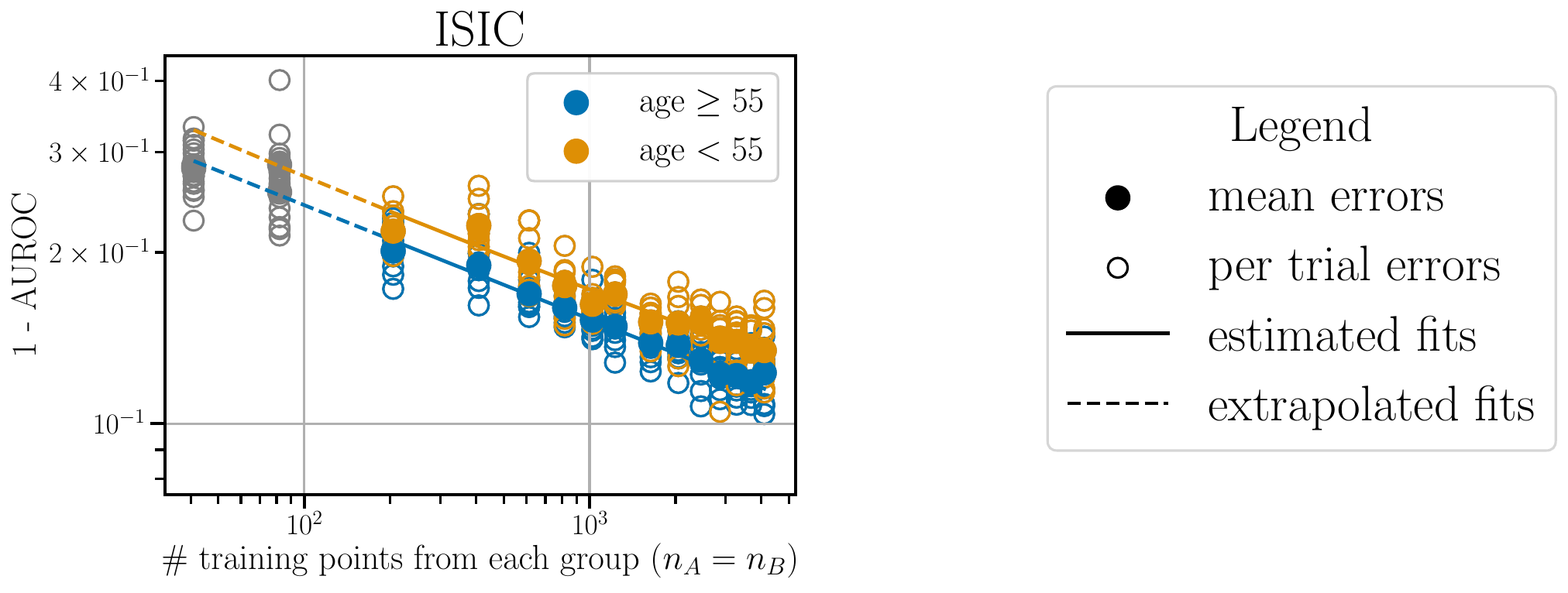}
     \end{subfigure}  \\
 \begin{subfigure}[b]{0.32\textwidth}
         \centering
 \includegraphics[height=4cm]{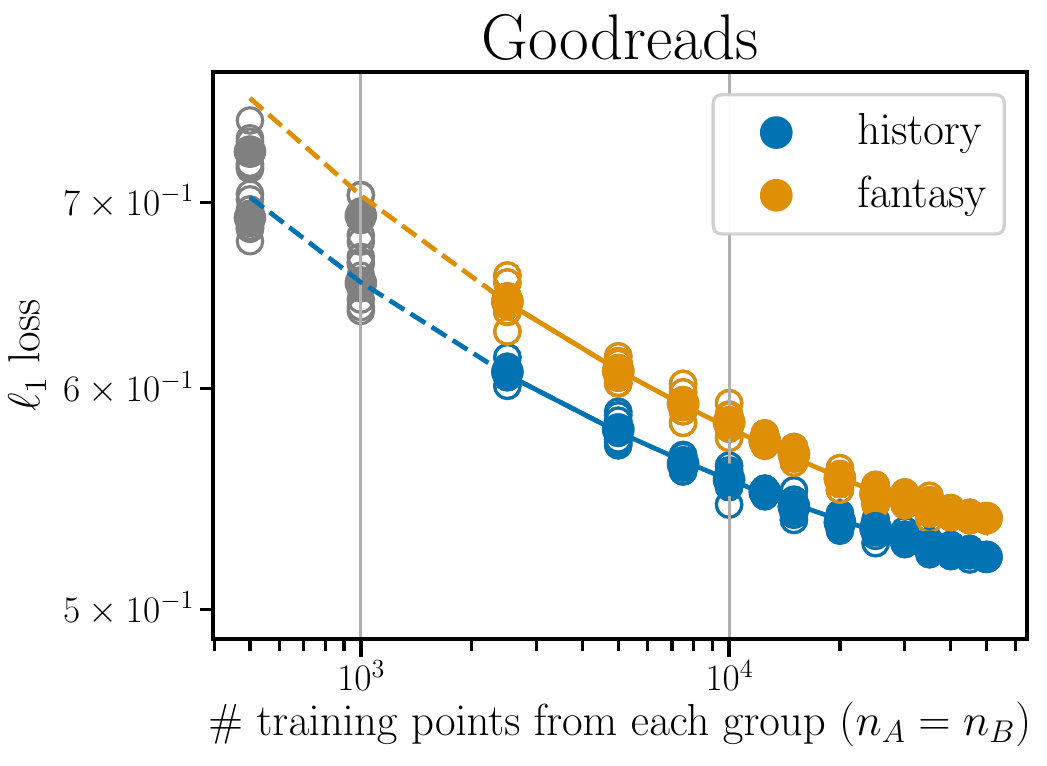}
     \end{subfigure}
      \hfill
      \begin{subfigure}[b]{0.32\textwidth}
         \centering
 \includegraphics[height=4cm]{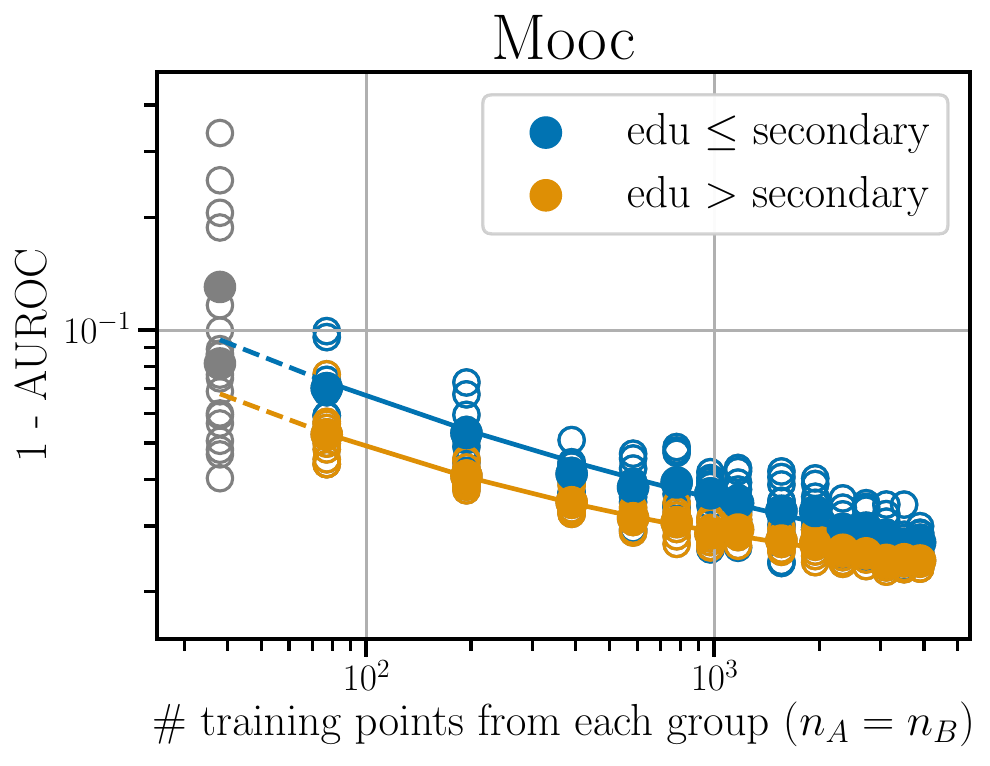}
     \end{subfigure}
\hfill
      \begin{subfigure}[b]{0.32\textwidth}
         \centering
 \includegraphics[height=4cm]{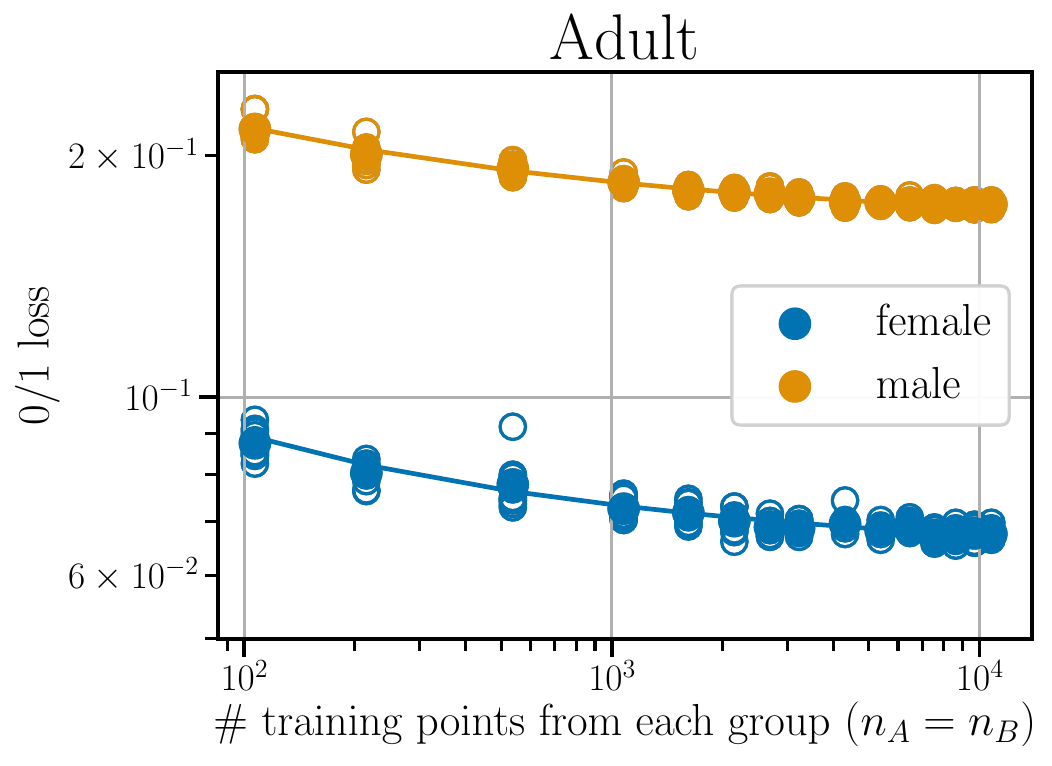}
     \end{subfigure}
        \caption{Estimated scaling law fits describe observed trends of group errors as a function of $(n_\group, n)$.
        Grey points are not included in the scaling law fit, as $n_g < M_g$ (see \cref{tab:scaling_fits}).
        }
        \label{fig:empirical_scaling}
\end{figure}

In addition to the results on the holdout set in the previous section, where we vary $\vec{\alpha}$ and keep the size of the training set fixed at $n = \min_\group n_\group$, we conduct subsetting experiments to assess the affect of $n$, as well as $n_\group$, on the group accuracies as in Eq.~\eqref{eq:general_scaling_law}. 
Specifically, for groups A and B, we vary the relative number of data points from group A and group B, as well as the total number of datapoints $n$. We define our subsetting grid by subsampling twice for each configuration of $(n_\group,n)$. First, we choose an allocation ratio of group A to group B (from options $[0.125\!:\!1, 0.25\!:\!1, 0.5\!:\!1, 1\!:\!1]$). Then, we subsample again to $x$ fraction of the largest subset size for this allocation ratio, for an  $x \in [0.01, 0.02, 0.05, 0.1, 0.15, 0.2, 0.25, .3, .4,  0.5,  0.6, 0.7, 0.8, 0.9, 1.0]$ (we skip $x \in [0.01]$ for allocation ratios $<1$). We repeat the sampling pattern for all pairs of allocation ratio and $x$, and again switching the roles of groups A and B in this procedure. This results in a set of 99 unique $n_A, n_B$ pairs, which we combine with the 25 pairs from the previous experiment, where we fix $n = n_A + n_B$ (see \cref{fig:scaling_sample_configurations}). We run ten random seeds of $(n_A, n_B)$ configuration, for a total of 1240 points from which we estimate the scaling law fits.

\begin{figure}[!htb]
     \centering
     \includegraphics[width=.4\columnwidth]{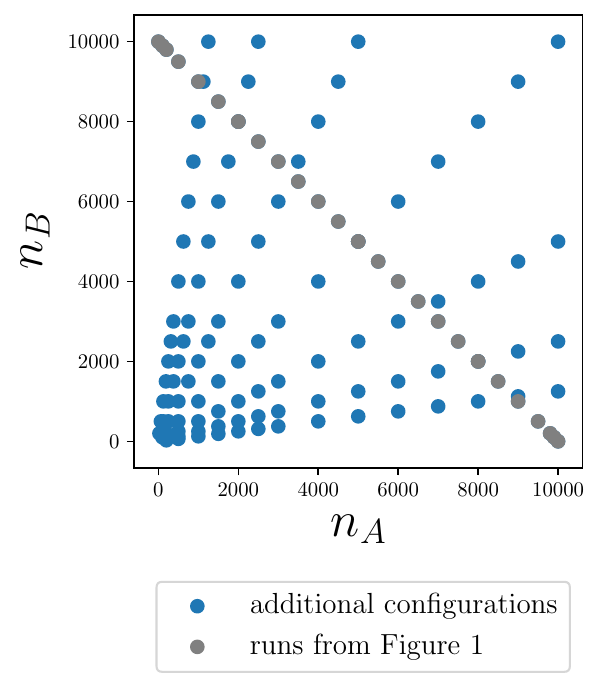}
         \footnotesize 
         \caption{ 
         Sample configurations for $(n_\group, n)$ described in  \cref{sec:scaling_law_fit_details}, shown for the CIFAR-4 dataset, where $\min_\group n_\group = 10,000$. Note that there are two groups, so $n = n_A + n_B$.
        \label{fig:scaling_sample_configurations}}
\end{figure}
Since our main point of comparison is across training set sizes and allocations, rather than optimizing hyperparameters for each new sample size, we use the same hyperparameter configurations as for first set of experiments (ERM rows in \cref{tab:hps_chosen}).
For the neural net classifiers, we decrease the number of epochs when the total training set size $n$ increases, so as to keep nearly the same number of gradient steps per training procedure. Specifically, we set the number of epochs as the nearest integer to {{$\textrm{(\# epochs for first set of experiments)} \times 
\textrm{($n$ for first set of experiments)} /
\textrm{($n$ for current run)}$ }}, where the first set of experiments corresponds to those in \cref{fig:uplot_multifig}. 

Together with the results from the previous experiment shown in~ \cref{fig:uplot_multifig}, we use these values of $(n,n_A,n_B)$ and the accuracies evaluated on groups $A$ and $B$ to estimate the parameters of the scaling laws in Eq.~\eqref{eq:general_scaling_law}. We use the nonlinear least squares implementation of
\href{https://docs.scipy.org/doc/scipy/reference/generated/scipy.optimize.curve\_fit.html}{\texttt{scipy.optimize.curve\_fit}}. The estimated parameters are given in \cref{tab:scaling_fits}. The standard deviations reported in \cref{tab:scaling_fits} are the estimated standard deviations resulting from the nonlinear least squares fits.    \cref{fig:empirical_scaling} plots the fitted model over a subset of the data used to fit it, showing that the modified scaling model in Eq.~\eqref{eq:general_scaling_law} can express the trends in per-group losses as a function of $n_\group$ and $n$.

The parameter estimates sometimes have large standard deviations; we also found the parameter estimates to vary with $M_g$, the minimum number of points required to include a data point in the fitting procedure, as well as the overall subsetting design. When seeking exact and robust estimates of the parameters in Eq.~\eqref{eq:general_scaling_law}, we suggest following the experimental procedures outlined in \cite{clauset2009power}, and additionally accounting for potential variation due to sampling patterns.

\subsection{Pilot Sample Experiment}
\label{sec:pilot_sample_details}
In this experiment, we take a small random sample from the Goodreads training set to instantiate a ``pilot sample" from which to estimate scaling law parameters and suggest an $\hat{\alpha}^*$ at which to sample a larger training sample in a subsequent sampling effort. 
The pilot sample contains $5000$ total instances, with $n_A = n_B = 2500$. 
With the pilot sample, we conduct a series of subsetting pairs for $(n_A, n_B)$, similar to the procedure outlined in \cref{sec:scaling_law_fit_details}, using relative allocations in $[0.0625\!:\!1, 0.125\!:\!1, 0.25\!:\!1, 0.5\!:\!1, 1\!:\!1]$ and $x \in [0.02, 0.05, 0.1, 0.15, 0.2, 0.25, .3, .4,  0.5,  0.6, 0.7, 0.8, 0.9, 1.0]$.
For each subset configuration, we sample 5 random seeds. From these  results, we estimate the scaling parameters of Eq.~\eqref{eq:general_scaling_law} according to the performance of each subset configuration on the heldout evaluation set. 
We set the minimum number of points from which to fit the scaling parameters to $M_g = 250$.

Next, we use these estimated fits to extrapolate predicted per-group losses with more data points. %
For a given sample size, we suggest the $\hat{\alpha}^*$ that would minimize the the maximum of the estimated per-group losses. 
We use the training data held separate from the pilot sample to sample a training set at $\vec{\alpha} = \hat{\alpha}^*$  allocation, and evaluate performance on the original held out evaluation set.
We follow the same procedure, sampling the new training set from $\vec{\alpha} = \vec{\gamma}$ and at $\vec{\alpha} = (0.5,0.5)$. 

\newtext{
As a point of comparison we also compute the results for all $\alpha$ in a grid of resolution $0.01$; i.e., $\alpha \in [0.0,0.01,0.02, \ldots 0.99,1.0]$ and denote the allocation value in this grid that minimizes the average maximum group loss over trials as $\alpha^*_{\textrm{grid}}$. The best allocation possible (up to the 0.01 grid resolution) were: $\alpha^*_{\textrm{grid}}= $ 0.01 for $n_{\textrm{new}}=5000$, $\alpha^*_{\textrm{grid}}= $ 0.06 for $n_{\textrm{new}}=10000$, $\alpha^*_{\textrm{grid}}= $ 0.03 for $n_{\textrm{new}}=20000$, and $\alpha^*_{\textrm{grid}}= $ 0.07 for $n_{\textrm{new}}=40000$. 
The $\alpha^*_\textrm{grid}$ baseline helps contextualize performance of other allocation strategies with the optimal-in-hindsight $\alpha^*_{\textrm{minmax}}$. Furthermore, the variability across trials due to subsetting around $\alpha = \alpha^*_\textrm{grid}$ is largely consistent with that of the other allocation sampling strategies examined.}

We repeat this entire procedure (starting with generation of the pilot sample) for 10 random seeds, and the results are reported in \cref{fig:pilot_results_goodread}.
 Since we have enough training data in the Goodreads dataset outside of the pilot sample to simulate gathering larger datasets of up to $40,000$, we simulate collecting a new training dataset of size $n_{new} \in [5000, 10000, 20000, 40000]$, which are $[1\times, 2\times, 4\times, 8\times]$ the size of the pilot training dataset, respectively.

The error bars in \cref{fig:pilot_results_goodread_all} are similarly large for all three allocation strategies we compare. The allocation $\vec{\alpha} = \hat{\alpha}^*_{\textrm{minmax}}$ when $n_{\textrm{new}} = 40000$ is an exception, indicating that high variation in $\hat{\alpha}^*_{\textrm{minmax}}$ may be an issue when $n_{\textrm{new}}$ is large relative to the pilot training sample size.

\subsection{Interactions Between Groups}
\label{sec:group_interactions_details}

While the main experiments exclude the SONIC study from the ISIC dataset, consistent with the `non-patch' dataset in \cite{sohoni2020no}, this experiment utilizes the larger corpus of labeled images from the ISIC repository. The difference is the inclusion of the SONIC sub-study. This larger dataset has 21,203 instances after subsetting to precisely benign/malignant cases. The number of data instances coming from each dataset are given by the grey bars (black annotated numbers) in \cref{fig:isic_substudies}; purple bars denote the number of malignant instances. Note the logarithmic scale.

We separate 16,965 images to the training set and the remaining 4,238 to the test set, with equal ratios of each sub-study represented in the train and test sets. 
\iftoggle{icml}{
\newtext{The train and test splits differ for the study and sub-study analysis, accounting for differing values in corresponding cells of \cref{fig:isic_loo_studies_appendix} and \cref{fig:isic_loo_groups_appendix}, particularly apparent when evaluating on JID Editorial Images, the study with the smallest number of data instances (see \cref{fig:isic_substudies}).}}
{
The train and test splits differ for the study and sub-study analysis, accounting for differing values in corresponding cells of \cref{fig:isic_loo_studies} and \cref{fig:isic_loo_substudies}, particularly apparent when evaluating on JID Editorial Images, the study with the smallest number of data instances (see \cref{fig:isic_substudies}).}
}
We choose hyperparameters based off a 5 fold cross-validation procedure on the 16,965 training instances, searching over the same hyperparameter options as in \cref{sec:experiments_hp}, ultimately using 20 epochs, momentum 0.9, weight decay 0.001 and learning rate 0.001 to tune the resnet-18 model (with ERM objective).
For each of the 5 studies, and 10 sub-studies that comprise the larger datasets, we retrain the resnet-18 model on the training set with that sub-study excluded from the training data, and compare to performance of the model trained on the entire training set.  As in \cref{sec:scaling_law_fit_details}, we keep the number of gradient steps for each training process roughly equal. We run 10 random seeds for all conditions and report the differences in the mean performance metrics across groups.
\iftoggle{icml}
{ \cref{fig:isic_loo_studies_appendix} (same as \cref{fig:isic_loo_studies}) shows the results of subsetting by each of the five studies; \cref{fig:isic_loo_substudies} shows the results of subsetting by each of the ten sub-studies. \newtext{We note that when the evaluation set is very small (as for the 2018 JID Ed. study), there can be high variation in estimated accuracy values due to randomization in the train/evaluation splits, highlighting the need for adequate representation across groups in evaluation data as well as in training data.}}
{}

\begin{figure}[!ht]
    \centering
         \includegraphics[width=.5\textwidth]{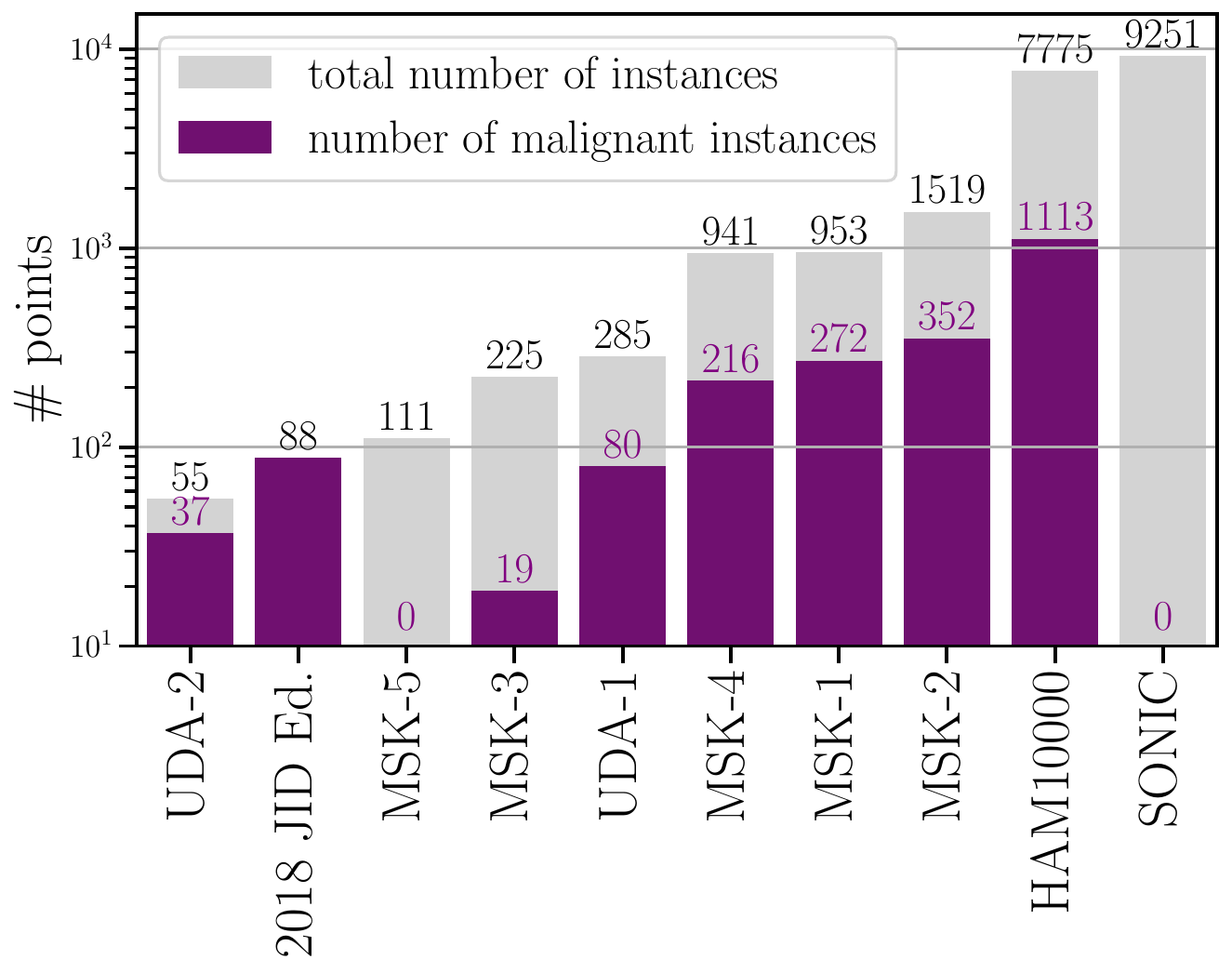}
         \caption{Sub-studies that comprise the ISIC dataset. The SONIC study  is excluded in the main analysis.}
         \label{fig:isic_substudies}
\end{figure}

\iftoggle{icml}
{{\begin{figure}[!thb]
     \centering
     \begin{subfigure}[b]{0.4\textwidth}
        \centering
         \includegraphics[width=\columnwidth]{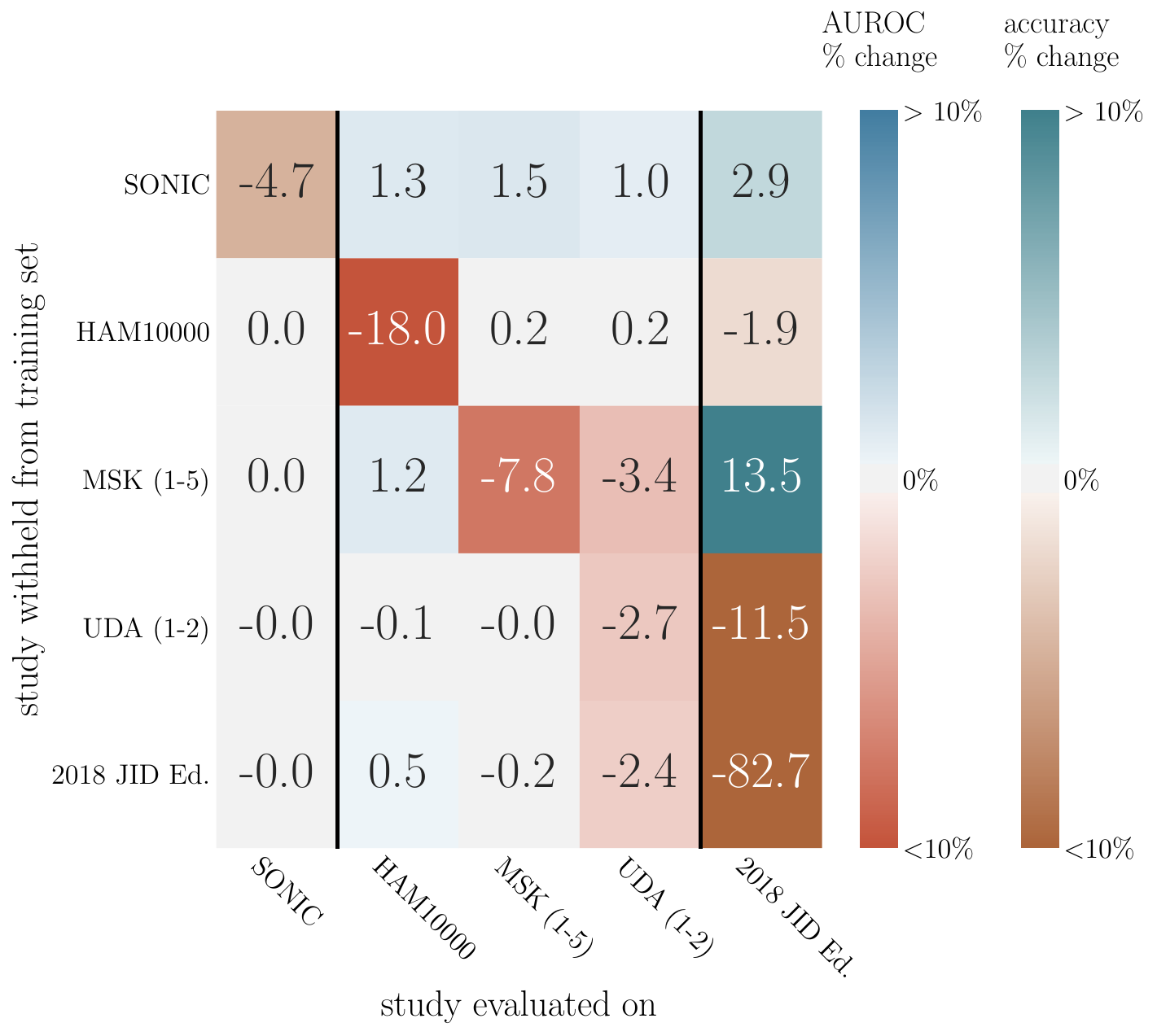}
        \caption{Grouped by study. \label{fig:isic_loo_studies_appendix}}
         \end{subfigure}
         ~
     \begin{subfigure}[b]{0.4\textwidth}
        \centering
         \includegraphics[width=\textwidth]{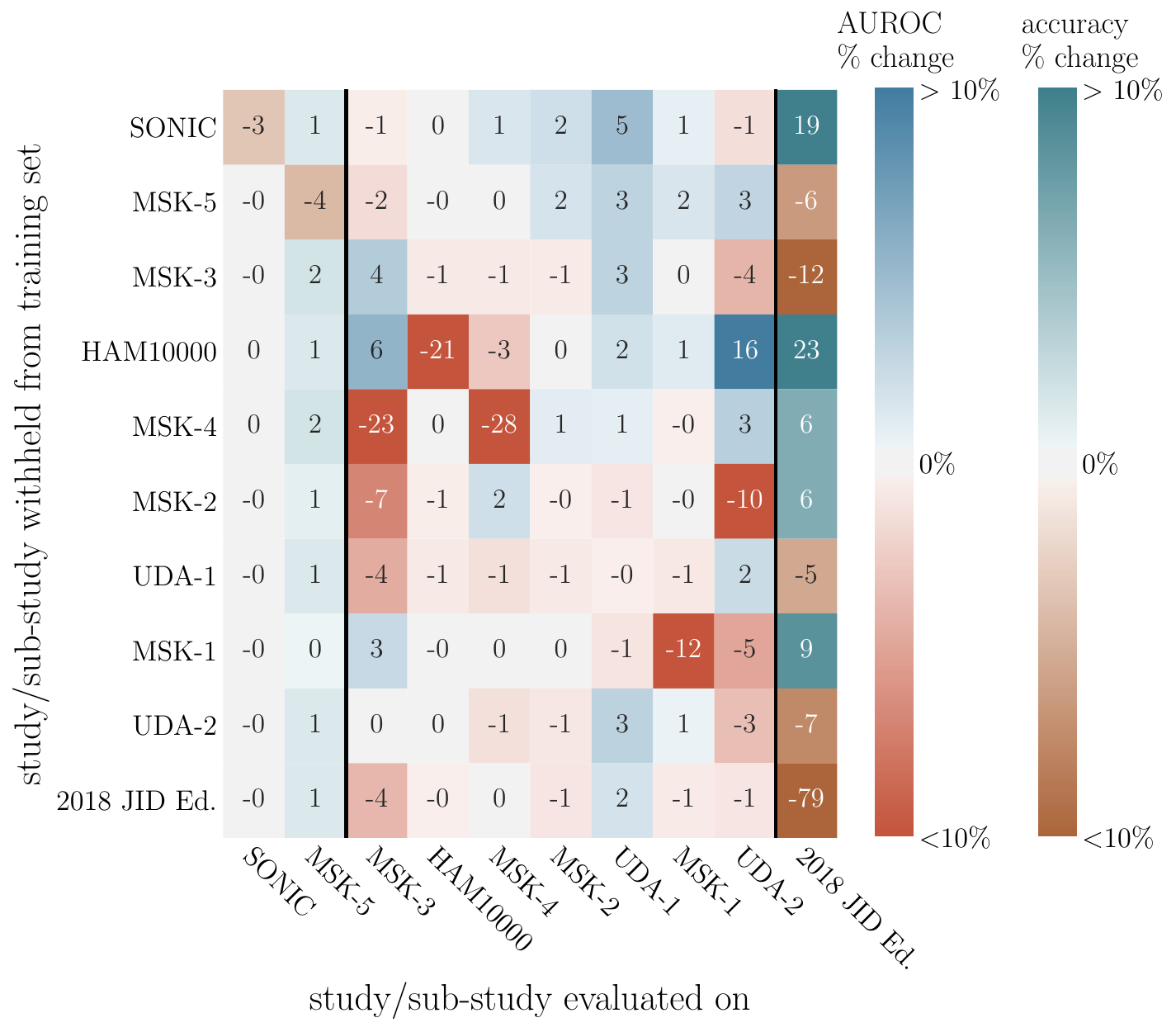}
         \caption{Grouped by sub-study.}
         \label{fig:isic_loo_substudies}
         \end{subfigure}
         \caption{
         \label{fig:isic_loo_groups_appendix}
         Performance changes due to holding out a study from the training set. Studies are ordered by $\%$ malignancy of the data in the evaluation set. SONIC and MSK-5 contain all benign instances and the 2018 JID Editorial Images dataset contains all malignant instances; for these we report $\%$ change in binary accuracy. For the remaining groups, we report $\%$ change in AUROC. \newtext{Note that the random training/test splits differ between (a) and (b), accounting for the differences in values for corresponding cells between the two figures.}}
\end{figure}}}
{
}
\label{sec:experimental_details}

\section{Supplementary Experiments}
\setcounter{figure}{0}   
In this appendix we detail additional experiments to supplement the findings of \cref{sec:experiments}. These experiments investigate the robustness of our findings to different problem formulations and data pre-processing.

\subsection{The Effects of Including Groups as Features or Not}
\label{sec:groups_as_features}
Here we compare models that use group or demographic information as features, and models that do not. We examine differences in group performance across $\vec{\alpha}$ and the effect of importance weighting  in both cases.

\begin{table}[]
\renewcommand{\arraystretch}{1.2}
    \centering
    \footnotesize
    \begin{tabular}{c|c|c|c|c|c|c|c}
    \toprule
        dataset & 
         $M_g$ &group $\group$ & $\hat{\sigma}_\group$ & $\hat{p}_\group$ & $\hat{\tau}_\group$ & $\hat{q}_\group$ & $\hat{\delta}_\group$  \\ \hline
 \multirow{2}{*}{\shortstack[1]{Mooc (with \\ dem. features) }} & \multirow{2}{*}{50} & edu $\leq 2^{\circ}$ & 0.08 (2.6e-05) & 0.14 (6.0e-03) & 0.73 (0.059) & 0.63 (4.8e-03) & 1.3e-15 (2.6e-04) \\ & & edu $> 2^{\circ}$ & 0.038 (6.2e-04) & 0.068 (6.3e-03) & 0.54 (6.5e-03) & 0.61 (9.8e-04) & 2.8e-12 (8.0e-04)\\
 \hline
 \multirow{2}{*}{\shortstack[1]{Mooc (without \\ dem. features) }} & \multirow{2}{*}{50} & edu $\leq 2^{\circ}$ & 0.41 (0.86) & 1.0 (0.26) & 0.6 (0.028) & 0.56 (3.5e-03) & 0.029 (1.4e-06) \\ & & edu $> 2^{\circ}$ & 0.029 (1.3e-03) & 0.055 (0.011) & 0.33 (2.2e-03) & 0.54 (9.5e-04) & 1.9e-14 (1.5e-03) \\
   \hline
   \multirow{2}{*}{\shortstack[1]{Adult (with \\ group features) } } & \multirow{2}{*}{50} & female & 0.036 (3.2e-06) & 0.14 (2.5e-03) & 0.23 (3.3e-03) & 0.47 (2.3e-03) & 0.055 (2.1e-05) \\ & & male & 0.12 (6.8e-05) & 0.24 (5.8e-04) & 0.3 (6.2e-03) & 0.47 (2.7e-03) & 0.16 (3.5e-06) \\
   \hline
   \multirow{2}{*}{\shortstack[1]{Adult (without \\ group features) } }  & \multirow{2}{*}{50} & female & 0.078 (0.051) & 0.018 (3.6e-03) & 0.43 (8.3e-03) & 0.59 (1.6e-03) & 8.0e-16 (0.052) \\ & & male & 0.066 (2.6e-05) & 0.21 (1.2e-03) & 0.47 (6.5e-03) & 0.50 (1.1e-03) & 0.16 (5.4e-06) \\
   \bottomrule
    \end{tabular}
    \caption{Estimated scaling parameters for eq.~\eqref{eq:general_scaling_law} with and without demographic features included. Parentheses denote standard deviations estimated by the nonlinear least squares fit. Parameters are constrained so that  $\hat{\tau}_\group, \hat{\sigma}_\group \geq 0$ and $\hat{p}_\group, \hat{q}_\group \in [0,2]$.}
    \label{tab:scaling_fits_without_demographic_features}
\end{table}

\begin{figure}[!ht]
     \centering
     \begin{subfigure}[b]{0.48\textwidth}
         \centering
 \includegraphics[width=.9\textwidth]{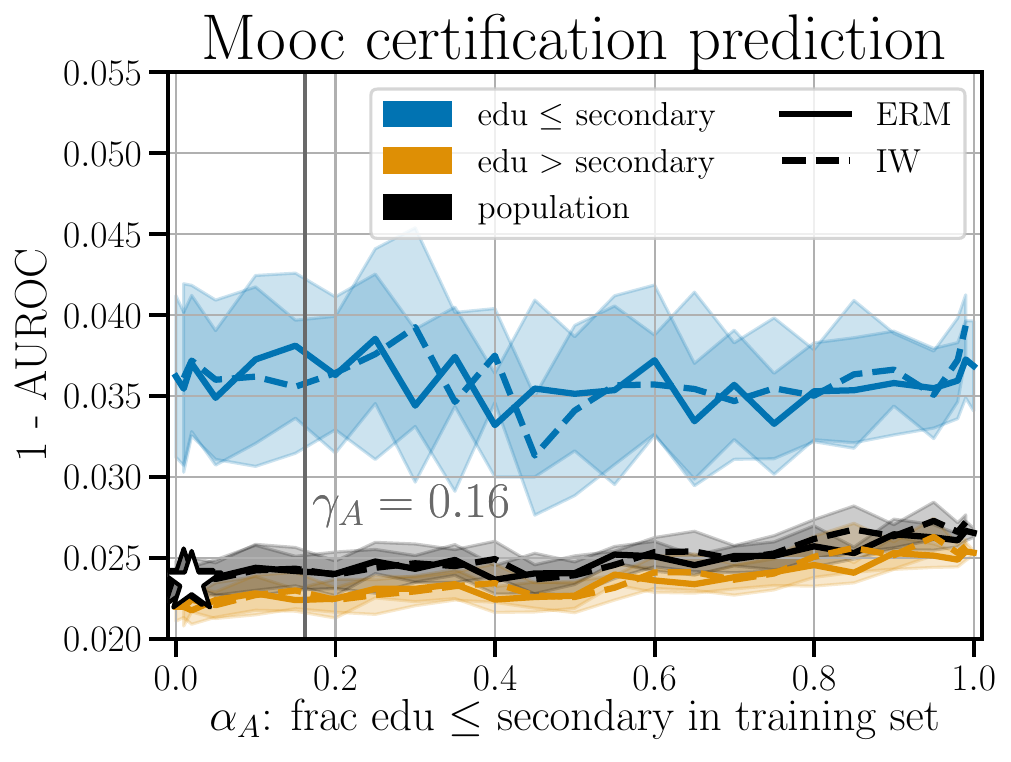}
 \caption{Mooc dataset with demographic features excluded.
 \label{fig:mooc_without_demographics}}
     \end{subfigure}
     ~~~
     \begin{subfigure}[b]{0.48\textwidth}
         \centering
  \includegraphics[width=.9\textwidth]{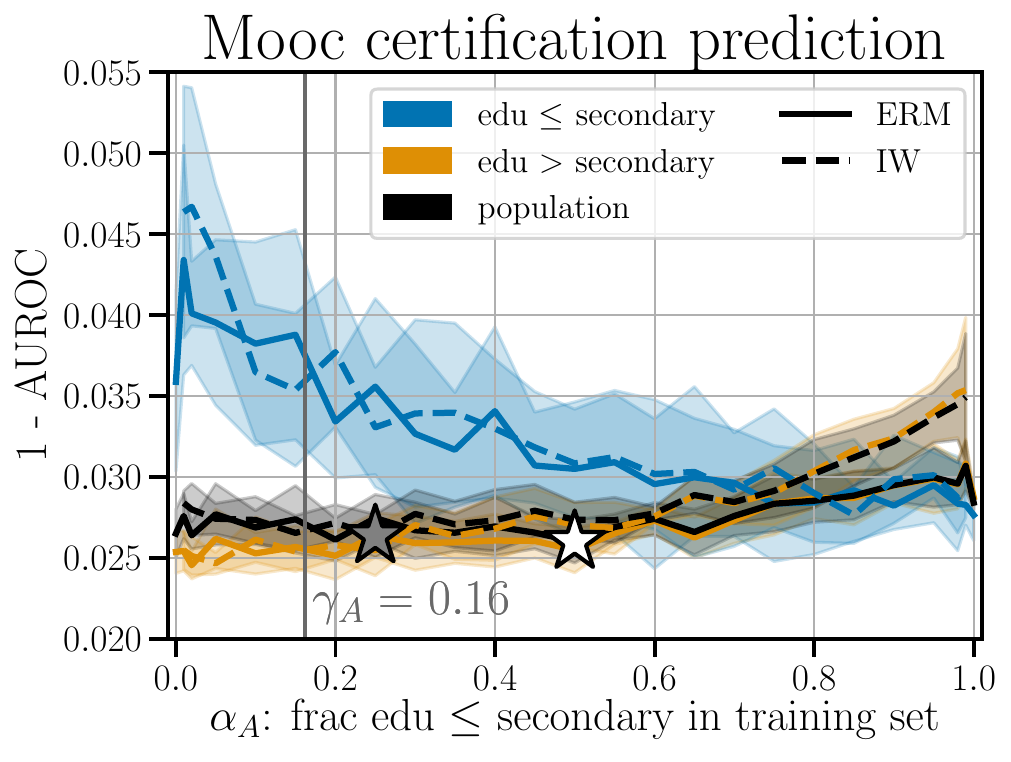}
 \caption{Mooc dataset with demographic features included.
 \label{fig:mooc_with_demographics}}
     \end{subfigure}
     \\
          \begin{subfigure}[b]{.48\textwidth}
         \centering
 \includegraphics[width=.9\textwidth]{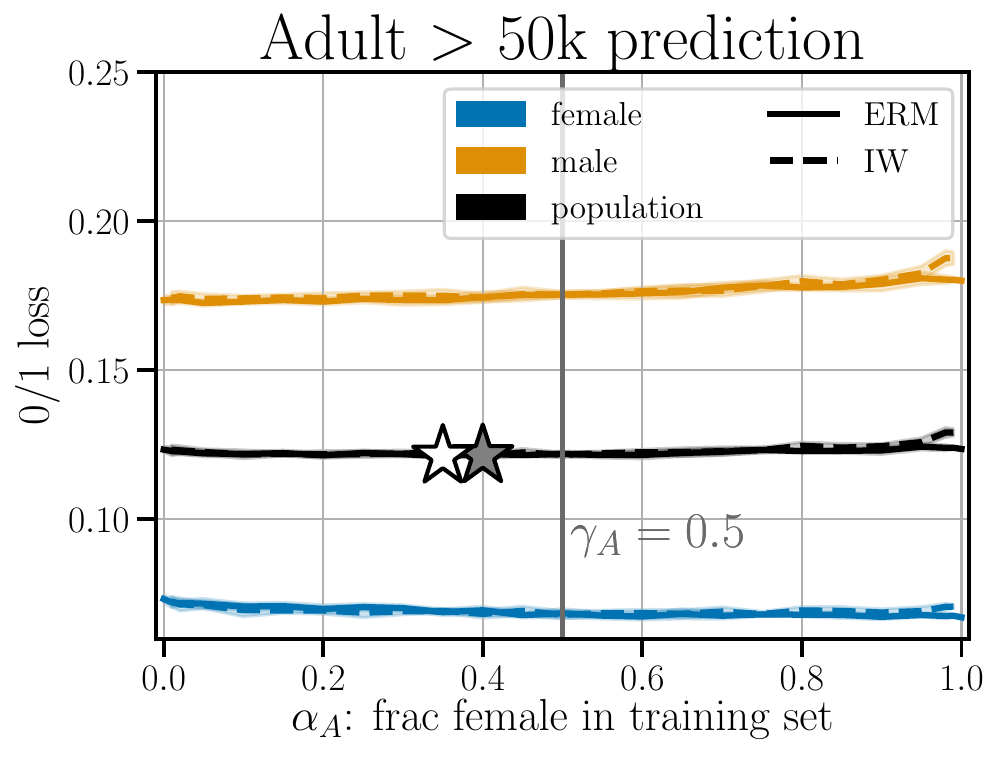}
 \caption{Adult dataset with demographic features excluded.
 \label{fig:adult_without_demographics}}
     \end{subfigure}
     ~~~
     \begin{subfigure}[b]{0.48\textwidth}
         \centering
  \includegraphics[width=.9\textwidth]{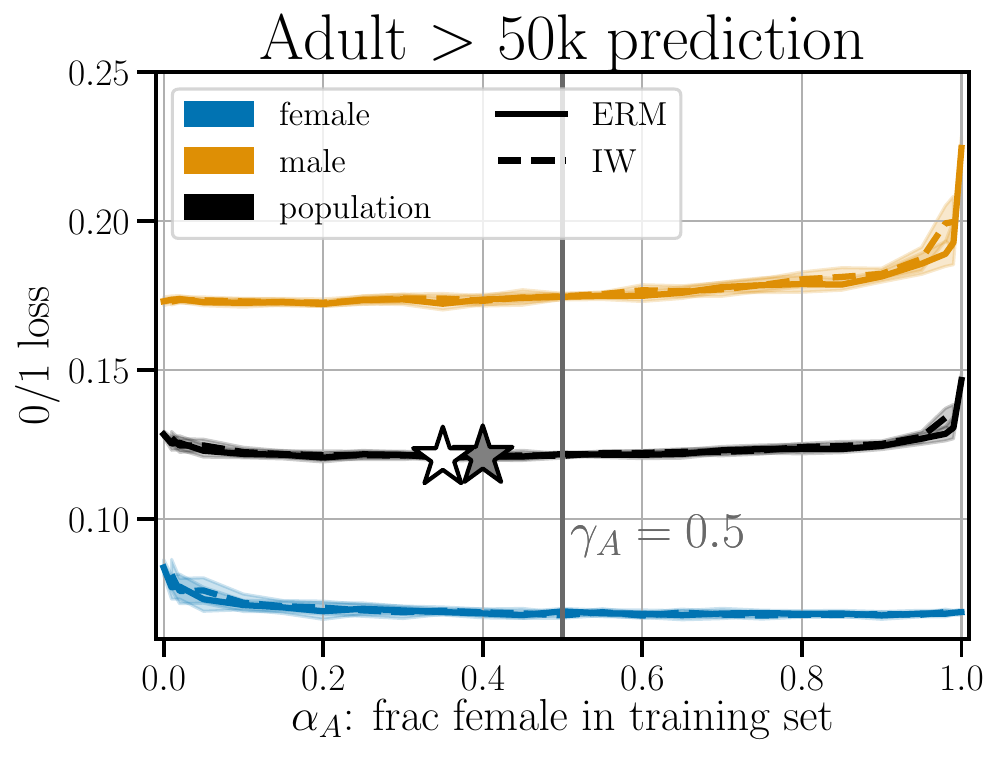}
 \caption{Adult dataset with demographic features included.
  \label{fig:adult_with_demographics}}
     \end{subfigure}
       \caption{Comparison of phenomena with group information included or excluded from training set. 
        \label{fig:with_and_without_demographics_adult_mooc}}
\end{figure}

For Mooc, we compare the model used in the main analysis  with removing all demographic information including education level, which we group data instances by. There are five remaining features: number of active days, number of video plays, number of forum posts, and number of events, and course-id, which we one-hot encode. The random forest model is still the best performing of those we considered, though the optimal hyperparameters after a 5 fold cross validation search shifted to $\texttt{num\_trees} = 400$ and $\texttt{max\_depth} = 8$. 
For the Adult dataset, we add back in the features for sex, wife, and husband. This results in 13 unique features, and after one-hot encoding, 103 feature columns. After running a 5 fold cross validation search with the new feature matrix, the optimal model and parameters were $\texttt{num\_trees} = 400$ and $\texttt{max\_depth} = 16$. 

\cref{fig:with_and_without_demographics_adult_mooc} compares the results of after excluding (left column) and including (right column) this information in the training feature matrix. Models with the group features excluded generally exhibit less degradation in performance for $\alpha_A$ near 0 or 1. The negative impacts of IW are lessened slightly when the models do not incorporate groups as features.

The different scaling law fits are shown in \cref{tab:scaling_fits_without_demographic_features}. For the Adult task without group features, the estimated exponent and scaling constant on group-agnostic data, $\hat{q}_\group$ and $\hat{\tau}_\group$, are larger than than for the model with demographic features. For the Mooc task, the effect is the reverse. In all settings, the fitted parameters show that $n$ influences group errors more than $n_\group$, so long as $n_\group$ is at least $M_\group = 50$ (the region for which we estimate the fit).

In the main analysis, we include demographic features for the Mooc certification prediction task, as this increases accuracy of the ``edu $\leq$ secondary group," and exclude gender features from the Adult income prediction, as it does not not greatly impact performance for $\alpha_A$ near 0.5, and improves performance for more extreme settings of $\alpha_A$ (\cref{fig:with_and_without_demographics_adult_mooc}). For the other three datasets we study, there is no singular feature corresponding to the group that we can hold out in a similar fashion.

\subsection{Trends across models and performance metrics}
\label{sec:additional_model_results}
Figure~\ref{fig:goodreads_mae_and_mse} contrasts the trade offs we can navigate with different allocations and re-weighting schemes for different models and  accuracy metrics. 
The left hand side of \cref{fig:goodreads_mae_and_mse} shows the result of the model fitting and evaluation procedure to minimize $\ell_1$ loss (MAE) of the Goodreads predictions (same as the results shown in \cref{fig:uplot_multifig}). The right hand side shows the result of the same model fitting and evaluation procedure applied to minimizing the $\ell_2$ loss (MSE), where ridge regression is the best performing model of those we consider. The ridge model evaluated with MSE exhibits similar trends across $\vec{\alpha}$ values as the multiclass logistic regression model evaluated on MAE. The value $\alpha^*_A$ that minimizes population loss is similar for both methods -- near to $\gamma_A$, though on opposite sides of $\gamma_A$. The degradation in population loss for $\alpha_A$ close to 0 or 1 is also apparent for the ridge model, and the range of $\vec{\alpha}$ for which importance weighting (IW) increases per-group and population losses compared to empirical risk minimization (ERM) is similar across models.

\begin{figure}[!ht]
     \centering
     \begin{subfigure}[b]{0.45\textwidth}
         \centering
 \includegraphics[width=\textwidth]{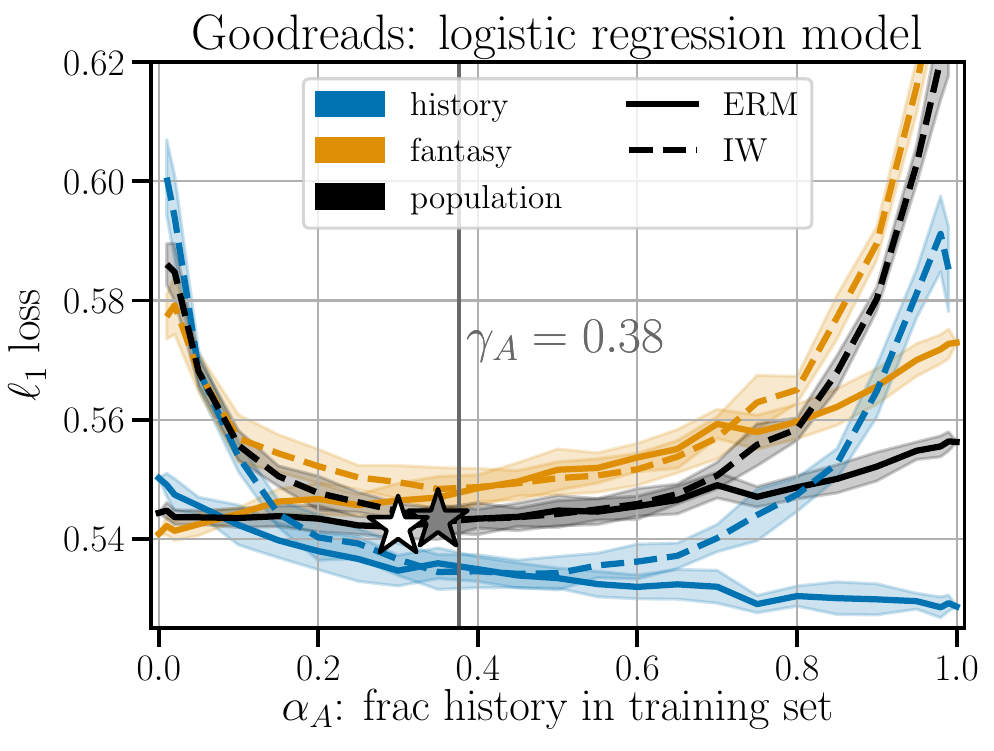}
     \end{subfigure}
     \hspace{1em}
     \begin{subfigure}[b]{0.45\textwidth}
         \centering
 \includegraphics[width=\textwidth]{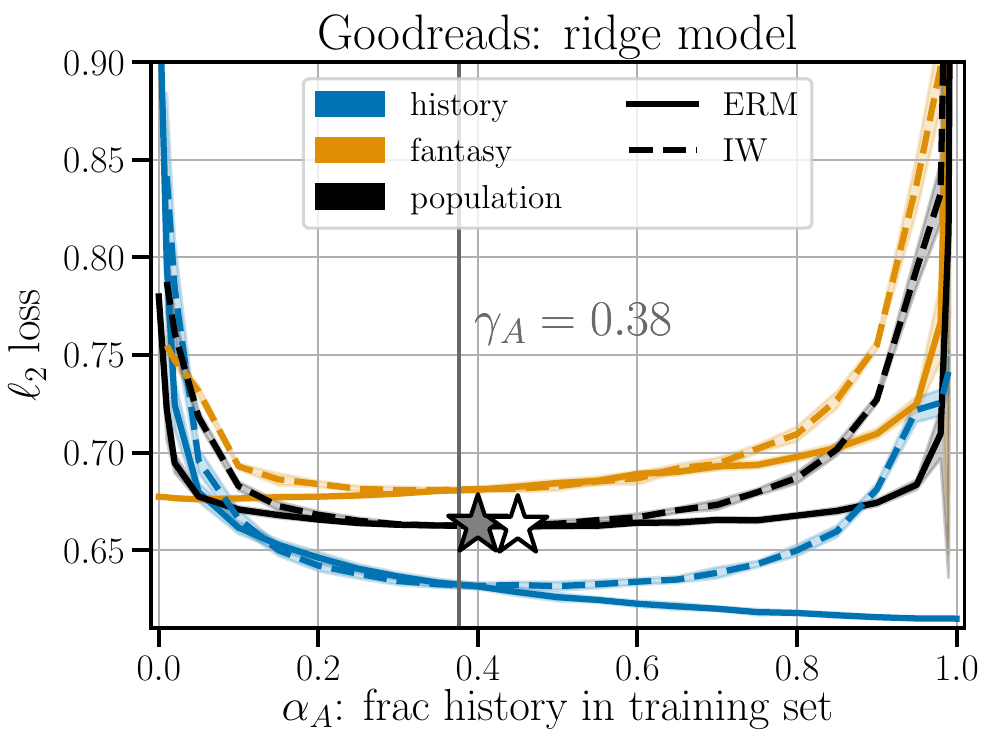}
     \end{subfigure}
      \hfill
         \caption{Comparison of phenomena for different accuracy metrics and models on the Goodreads dataset.}
        \label{fig:goodreads_mae_and_mse}
\end{figure}
\label{sec:appendix_experiments}

\end{document}